\documentclass[letterpaper]{article} 
\usepackage{aaai2026}  
\usepackage{times}  
\usepackage{helvet}  
\usepackage{courier}  
\usepackage[hyphens]{url}  
\usepackage{graphicx} 
\usepackage{natbib}  
\usepackage{caption} 
\usepackage{algorithm}
\usepackage{algorithmic}
\usepackage{subfigure}

\usepackage{mathtools}
\usepackage{amsthm}

\usepackage{amssymb}
\usepackage{amsmath}  
\usepackage{booktabs} 
\usepackage{multirow} 

\theoremstyle{plain}
\newtheorem{theorem}{Theorem}

\newtheorem{lemma}{Lemma}

\newtheorem{cor}{Corollary}

\theoremstyle{definition}
\newtheorem{definition}{Definition}
\newtheorem{assumption}{Assumption}

\theoremstyle{remark}
\newtheorem{remark}{Remark}

\DeclareMathOperator{\diag}{diag}
\DeclareMathOperator{\tr}{tr}
\DeclareMathOperator{\Var}{Var}
\DeclareMathOperator{\Cov}{Cov}

\newcommand{\norm}[1]{\left\lVert#1\right\rVert}

\usepackage{newfloat}
\usepackage{listings}
\DeclareCaptionStyle{ruled}{labelfont=normalfont,labelsep=colon,strut=off} 
\floatstyle{ruled}
\newfloat{listing}{tb}{lst}{}
\floatname{listing}{Listing}
\title{High Dimensional Distributed Gradient Descent with Arbitrary Number of Byzantine Attackers}

\author {
	Wenyu Liu\textsuperscript{\rm 1},
	Tianqiang Huang \textsuperscript{\rm 1},
	Pengfei Zhang \textsuperscript{\rm 2},
	Zong Ke \textsuperscript{\rm 3},
	Minghui Min \textsuperscript{\rm 4 },
	Puning Zhao \textsuperscript{\rm 5\footnote{corresponding author}}
}
\affiliations {
	\textsuperscript{\rm 1} College of Computer and Cyber Security, Fujian Normal University\\
	\textsuperscript{\rm 2} Anhui University of Science and Technology\\
	\textsuperscript{\rm 3} National University of Singapore\\
	\textsuperscript{\rm 4} China University of Mining Technology\\

	\textsuperscript{\rm 5} School of Cyber Science and Technology, Sun Yat-sen University\\

	\{liuwenyu, fjhtq\}@fjnu.edu.cn, \{zpf.bupt\}@bupt.cn,\{a0129009\}@u.nus.edu \{minmh\}@cumt.edu.cn,\{zhaopn\}@mail.sysu.edu.cn}

\usepackage{bibentry}

\begin{document}

\maketitle

\begin{abstract}
Adversarial attacks pose a major challenge to distributed learning systems, prompting the development of numerous robust learning methods. However, most existing approaches suffer from the curse of dimensionality, i.e. the error increases with the number of model parameters. In this paper, we make a progress towards high dimensional problems, under arbitrary number of Byzantine attackers. The cornerstone of our design is a direct high dimensional semi-verified mean estimation method. The idea is to identify a subspace with large variance. The components of the mean value perpendicular to this subspace are estimated using corrupted gradient vectors uploaded from worker machines, while the components within this subspace are estimated using auxiliary dataset. As a result, a combination of large corrupted dataset and small clean dataset yields significantly better performance than using them separately. 
We then apply this method as the aggregator for distributed learning problems. The theoretical analysis shows that compared with existing solutions, our method gets rid of $\sqrt{d}$ dependence on the dimensionality, and achieves minimax optimal statistical rates. Numerical results validate our theory as well as the effectiveness of the proposed method.
\end{abstract}


\section{Introduction}
Many modern machine learning tasks require distributed computing and storage. In such systems, there are many \textit{worker machines}, which operate under the coordination of a centralized server, known as the \textit{master machine}. During the whole training process, worker machines only need to compute an update to the current model, and send it to the master, while the local dataset remains private. This framework is called Federated Learning (FL) \cite{mcmahan2017communication}, which has drawn considerable attention in recent years \cite{zhang2021survey,kairouz2021advances}. With advantages in both computational efficiency and privacy protection, FL has been widely applied in various areas, including smart devices, industrial engineering and healthcare \cite{li2020review,li2020federated}.

Despite significant advantages and widespread applications, FL is still facing many challenges \cite{kairouz2021advances}. One critical issue is that not all machines are trustable \cite{lyu2020threats}. In large scale FL systems, some worker machines may send wrong gradients to the master due to various reasons, including system crashes, faulty hardware, communication errors, and even malicious attacks \cite{bagdasaryan2020backdoor,bhagoji2019analyzing,fang2020local,sun2021data,luo2021feature}. These faults are typically modeled as Byzantine failure \cite{lamport1982byzantine}, such that some workers are manipulated by an adversary, and send wrong gradient vectors to the master.

There have been extensive research on Byzantine robust distributed learning. Many popular methods \cite{blanchard2017machine,guerraoui2018hidden,chen2017distributed,yin2018byzantine} have two issues that make them not perfect. Firstly, these methods only allow less than half worker machines to be Byzantine. Secondly, the statistical risk grows with dimension $d$, which is serious in modern large scale learning models. For the former issue, several new methods suitable for arbitrary number of Byzantine attackers have been proposed recently \cite{cao2019distributed,regatti2020bygars,xie2019zeno,xie2020zeno++,cao2020fltrust}. These methods rely on the help of a small auxiliary clean dataset, which is necessary since the breakdown point of robust statistics can not exceed $1/2$ \cite{hampel1971general}. The latter problem, i.e. curse of dimensionality, has been addressed by some recent works based on high dimensional robust statistics  \cite{diakonikolas2023algorithmic}, such as \cite{su2018securing,shejwalkar2021manipulating,zhu2022c}, which have constant dependence on $d$. However, to the best of our knowledge, there are no previous methods that can solve both two issues mentioned above. Unfortunately, in modern federated learning problems, it may happen that models have large scale and most workers are not trusted. 

In this paper, we tackle both problems mentioned above simultaneously. In particular, we propose a novel approach to high dimensional distributed learning under arbitrary number of Byzantine attackers, with the help of a small auxiliary dataset. The core of our new approach is a new algorithm for semi-verified mean estimation, whose goal is to estimate the statistical mean by combining a small clean dataset and a large untrusted dataset \cite{charikar2017learning}. The basic idea is to filter out some gradient vectors that are very likely to come from Byzantine workers, and then identify a linear subspace, in which the components in this subspace are hard to estimate using these corrupted gradient vectors. These components are then estimated using auxiliary clean dataset. For other components that are perpendicular to this subspace, we just use the projected sample average. Our design is partially motivated by \cite{diakonikolas2021list}. A theoretical analysis is provided, including the upper bound of our algorithm and the information-theoretic minimax lower bound. The new semi-verified mean estimation algorithm is then used as the gradient aggregator for robust distributed learning problem. 

\section{Preliminaries}
\subsection{Problem Statement of Byzantine Robust Distributed Learning}
To begin with, we formalize the problem of distributed learning under Byzantine failures. Our formalization follows \cite{chen2017distributed,yin2018byzantine,zhu2023byzantine}.

Suppose there is a master machine $W_0$ and $m$ worker machines, $W_1,\ldots, W_m$. A clean dataset $A=\{\mathbf{Z}_{01},\ldots, \mathbf{Z}_{0N_A} \}$ is stored in the master, with size $|A|=N_A$.  Each worker machine $W_i$ has $n$ samples $\mathbf{Z}_{i1},\ldots, \mathbf{Z}_{in}$. All samples stored in both master and worker machines are identically and independently distributed (i.i.d), following a common distribution $Q$.

Let $f(\mathbf{w}, \mathbf{z})$ be a loss function of a parameter vector $\mathbf{w}\in \Omega\subseteq \mathbb{R}^d$, in which $\Omega$ is the parameter space. The population risk function is defined as
\begin{eqnarray}
	F(\mathbf{w}) = \mathbb{E}[f(\mathbf{w}, \mathbf{Z})],
\end{eqnarray} 
in which $\mathbf{Z}\sim Q$. Our goal is to learn the minimizer of $F$, i.e.
	$\mathbf{w}^*=\underset{\mathbf{w}\in \Omega}{\arg\min}F(\mathbf{w})$.


The FL framework runs as following. In each iteration, the master machine broadcasts current parameter $\mathbf{w}_t$ to all worker machines. Then these worker machines calculate the gradient using the local dataset:
\begin{eqnarray}
	\mathbf{X}_i(t) =\frac{1}{n}\sum_{j=1}^n \nabla f(\mathbf{w}_t, \mathbf{Z}_{ij}).
	\label{eq:xi}
\end{eqnarray}
For any benign workers, the master can receive $\mathbf{X}_i(t)$ exactly. On the contrary, Byzantine machines can send arbitrary messages determined by the attacker. Denote $\mathbf{Y}_i(t)\in \mathbb{R}^d$ as the vector received by the master at $t$-th iteration, then
\begin{eqnarray}
	\mathbf{Y}_i(t)=\left\{
	\begin{array}{ccc}
		\mathbf{X}_i(t) &\text{if} & i\notin \mathcal{B}\\
		\star &\text{if} & i\in \mathcal{B}.
	\end{array}
	\right.
	\label{eq:sendmaster}
\end{eqnarray}
Furthermore, we can calculate the gradient using the clean dataset stored in the master that is never corrupted, i.e.
\begin{eqnarray}
	\mathbf{X}_0(t)=\frac{1}{N_A}\sum_{j=1}^{N_A}\nabla f(\mathbf{w}_t, \mathbf{Z}_{0j}).
	\label{eq:x0}
\end{eqnarray}
\begin{algorithm}[tb]
	\caption{Robust Distributed Gradient Descent}\label{alg:learning}
	\textbf{Input:} Master machine $W_0$, worker machines $W_1,\ldots, W_m$\\
	\textbf{Parameter:}Initial weight parameter $\mathbf{w}_0\in \Omega$, step length $\eta$, total number of iterations $T$\\
	\textbf{Output:}Estimated weight $\hat{\mathbf{w}}$
	\begin{algorithmic}[1]
		\FOR{$t=0,1,\ldots, T-1$}
		\STATE \underline{Master machine}: broadcast current parameter $\mathbf{w}_t$ to all worker machines;
		\FOR{$i\in [m]$ \textbf{in parallel}}
		\STATE \underline{Worker machine $i$}: compute local gradient $\mathbf{X}_i(t)$ using \eqref{eq:xi};\\
		\IF{$i$ is normal machine}
		\STATE send $\mathbf{X}_i(t)$ to master;
		\ELSE
		\STATE send arbitrary $d$ dimensional vector to master;
		\ENDIF
		\ENDFOR
		\STATE \underline{Master machine}: Receive $\mathbf{Y}_i(t)$, $i=1,\ldots, m$ from each worker machine;\\
		\STATE Calculate aggregated gradient $g(\mathbf{w}_t)$ using $\mathbf{Y}_i(t)$, $i=1,\ldots, m$ and clean dataset in $W_0$;
		\STATE Update parameter $\mathbf{w}_{t+1} = \mathbf{w}_t-\eta g(\mathbf{w}_t)$ \label{step:update};			
		\ENDFOR
	\end{algorithmic}
\end{algorithm}

We would like to clarify that the adversary knows our algorithm, as well as $m$ gradient vectors $\mathbf{X}_i(t)$, $i=1,\ldots, m$. However, the adversary has no knowledge of the clean dataset $A$. This is practical since it is usually not hard to collect only a small set of auxiliary clean data from a reliable source that is not inspected by the attacker.

It remains to discuss how to calculate the aggregated gradient $g(\mathbf{w}_t)$. This problem can be formulated as \textit{Semi-verified mean estimation} (see Definition \ref{def:sv}), which was first defined in \cite{charikar2017learning}. 
\subsection{Problem Statement of Semi-verified Mean Estimation}
The aggregator $g(\mathbf{w}_t)$ estimates $\nabla F(\mathbf{w}_t)$ using $\mathbf{Y}_i(t)$, $i=1,\ldots, m$ as well as $\mathbf{X}_0(t)$, which is calculated using \eqref{eq:x0}. Among $m$ worker machines, $\alpha m$ of them are ensured to be benign, while others are probably Byzantine. Therefore, we state semi-verified mean estimation problem as follows.

\begin{definition}\label{def:sv}
	(Semi-verified mean estimation problem) Suppose $\mathbf{X}_1, \ldots, \mathbf{X}_m$ are i.i.d with mean $\mu^*$ and covariance matrix $V^*$. $\mathcal{B}\subset S_0=[m]$ is the set of attacked samples. If $i\notin \mathcal{B}$, then $\mathbf{Y}_i=\mathbf{X}_i$, otherwise $\mathbf{Y}_i$ is an arbitrary vector determined by the attacker. In addition, we have a sample $\mathbf{X}_0$ with mean $\mu^*$ and covariance matrix $V_A^*$, which is guaranteed to be unattacked. The task is to estimate $\mu^*$ using untrusted samples $\mathbf{Y}_1,\ldots, \mathbf{Y}_m$ and the clean sample $\mathbf{X}_0$.
\end{definition} 
Consider that the semi-verified mean estimation need to be repeated for all iterations, we omit timestep $t$ in Definition \ref{def:sv}. Each sample $\mathbf{Y}_i$ corresponds to the gradient vector from a worker machine $W_i$. At iteration $t$, $\mu^*=\nabla F(\mathbf{w}_t)$. Denote $V_0$ as the covariance matrix of $\nabla f(\mathbf{w}_t, \mathbf{Z}_{ij})$, then according to \eqref{eq:xi} and \eqref{eq:x0}, $V^*$ and $V_A^*$ in Definition \ref{def:sv} become
\begin{eqnarray}
	V^*=\frac{V_0}{n},
	V_A^*=\frac{V_0}{N_A}.\label{eq:vstar}
\end{eqnarray}
Finally, we discuss two attack models that are different on whether the attacked samples are randomly (Definition \ref{def:additive}) or carefully selected (Definition \ref{def:strong}).
\begin{definition}\label{def:additive}
	(Additive Contamination Model) The adversary randomly pick at most $(1-\alpha) m$ samples from $\mathbf{X}_1,\ldots, \mathbf{X}_m$, and alter their values arbitrarily.
\end{definition}
\begin{definition}\label{def:strong}
	(Strong Contamination Model) The adversary picks $(1-\alpha)m$ samples from $\mathbf{X}_1,\ldots, \mathbf{X}_m$ according to their values, and alter their values arbitrarily.
\end{definition}
These definitions follow \cite{diakonikolas2016robust,diakonikolas2020outlier}. Under additive contamination model,  adversarial samples are randomly selected, so the distribution of remaining clean samples remains unchanged; equivalently, there are already $\alpha m$ samples, and the adversary injects the remaining $(1-\alpha)m$ corrupted samples with arbitrary values.  Under strong contamination model,  since adversarial samples are carefully selected by the attacker, and thus those remaining samples follow a different distribution. Even if good samples can be identified, estimation of $\mu^*$ is still not guaranteed.

Practical scenarios usually lie in between these two models. In distributed learning problems, the attacked worker machines are usually not randomly selected, thus additive contamination model underestimates the impact of Byzantine attacks. On the contrary, strong contamination model tends to overestimate such impact, since it is unlikely for the adversary to gather full information and design optimal attack strategy. In this work, we analyze both two models.

\section{Semi-verified Mean Estimation}\label{sec:estimation}

Our algorithm for semi-verified mean estimation is shown in Algorithm \ref{alg}, which is partially motivated by the Subspace Isotropic Filtering (SIFT) algorithm proposed in \cite{diakonikolas2021list}. The idea is illustrated in Figure \ref{fig:illustrate}. The procedures are explained as follows.

\begin{algorithm}[tb]
	\caption{Semi verified mean estimation}\label{alg}
	\textbf{Input:} Untrusted samples $\mathbf{Y}_1,\ldots, \mathbf{Y}_m$ (from \eqref{eq:sendmaster}), and clean sample $\mathbf{X}_0$ (from \eqref{eq:x0})\\
	\textbf{Parameter:} $p$, $\lambda_c$\\
	\textbf{Output:} Estimated mean $\hat{\mu}$
	\begin{algorithmic}[1] 
		\STATE Initialize $S = S_0$\label{step:initial};
		\STATE 	$S=S\setminus\{i|\|\mathbf{Y}_i\|>m^{1/3} \}$;\label{step: prefilter}\\
		\WHILE{True}
		\STATE Calculate sample mean $\mu(S)$ and sample covariance $V(S)$ using \eqref{eq:mean} and \eqref{eq:cov};
		\STATE Conduct spectral decomposition of $V(S)$, such that $V(S)=\mathbf{U}\mathbf{\Lambda} \mathbf{U}^T$, in which $\mathbf{U}$ is orthogonal, and $\mathbf{\Lambda} = \diag (\lambda_1,\ldots, \lambda_d)$, $\lambda_1\geq \ldots\geq \lambda_d$;\label{step:spectral}
		\IF {$\lambda_p\geq \lambda_c$}
		\STATE Let $\mathbf{P}=\mathbf{U}_p\mathbf{U}_p^T$, in which $\mathbf{U}_p\in \mathbb{R}^{d\times p}$ is the first $p$ columns of $\mathbf{U}$;\label{step:proj}
		\STATE For each $i\in S$, calculate $\tau_i$ using \eqref{eq:taudef};
		\STATE $\tau_{\max}=\max_{i\in S}\tau_i$;
		\STATE For each $i\in S$, remove $i$ from $S$ with probability $\tau_i/\tau_{\max}$\label{step:remove};
		\ELSE
		\STATE break;
		\ENDIF
		\ENDWHILE
		\STATE Calculate $V(S)$, $\mathbf{P}$\label{step:again};
		\STATE $\hat{\mu} = \mathbf{P}\mathbf{X}_0+(\mathbf{I}-\mathbf{P})\mu(S)$;\label{step:last}
		\STATE \textbf{return} $\hat{\mu}$;
	\end{algorithmic}
\end{algorithm}
\begin{figure}[h!]
	\subfigure[Initialization and prefiltering (Step 1-2).]{\includegraphics[width=0.48\linewidth]{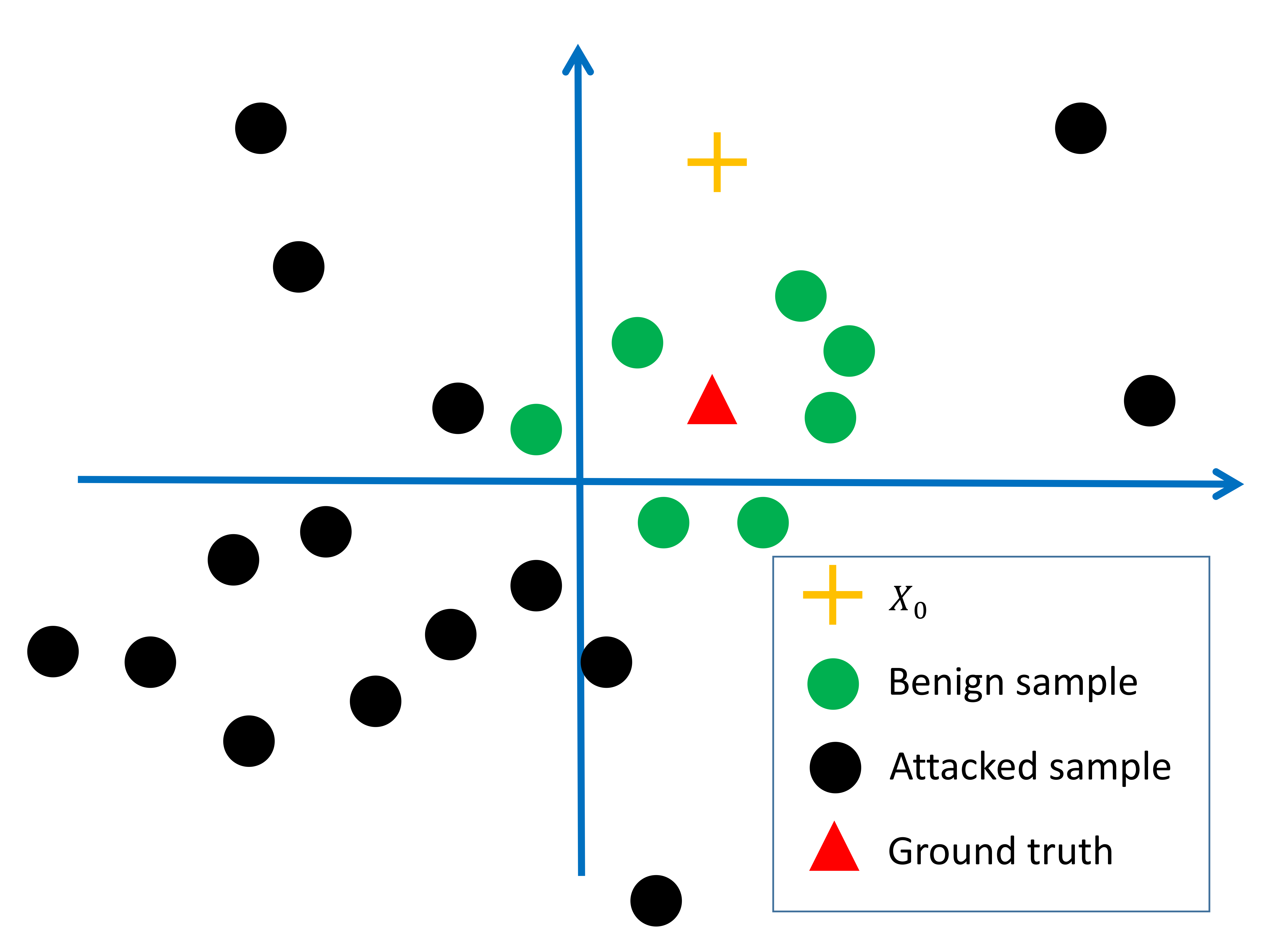}}
	\subfigure[Iterative filtering (Step 3-14).]{\includegraphics[width=0.48\linewidth]{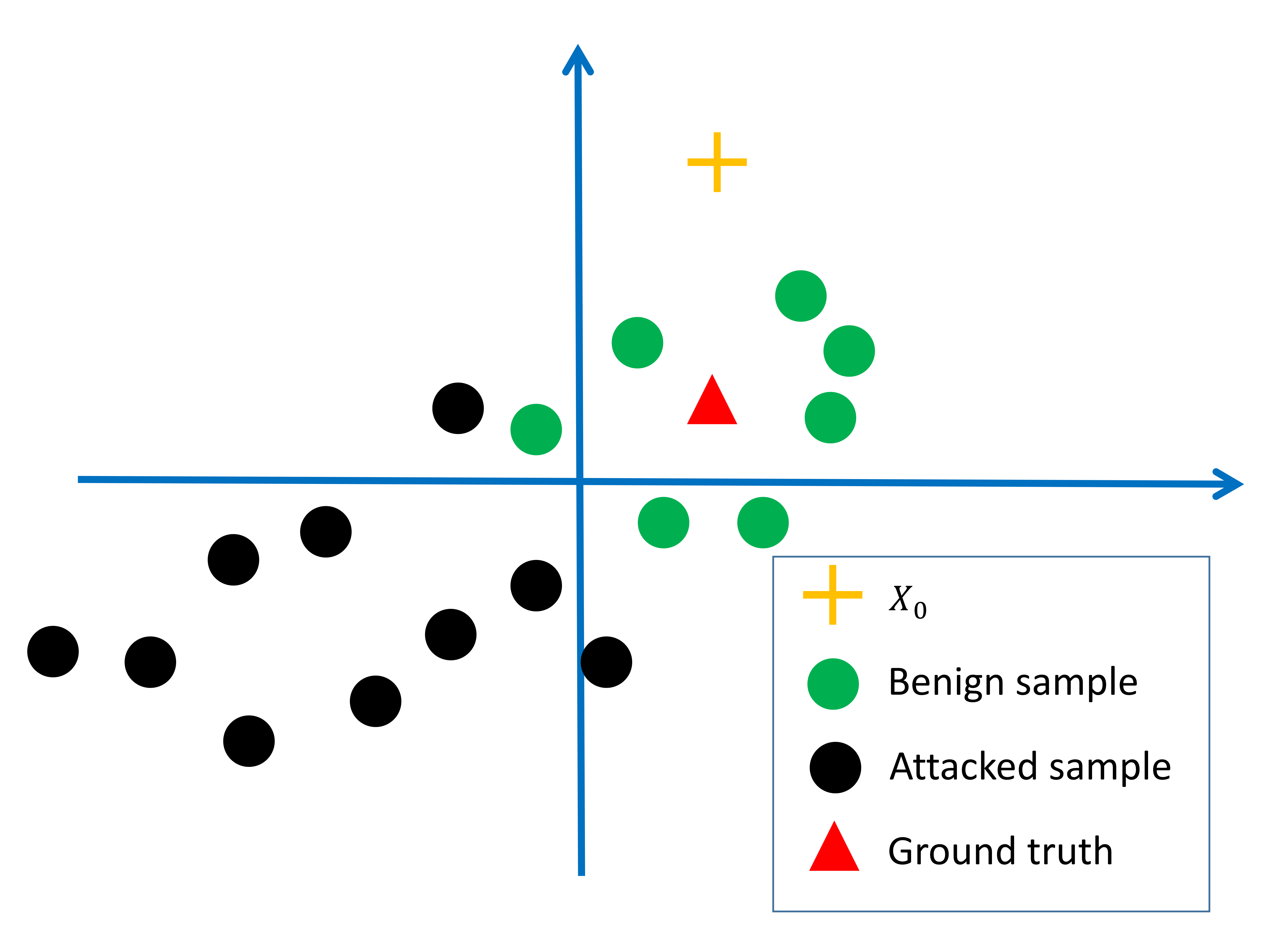}}
	\subfigure[Calculate sample mean and projection matrix (Step 15).]{\includegraphics[width=0.48\linewidth]{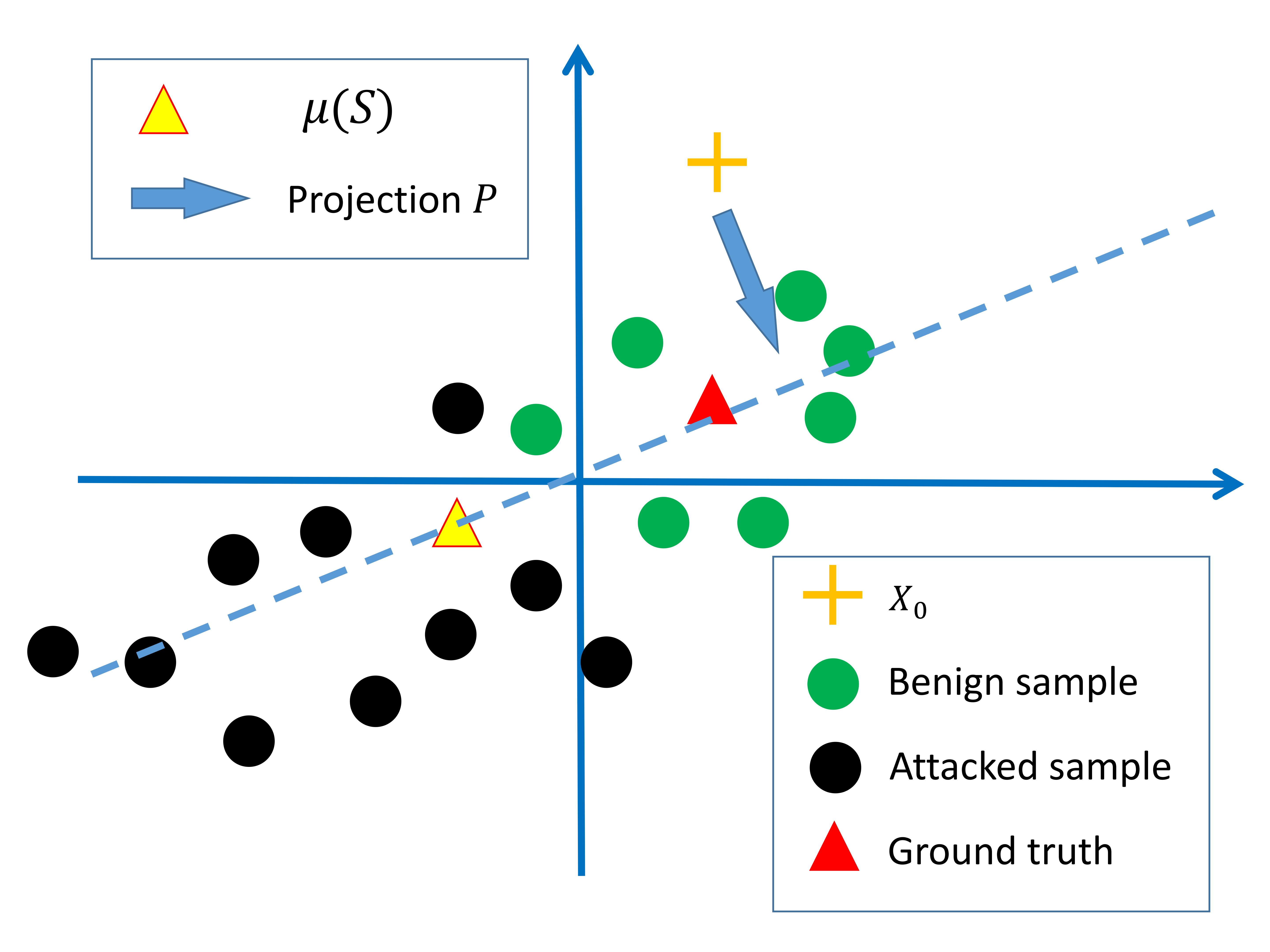}}
	\subfigure[Calculate final estimate $\hat{\mu}$ (Step 17).]{\includegraphics[width=0.48\linewidth]{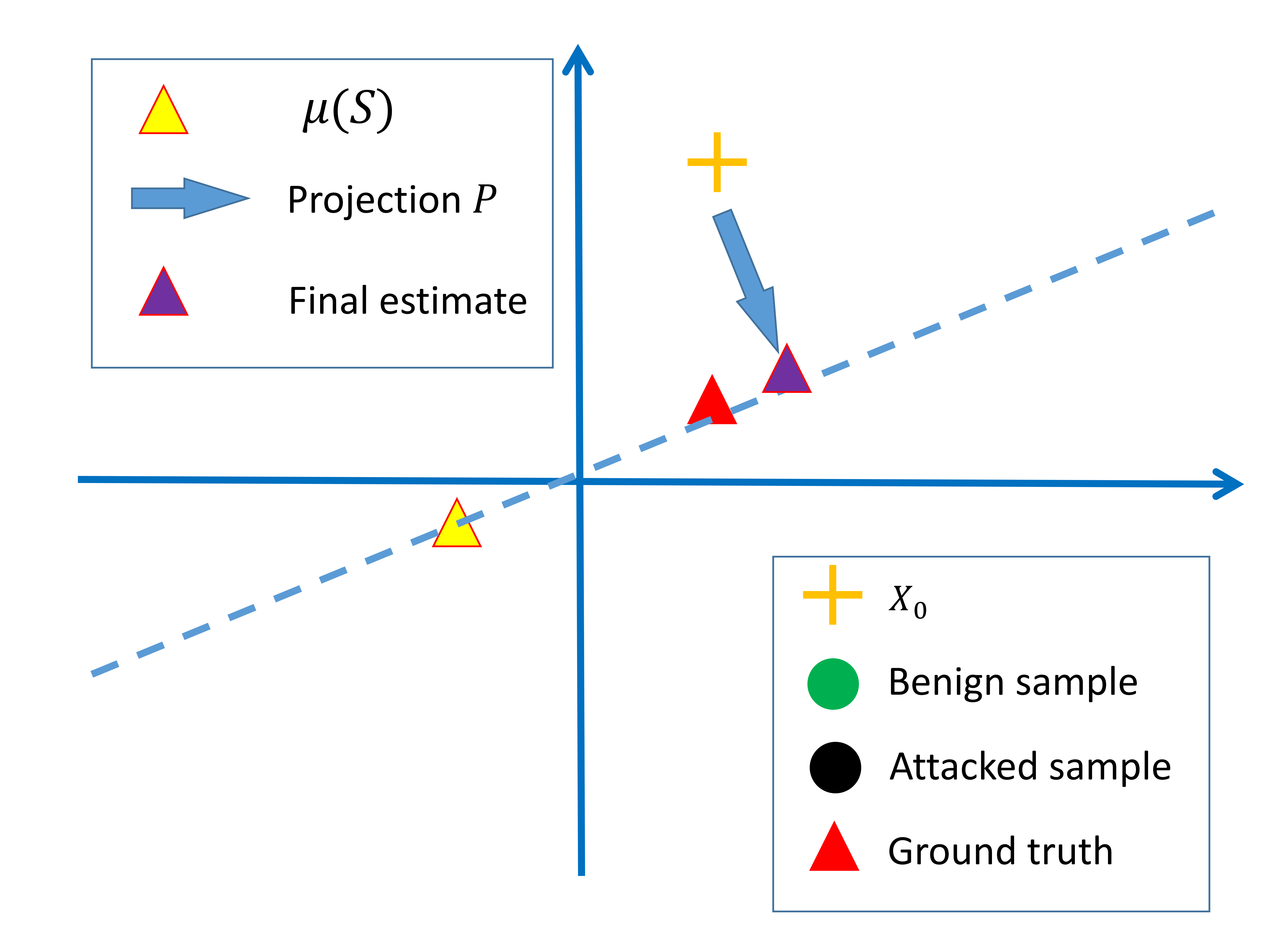}}
	\caption{A two-dimensional illustration of the semi-verified mean estimation method shown in Algorithm \ref{alg}. $\mu^*$ is represented by the red triangle. Benign and attacked samples correspond to green and black dots, respectively. The orange plus sign denotes  $\mathbf{X}_0$.}\label{fig:illustrate}
\end{figure}
(1) \textit{Initialization and prefiltering.} Let $S_0=\{1,\ldots,m \}$ be the indices of all untrusted samples and initialize $S=S_0$. We use a prefilter (step 2) to remove the samples whose norms are too large. This step is mainly used to simplify the theoretical analysis, and may not be necessary in practice. 

(2) \textit{Iterative filtering.} The algorithm then sanitizes the dataset by removing some samples that are highly likely to be attacked. In each iteration, sample mean and covariance matrix are calculated first:
\begin{eqnarray}
	\mu(S) &=& \frac{1}{|S|}\sum_{i\in S} \mathbf{Y}_i,\label{eq:mean}\\
	V(S)&=&\frac{1}{|S|}\sum_{i\in S} (\mathbf{Y}_i-\mu(S))(\mathbf{Y}_i-\mu(S))^T.
	\label{eq:cov}
\end{eqnarray}
If some eigenvalues of $V(S)$ are too large, then there should be some attacked samples which are moved in the directions of corresponding eigenvectors by the adversary. Motivated by this intuition, we conduct spectral decomposition of $V(S)$, such that $V(S) = \mathbf{U}\mathbf{\Lambda}\mathbf{U}^T$, with $\mathbf{\Lambda} = \diag(\lambda_1,\ldots, \lambda_d)$, $\lambda_1\geq \ldots \geq \lambda_d$. If there are $p$ eigenvalues larger than a certain threshold $\lambda_c$, i.e. $\lambda_p\geq \lambda_c$, then the algorithm removes some points to reduce the eigenvalue. In particular, each sample $i$ is assigned a score
\begin{eqnarray}
	\tau_i = \|\mathbf{P}V^{-\frac{1}{2}}(S)(\mathbf{Y}_i-\mu(S))\|_2^2,
	\label{eq:taudef}
\end{eqnarray}
in which $\mathbf{P}=\mathbf{U}_p\mathbf{U}_p^T$, with $\mathbf{U}_p\in \mathbb{R}^{d\times p}$ being the matrix of first $p$ columns of $\mathbf{U}$. $\mathbf{P}$ is a projection matrix to the space spanned by $p$ principal components. 

We now give  an intuitive explanation of \eqref{eq:taudef}. We first regularize the dataset by linear transformation with $V^{-\frac{1}{2}}(S)$. It is easy to show that samples after transformation have an identity covariance matrix:
\begin{eqnarray}
	&&\hspace{-1cm}\frac{1}{|S|}\sum_{i\in S} \left[V^{-\frac{1}{2}}(S)(\mathbf{Y}_i-\mu(S))\right]\left[V^{-\frac{1}{2}}(S)(\mathbf{Y}_i-\mu(S))\right]^T \nonumber\\
	&=& V^{-\frac{1}{2}}(S) V(S) V^{-\frac{1}{2}}(S)=\mathbf{I}_d.
\end{eqnarray}
After transformation, if a sample is far from the mean, i.e. $\|V^{-\frac{1}{2}}(S)(\mathbf{Y}_i-\mu(S))\|$ is large, then it is likely that this sample is attacked. As has been discussed earlier, it would be better to filter samples in directions corresponding to large eigenvalues. Therefore, \eqref{eq:taudef} incorporates the projection matrix $\mathbf{P}$. Larger $\tau_i$ indicates higher confidence that sample $i$ is attacked. With this intuition, sample $i$ is removed with probability $\tau_i/\tau_{\max}$, where $\tau_{\max}=\max_{i\in S} \tau_i$. This randomized removal is mainly for theoretical convenience; in practice, one may simply remove the $k$  samples with the largest 
$\tau$ values. The removal process continues until $\lambda_p<\lambda_c$.


Sometimes $V(S)$ has some very small eigenvalues, thus accurate numerical computation of $V^{-\frac{1}{2}}(S)$ is hard. In this case, the calculation can be simplified to $\tau_i=\|\sum_{j=1}^p \lambda_j^{-\frac{1}{2}} \mathbf{u}_j\mathbf{u}_j^T (\mathbf{Y}_i-\mu(S))\|_2^2$. This format can avoid the numerical problem of calculating \eqref{eq:taudef} directly.

(3) \textit{Calculate $V(S)$ and $\mathbf{P}$ again.} At step 15, after iterative removal, the algorithm calculates the projection matrix using step 7 in Algorithm \ref{alg}. This is illustrated in Figure \ref{fig:illustrate}(c), in which the blue arrow represents the projection operation.

(4) \textit{Calculate $\hat{\mu}$}. The semi-verified mean estimator is
\begin{eqnarray}
	\hat{\mu} = \mathbf{P}\mathbf{X}_0+(\mathbf{I}-\mathbf{P})\mu(S).
	\label{eq:hatmu}
\end{eqnarray}
As shown in Figure \ref{fig:illustrate}(d), both $\mathbf{X}_0$ and $V(S)$ are relatively far from the ground truth $\mu^*$. However, with appropriate projection, the final estimate with \eqref{eq:hatmu} is much closer to $\mu^*$.

\section{Theoretical Analysis}\label{sec:theory}

This section provides theoretical analysis about the semi-verified mean estimation method in Algorithm \ref{alg}. From now on, we use the following notations: For two matrices $\mathbf{A}$ and $\mathbf{B}$, $\mathbf{A}\preceq \mathbf{B}$ if $\mathbf{B}-\mathbf{A}$ is positive semidefinite. $a\lesssim b$ if $a\leq Cb$ for some constant $C$ that does not depend on $N_A$, $m$, $n$, $d$ and $\alpha$. $\Cov[\mathbf{U}]$ denotes the covariance matrix of random vector $\mathbf{U}$. The proofs are shown in the appendix. 

Our analysis begins with the following assumption, which requires the boundedness of covariance matrix:

\begin{assumption}\label{ass:var}
		$V_0\preceq \sigma^2 \mathbf{I}_d$,
	in which $V_0$ is the covariance matrix of $\nabla f(\mathbf{w}_t, \mathbf{Z}_{ij})$.
\end{assumption}
According to \eqref{eq:vstar}, we have $V^*\preceq (\sigma^2/n)\mathbf{I}_d$ and $V_A^*\preceq (\sigma^2/N_A)\mathbf{I}_d$.

\subsection{Upper Bound}
Theorem \ref{thm:additive} provides a bound of the performance of Algorithm \ref{alg} under additive contamination model.
\begin{theorem}\label{thm:additive}
	Under additive contamination model, if Assumption \ref{ass:var} hold, and parameters $p$, $\lambda_c$ in Algorithm \ref{alg} satisfy
	\begin{eqnarray}
		p&>&\frac{8}{\alpha},\label{eq:m}\\
		\lambda_c&>&32\frac{\sigma^2}{n}\left(1+\frac{2d}{\alpha m}\right),\label{eq:lamc}
	\end{eqnarray}
	then
	\begin{eqnarray}
		\mathbb{E}\left[\|\hat{\mu}-\mu^*\|_2^2\right]\leq \frac{3\sigma^2 p}{N_A}+\frac{15\lambda_c}{2\alpha}+\delta_m,
		\label{eq:result}
	\end{eqnarray}
	in which $\delta_m$ decays faster than any polynomial of $m$.
\end{theorem}

With $m\gtrsim d/\alpha$, from \eqref{eq:lamc}, we can let $\lambda_c\sim \sigma^2/n$, then \eqref{eq:result} becomes
\begin{eqnarray}
	\mathbb{E}\left[\|\hat{\mu}-\mu^*\|_2^2\right]\lesssim \sigma^2\left(\frac{p}{N_A}+\frac{1}{\alpha n}\right).	
	\label{eq:result2}
\end{eqnarray} 

We now discuss the selection of parameter $p$. Since \eqref{eq:m} requires $p\gtrsim 1/\alpha$, we discuss two cases. If $N_A\lesssim n$, according to \eqref{eq:result2}, let $p\sim 1/\alpha$, then $\mathbb{E}\left[\|\hat{\mu}-\mu^*\|_2^2\right]\lesssim \sigma^2 /(\alpha N_A)$. If $N_A\gtrsim n$, according to \eqref{eq:result2}, as long as $p$ satisfies $1/\alpha \lesssim p\lesssim N_A/(\alpha n)$, then $\mathbb{E}\left[\|\hat{\mu}-\mu^*\|_2^2\right]\lesssim \sigma^2 /(\alpha n)$. With this selection rule of $p$, \eqref{eq:result2} becomes
\begin{eqnarray}
	\mathbb{E}\left[\|\hat{\mu}-\mu^*\|_2\right]\lesssim \sigma\alpha^{-\frac{1}{2}}\left(\frac{1}{\sqrt{N_A}}+\frac{1}{\sqrt{n}}\right).
	\label{eq:finala}
\end{eqnarray}

Theorem \ref{thm:strong} shows the performance of Algorithm \ref{alg} under strong contamination model.

\begin{theorem}\label{thm:strong}
	Under strong contamination model, if Assumption \ref{ass:var} hold, $p$ satisfies \eqref{eq:m}, and $\lambda_c$ satisfies 
	\begin{eqnarray}
		\lambda_c>\frac{8\sigma^2}{\alpha n} \left(1+\sqrt{\frac{16d\ln^2 m}{3\alpha m}}\right)^2,
		\label{eq:lamc-strong}
	\end{eqnarray}
	then 	
	\begin{eqnarray}
		\mathbb{E}\left[\|\hat{\mu}-\mu^*\|_2^2\right]\leq \frac{3\sigma^2 p}{N_A}+\frac{15\lambda_c}{2\alpha}+\delta_m,
		\label{eq:result_strong}
	\end{eqnarray}
	in which $\delta_m$ decays faster than any polynomial of $m$.
\end{theorem}
Despite that \eqref{eq:result_strong} appears to be the same as \eqref{eq:result} for additive contamination model, the minimum value of $\lambda_c$ is different between these two models. With $m/\ln^2 m\gtrsim d/\alpha$, from \eqref{eq:lamc-strong}, we let $\lambda_c\sim \sigma^2 /(\alpha n)$, then \eqref{eq:result_strong} becomes
\begin{eqnarray}
	\mathbb{E}\left[\|\hat{\mu}-\mu^*\|_2^2\right]\lesssim \sigma^2 \left(\frac{p}{N_A}+\frac{1}{\alpha^2 n}\right).
	\label{eq:resstrong2}
\end{eqnarray}
If $N_A\lesssim \alpha n$, according to \eqref{eq:resstrong2}, let $p\sim 1/\alpha$. Otherwise, just need to ensure that $1/\alpha \lesssim p\lesssim N_A/(\alpha^2 n)$. Then \eqref{eq:resstrong2} becomes
\begin{eqnarray}
	\mathbb{E}\left[\|\hat{\mu}-\mu^*\|_2\right]\lesssim \sigma\alpha^{-\frac{1}{2}}\left(\frac{1}{\sqrt{N_A}}+\frac{1}{\sqrt{\alpha n}}\right).
	\label{eq:finalb}
\end{eqnarray}

Traditional methods, such as Zeno \cite{xie2019zeno,xie2020zeno++} has an error of $\sigma \alpha^{-1/2}(\sqrt{d/N_A}+\sqrt{d/n})$. In contrast, our approach—as shown in \eqref{eq:finala} and \eqref{eq:finalb}—does not suffer from the $\sqrt{d}$ dependence, making it significantly more effective for high-dimensional mean estimation.
Depending on whether we are assuming additive or strong contamination, the results are slightly different ($1/\sqrt{n}$ vs $1/\sqrt{\alpha n}$). As discussed earlier, strong contamination model gives the attacker more room for manipulation, since it allows the attacker to pick samples arbitrarily instead of randomly, thus the estimation error is larger.

\begin{remark}
	 \cite{diakonikolas2021list} has proposed an efficient method for list-decodable mean estimation, which generates a list of $O(1/\alpha)$ hypotheses whose minimum distance to $\mu^*$ is $O(1/\sqrt{\alpha})$. As is discussed in \cite{charikar2017learning}, list-decodable method can be used in semi-verified mean estimation. Compared with this indirect approach, our new method requires less number of auxiliary clean samples, and the parameter selection is more flexible. Detailed comparisons can be found in the appendix.
\end{remark}


\subsection{Minimax lower bound}
Now we show the information theoretic lower bound of semi-verified mean estimation problem. Under additive contamination model, the minimax risk is defined as following:
\begin{eqnarray}
	R_A(\alpha) = \underset{\hat{\mu}}{\inf}\underset{D\in \mathcal{F}}{\sup}\underset{\pi_A(\alpha)}{\sup}\|\hat{\mu}-\mu^*\|_2,
	\label{eq:ra}
\end{eqnarray}
in which $\mathcal{F}$ is the set of all distributions satisfying Assumption \ref{ass:var}, as well as \eqref{eq:vstar}, which implies that $\Cov[\mathbf{X}_i]\preceq (\sigma^2/n)\mathbf{I}$ and $\Cov[\mathbf{X}_0]\preceq (\sigma^2/N_A)\mathbf{I}$. $\pi_A(\alpha)$ is the policy of attacker, which maps $\mathbf{X}_i$ for $i\in S_0$ to $\mathbf{Y}_i$, $i\in S_0$. In particular, it picks $\lceil \alpha N\rceil$ samples randomly, and let $\mathbf{Y}_i=\mathbf{X}_i$ for these samples. For other samples, $\mathbf{Y}_i$ are arbitrary. $\hat{\mu}$ is the estimator, which is a function of $\mathbf{Y}_1,\ldots, \mathbf{Y}_m, \mathbf{X}_0$. 

Similarly, under strong contamination model, the minimax risk is defined as
\begin{eqnarray}
	R_S(\alpha) = \underset{\hat{\mu}}{\inf}\underset{D\in \mathcal{F}}{\sup}\underset{\pi_S(\alpha)}{\sup}\|\hat{\mu}-\mu^*\|_2,
	\label{eq:rs}
\end{eqnarray}
in which $\pi_S(\alpha)$ is policy of strong contamination. It maps $\mathbf{X}_i$ to $\mathbf{Y}_i$ arbitrarily, as long as $\mathbf{Y}_i=\mathbf{X}_i$ for at least $\alpha N$ samples.

The results are shown in Theorem \ref{thm:minimax}.
\begin{theorem}\label{thm:minimax}
	If
	$\left\lceil N_A/n\right\rceil\leq \ln \frac{d}{4}/\left(4\beta\left(\ln \frac{1}{\beta} + 1\right)\right)$,
	then with probability at least $1/2-\exp[-(\ln 2-1/2)m\alpha]$,
	\begin{eqnarray}
		R_A(\alpha)&\geq& \frac{\sigma}{2\sqrt{2\alpha}}\left(\frac{1}{\sqrt{n}}+\frac{1}{\sqrt{N_A}}\right),\label{eq:rabound}\\
		R_S(\alpha)&\geq & \frac{\sigma}{2\sqrt{2\alpha}}\left(\frac{1}{\sqrt{\alpha n}}+\frac{1}{\sqrt{N_A}}\right).
		\label{eq:rsbound}
	\end{eqnarray}
\end{theorem}

The upper bound matches the minimax lower bound up to constant factors. Such results indicate that the error rates of Algorithm \ref{alg} are optimal. In other words, it is impossible to further improve the bounds in Theorem \ref{thm:additive} and \ref{thm:strong} in general.

\section{Application in Distributed Learning under Byzantine Attack}\label{sec:byzantine}

Based on the analysis of semi-verified mean estimation in the previous section, we now analyze the distributed learning method in Algorithm \ref{alg:learning}, which uses our new semi-verified mean estimation method as the aggregator function. Our analysis is based on the following assumption.
\begin{assumption}\label{ass:distributed}
	(a) For all $\mathbf{w}\in \Omega$, $\nabla f(\mathbf{w}, \mathbf{Z})$ is sub-exponential with parameter $\sigma$, i.e.
	\begin{eqnarray}
		\underset{\mathbf{v}:\|\mathbf{v}\|_2=1}{\sup}\mathbb{E}\left[e^{\lambda \mathbf{v}^T\left(\nabla f(\mathbf{w}, \mathbf{Z})-\nabla F(\mathbf{w})\right)}\right]\leq e^{\frac{1}{2}\sigma^2\lambda^2}, 
	\end{eqnarray}
	for all $\lambda$ with $|\lambda|\leq 1/\sigma$;
	
	(b) $F(\mathbf{w})$ is $\mu$-strong convex and $L$-smooth in $\mathbf{w}$.
\end{assumption}

(a) is more restrictive than Assumption \ref{ass:var}. (a) requires the distribution of gradient values to be sub-exponential, while the latter only requires bounded eigenvalues of covariance matrix. Such strengthened assumption is made only for theoretical completeness. For (b), despite that our theoretical results are derived under the assumption that $F$ is strong convex, similar to \cite{yin2018byzantine}, our analysis can be easily generalized to the case with non-strong convex and nonconvex functions.

Theorem \ref{thm:distributed} bounds the final error of $\hat{\mathbf{w}}$ under additive and strong contamination model.
\begin{theorem}\label{thm:distributed}
	Suppose that the following conditions are satisfied: 	
	(1) Assumption \ref{ass:distributed} hold; (2) $p>8/\alpha$; (3) $\lambda_c$ satisfies
	\begin{eqnarray}
		\lambda_c > 32\frac{\sigma^2}{n}\left(1+\frac{2d}{\alpha m}\right)
		\label{eq:lamc-dl}
	\end{eqnarray}	
	under additive contamination model, or
	\begin{eqnarray}
		\lambda_c > \frac{8\sigma^2}{\alpha n} \left(1+\sqrt{\frac{16d\ln^2 m}{3\alpha m}}\right)^2
		\label{eq:lamc-dl-strong}
	\end{eqnarray}
	under strong contamination model; (4) $\eta \leq 1/L$; (5) The size of auxiliary clean dataset satisfies
		$N_A\geq 2\ln (mT/\delta)$.	
	
	Then with probability at least $1-\delta-Te^{-\frac{1}{64}\alpha m}-4Tme^{-\frac{1}{16} \lambda_cm\alpha^2 m^\frac{1}{3}\epsilon^2}$,
	\begin{eqnarray}
		\|\hat{\mathbf{w}}-\mathbf{w}^*\|\leq (1-\rho)^T\|\mathbf{w}_0-\mathbf{w}^*\|_2+\frac{2\Delta}{\mu},
	\end{eqnarray}
	in which $\hat{\mathbf{w}}=\mathbf{w}_T$ is the updated weight after $T$ iterations, and $\rho = \eta\mu/2$,
	
	\begin{eqnarray}
		\Delta = \sqrt{\frac{6p}{N_A}\sigma^2 \ln \frac{pT}{\delta} + \frac{15\lambda_c}{2\alpha}}.
		\label{eq:delta}
	\end{eqnarray}
\end{theorem}

If $m\gtrsim d/\alpha$, $p\sim 1/\alpha$, and $T$ is large enough, then under additive contamination model, with $\lambda_c\sim \sigma^2 /n$,
\begin{eqnarray}
	\|\mathbf{w}_T-\mathbf{w}^*\|_2=\tilde{O}\left(\frac{\sigma}{\sqrt{\alpha}}\left(\frac{1}{\sqrt{N_A}}+\frac{1}{\sqrt{n}}\right)\right).
	\label{eq:final}
\end{eqnarray}

Under strong contamination model, with $\lambda_c\sim \sigma^2 / (\alpha n)$,
\begin{eqnarray}
	\|\mathbf{w}_T-\mathbf{w}^*\|_2=\tilde{O}\left(\frac{\sigma}{\sqrt{\alpha}}\left(\frac{1}{\sqrt{N_A}}+\frac{1}{\sqrt{\alpha n}}\right)\right).
\end{eqnarray}

The above results show that, as long as the number of worker machines $m$ grows proportionally with $d$, and $N_A$ grows slightly with $d$, then the learning error does not increase with dimensionality. Hence, compared with \cite{cao2019distributed,xie2019zeno,xie2020zeno++}, our new method removes the $\sqrt{d}$ dependence and is thus more suitable to high dimensional problems. 


\section{Numerical Results}\label{sec:numerical}
This section provides results of numerical simulation. We compare different gradient aggregators:

\textbf{1. Master only.} This means that we only use the gradient values from the auxiliary clean data stored in the master:
\begin{eqnarray}
	g_{Master}(\mathbf{w}) = \frac{1}{N_A}\sum_{j=1}^{N_A}\nabla f(\mathbf{w}, \mathbf{Z}_{ij}).
\end{eqnarray}
This method is used as a baseline, in order to show the benefits of combining clean samples with the untrusted gradient information from worker machines.

\textbf{2. Distance based filtering.} This method comes from \cite{cao2019distributed}. Here we just assume that $q$ is known. Among all $m$ gradient vectors from worker machines, this method picks $m-q$ closest one, and then calculate the weighted average:
\begin{eqnarray}
	g_{DBF}(\mathbf{w}) = \frac{N_Ag_{Master}+n\underset{i\in \mathcal{N}_{m-q}(g_{Master}(\mathbf{w}))}{\sum} \mathbf{Y}_i}{N_A+n(m-q)},
\end{eqnarray}
in which $\mathcal{N}_{m-q}(g_{Master}(\mathbf{w}))$ means the $m-q$ nearest neighbors of $g_{Master}$ among corrupted gradient vectors from worker machines, $\{\mathbf{Y}_1,\ldots, \mathbf{Y}_m \}$.

\textbf{3. Zeno.} This method was proposed in \cite{xie2019zeno}. After that, an improved version was provided in \cite{xie2020zeno++}, called Zeno++. Here we use the first order expansion of stochastic descent score mentioned in \cite{xie2020zeno++}. In particular, this method assigns a score 
\begin{eqnarray}
	Score(\mathbf{Y}_i, g_{Master}) = \gamma \langle g_{Master}, \mathbf{Y}_i\rangle - \rho\|\mathbf{Y}_i\|_2^2,
	\label{eq:sdscore}
\end{eqnarray}
and then use the average gradients from $m-q$ worker machines with highest scores:
	$g_{Zeno}(\mathbf{w}) =(1/(m-q))\sum_{i=1}^{m-q}\mathbf{Y}_{(i)}$, 
in which $\mathbf{Y}_{(i)}$ is the gradient vector with $i$-th highest score according to \eqref{eq:sdscore}.

\textbf{4. Our approach.} This refers to our new approach in Algorithm \ref{alg:learning}.

For all these four methods, we implement two types of attack: random attack, in which $\mathbf{Y}_i\sim \mathcal{N}(0, \sigma_{attack}^2 \mathbf{I})$, and sign-flip attack, which flips the sign of original gradient vector, such that $\mathbf{Y}_i=-\mathbf{X}_i$.

We would like to remark here that our method relies on auxiliary clean dataset, thus we only compare with previous methods that also use auxiliary samples. Other methods designed for $q<m/2$, such as Krum \cite{blanchard2017machine}, geometric median-means \cite{chen2017distributed}, coordinate-wise median or trimmed mean \cite{yin2018byzantine}, and recent high dimensional methods \cite{zhu2023byzantine}, are not shown here due to unfair comparison.

\begin{figure}[h!]
	\centering	
		\subfigure[Random attack, $d=20$.]{\includegraphics[width=0.48\linewidth,height=0.4\linewidth]{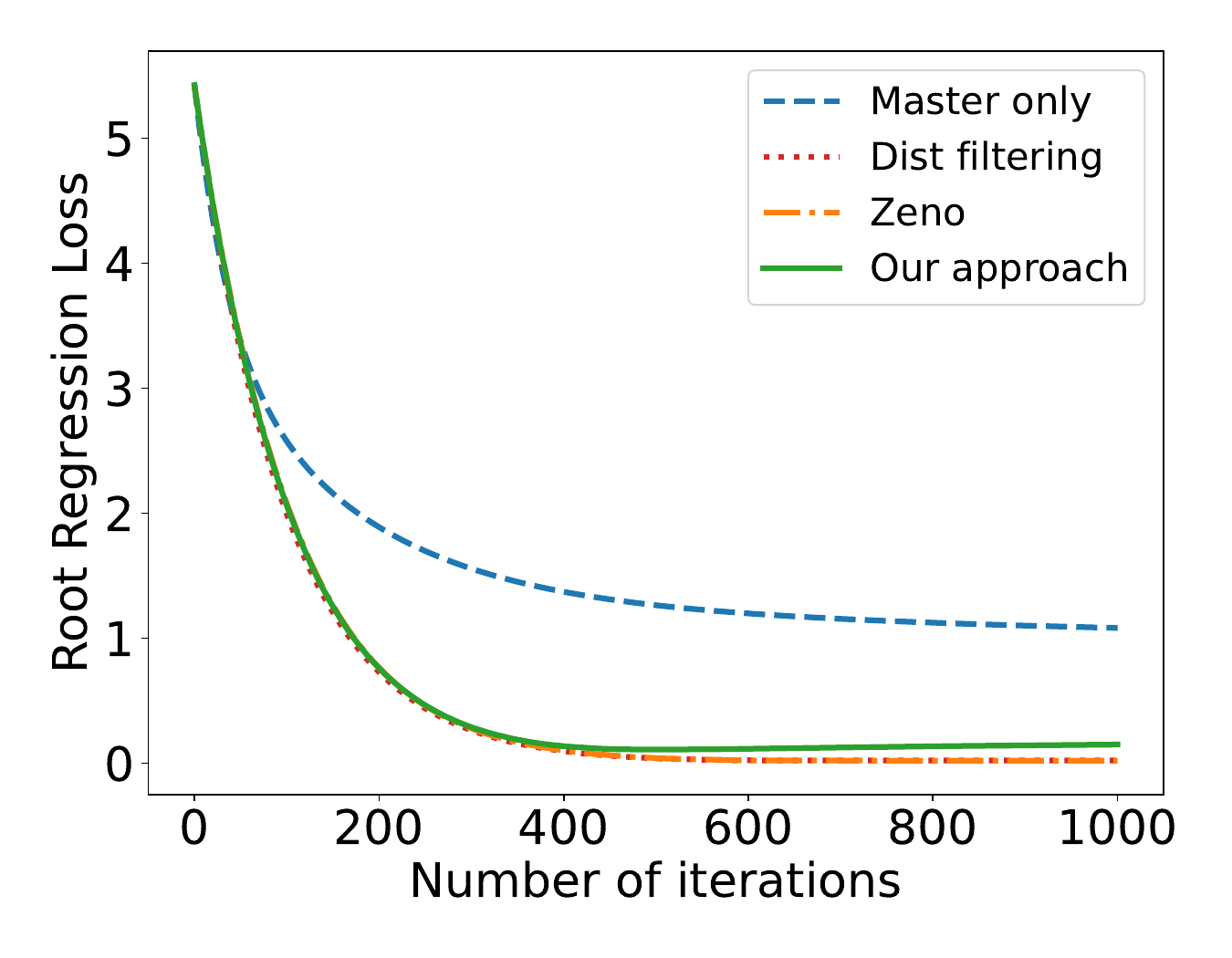}}
		\subfigure[Sign-flip attack, $d=20$.]{\includegraphics[width=0.48\linewidth,height=0.4\linewidth]{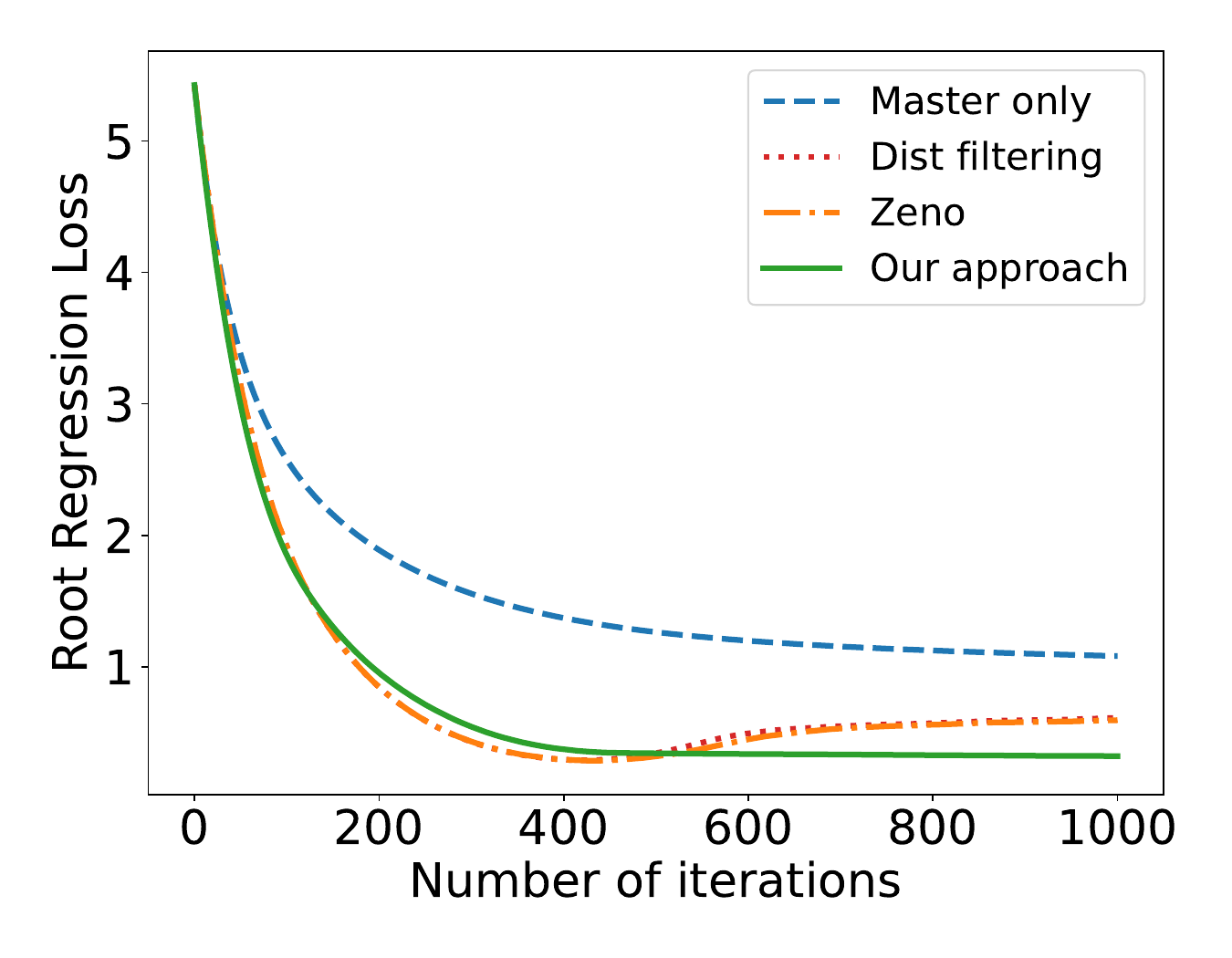}}
		\subfigure[Random attack, $d=50$.]{\includegraphics[width=0.48\linewidth,height=0.4\linewidth]{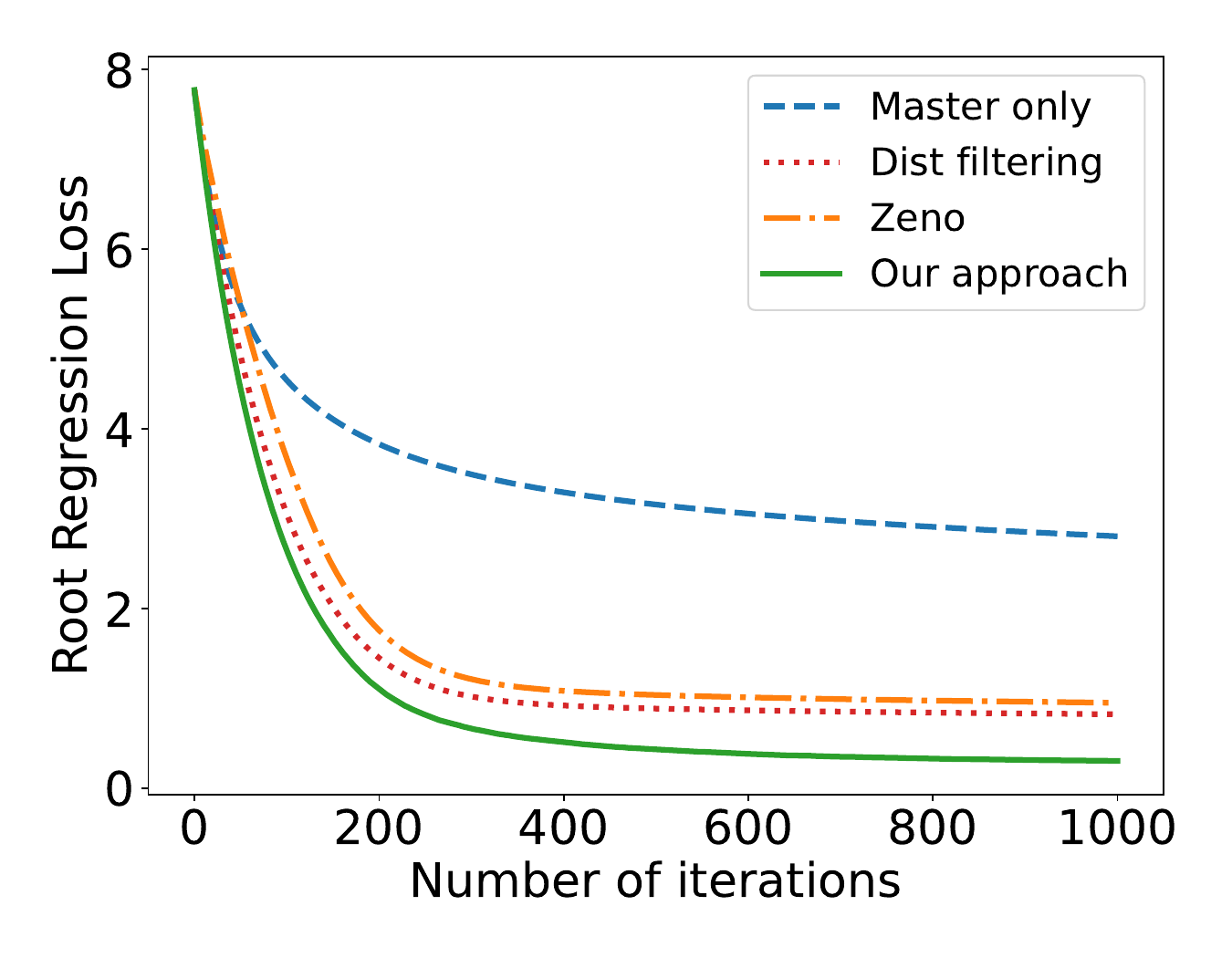}}
		\subfigure[Sign-flip attack, $d=50$.]{\includegraphics[width=0.48\linewidth,height=0.4\linewidth]{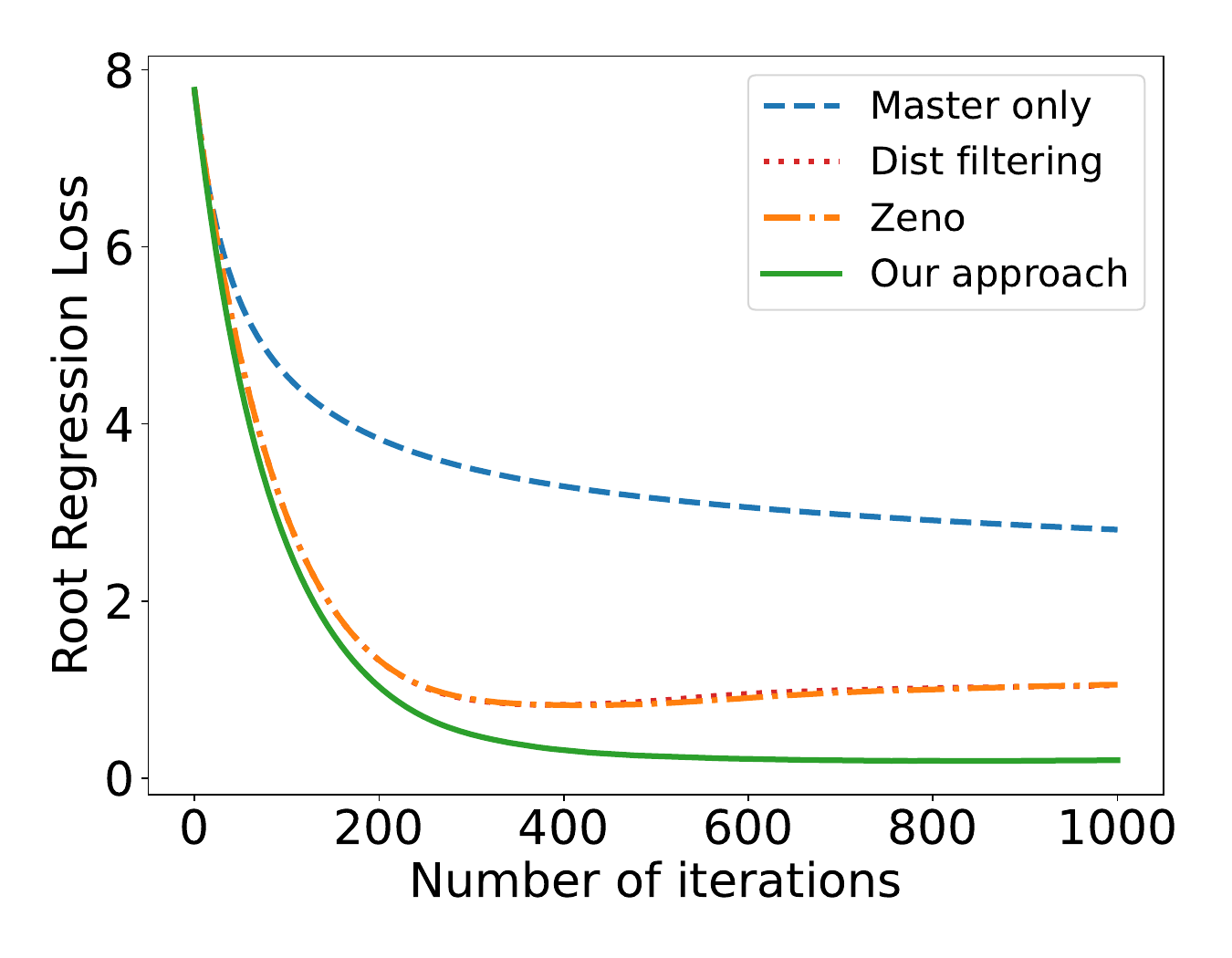}}
		\subfigure[Random attack, $d=100$.]{\includegraphics[width=0.48\linewidth,height=0.4\linewidth]{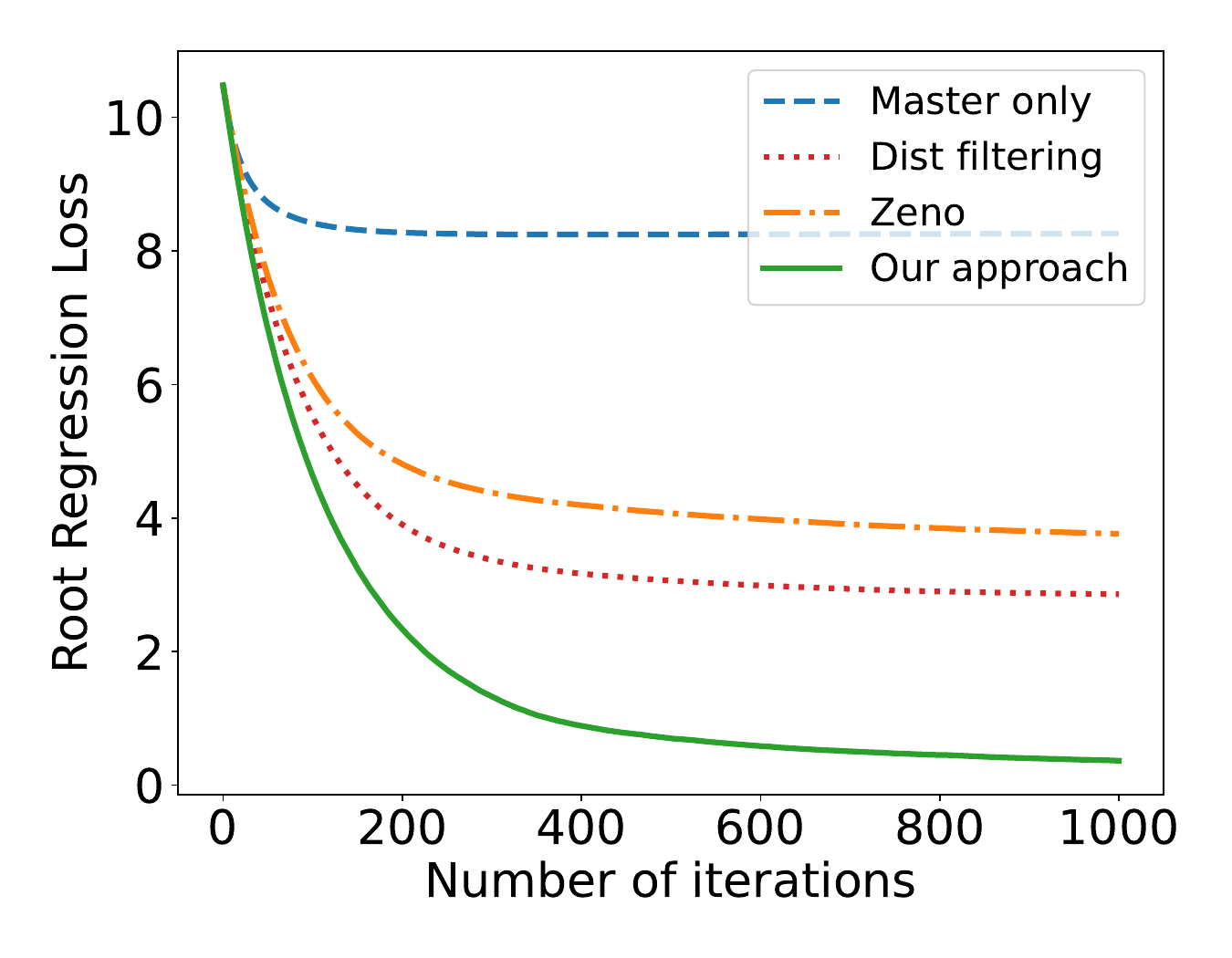}}
		\subfigure[Sign-flip attack, $d=100$.]{\includegraphics[width=0.48\linewidth,height=0.4\linewidth]{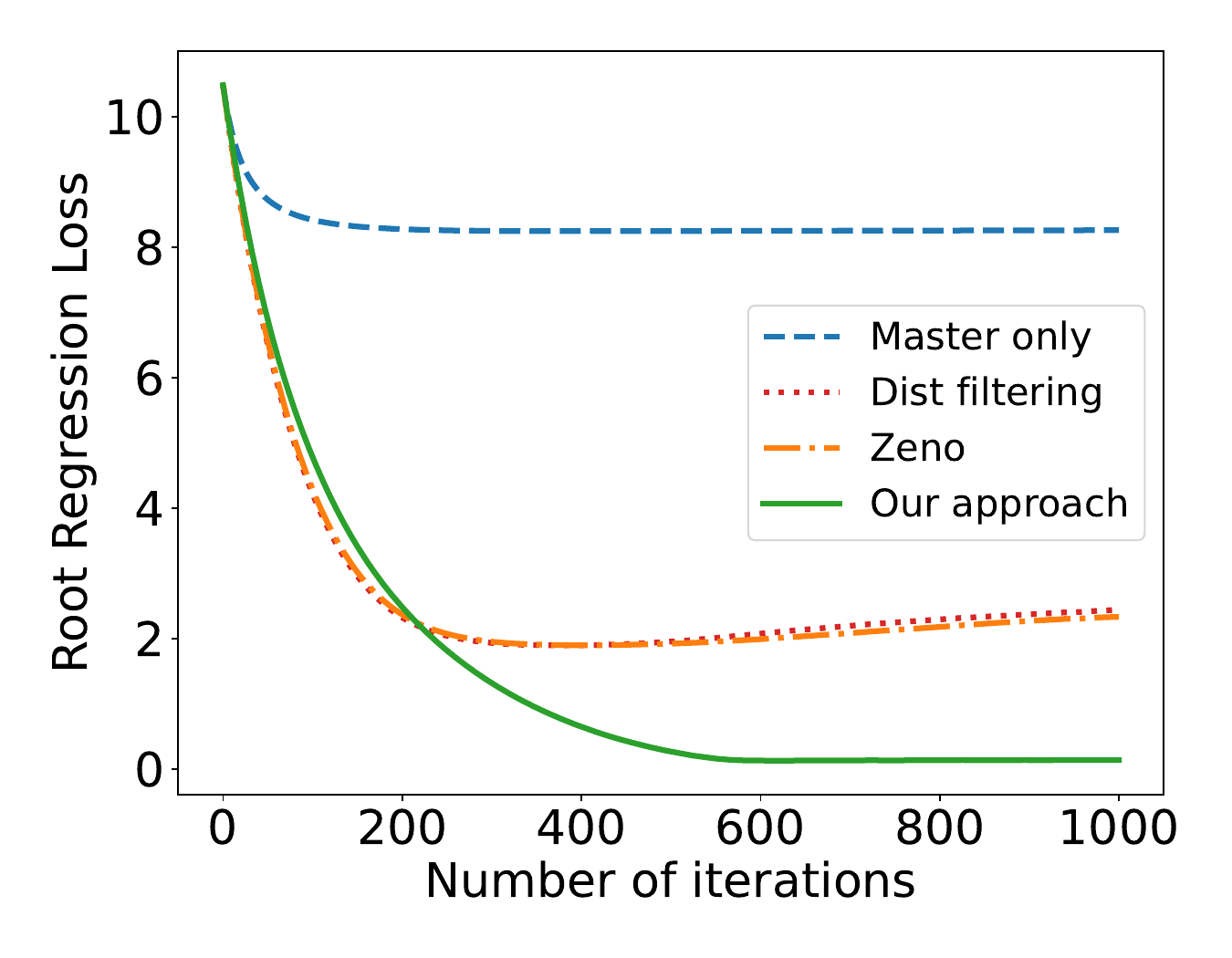}}
	\caption{Experiment results with synthesized data under linear model, with $q/m=0.8$.}\label{fig:syn1}
    \vspace{-5mm}
\end{figure}

\subsection{Synthesized Data}

We first run experiments with distributed linear regression. The model is
\begin{eqnarray}
	V_i=\langle \mathbf{U}_i, \mathbf{w}^*\rangle +W_i,
	\label{eq:syntheticmodel}
\end{eqnarray}
in which $\mathbf{U}_i, \mathbf{w}^*\in \mathbb{R}^d$, $\mathbf{w}^*$ is the true parameter vector, $W_i$ is the i.i.d noise following normal distribution $\mathcal{N}(0,1)$. Now we randomly generate $\mathbf{w}^*$ with each coordinate following normal distribution $\mathcal{N}(0,2)$. We generate $N=50,000$ samples. For each sample, all components of $\mathbf{U}_i$ are i.i.d following $\mathcal{N}(0,1)$, and $V_i$ are calculated using \eqref{eq:syntheticmodel}. These samples are distributed into $m= 500$ worker machines. The learning rate is set to be $\eta = 5\times 10^{-3}$. Moreover, there are $N_A=50$ auxiliary clean samples. The results with $q=400$ (i.e. $\alpha = 0.2$) and $q=100$ (i.e. $\alpha = 0.8$) Byzantine machines are shown in Figure~\ref{fig:syn1} and Figure~\ref{fig:syn2}, respectively. We set $p=5$ for $q=400$, and $p=2$ for $q=100$. In both experiments, we show the results with different dimensionality $d=20, 50, 100$. We run experiments with both random attack and sign-flip attack, in which $\sigma_{attack} = 1$ for the former one.
\begin{figure}[h!]
	\centering	
\subfigure[Random attack, $d=20$.]{\includegraphics[width=0.48\linewidth,height=0.4\linewidth]{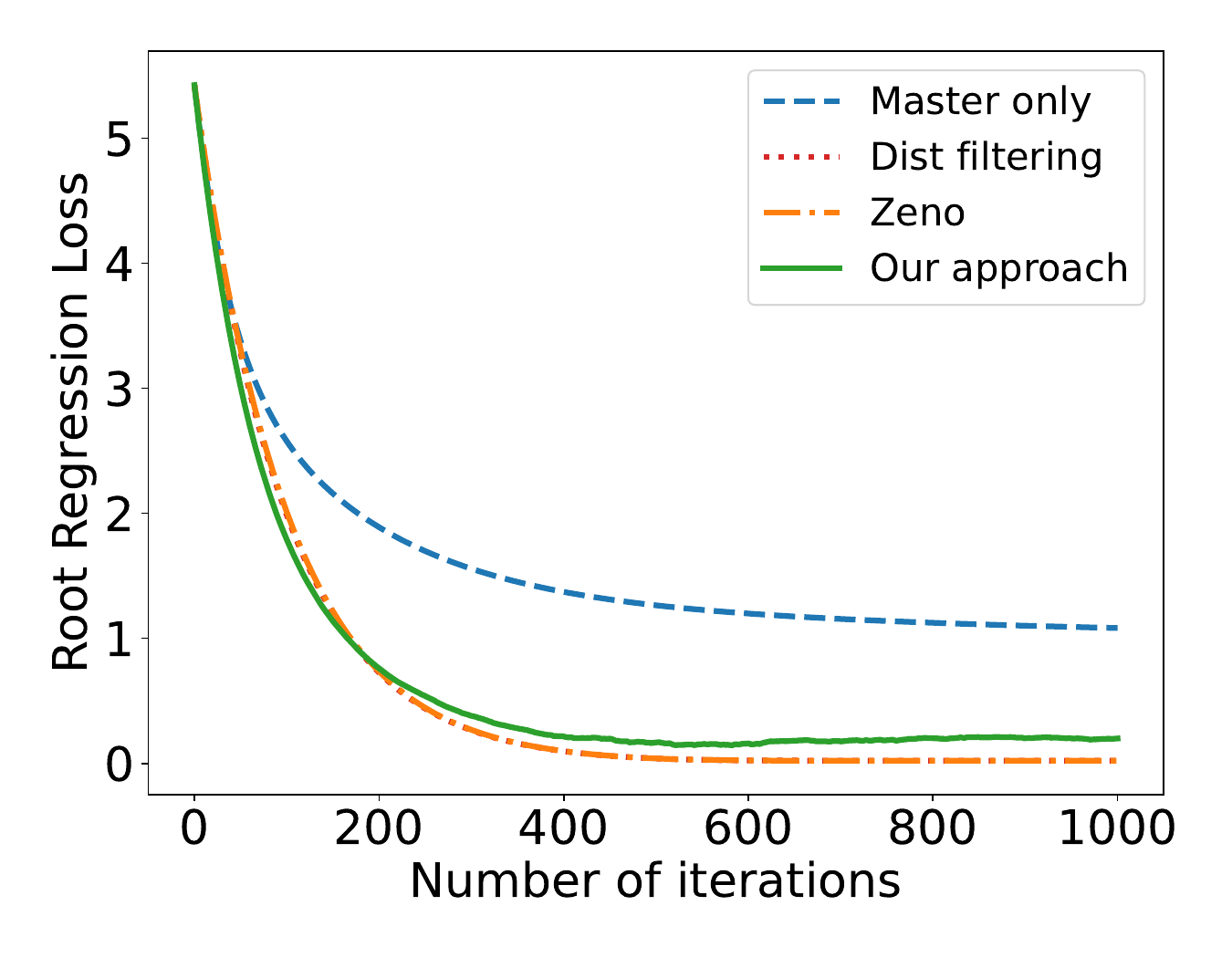}}
\subfigure[Sign-flip attack, $d=20$.]{\includegraphics[width=0.48\linewidth,height=0.4\linewidth]{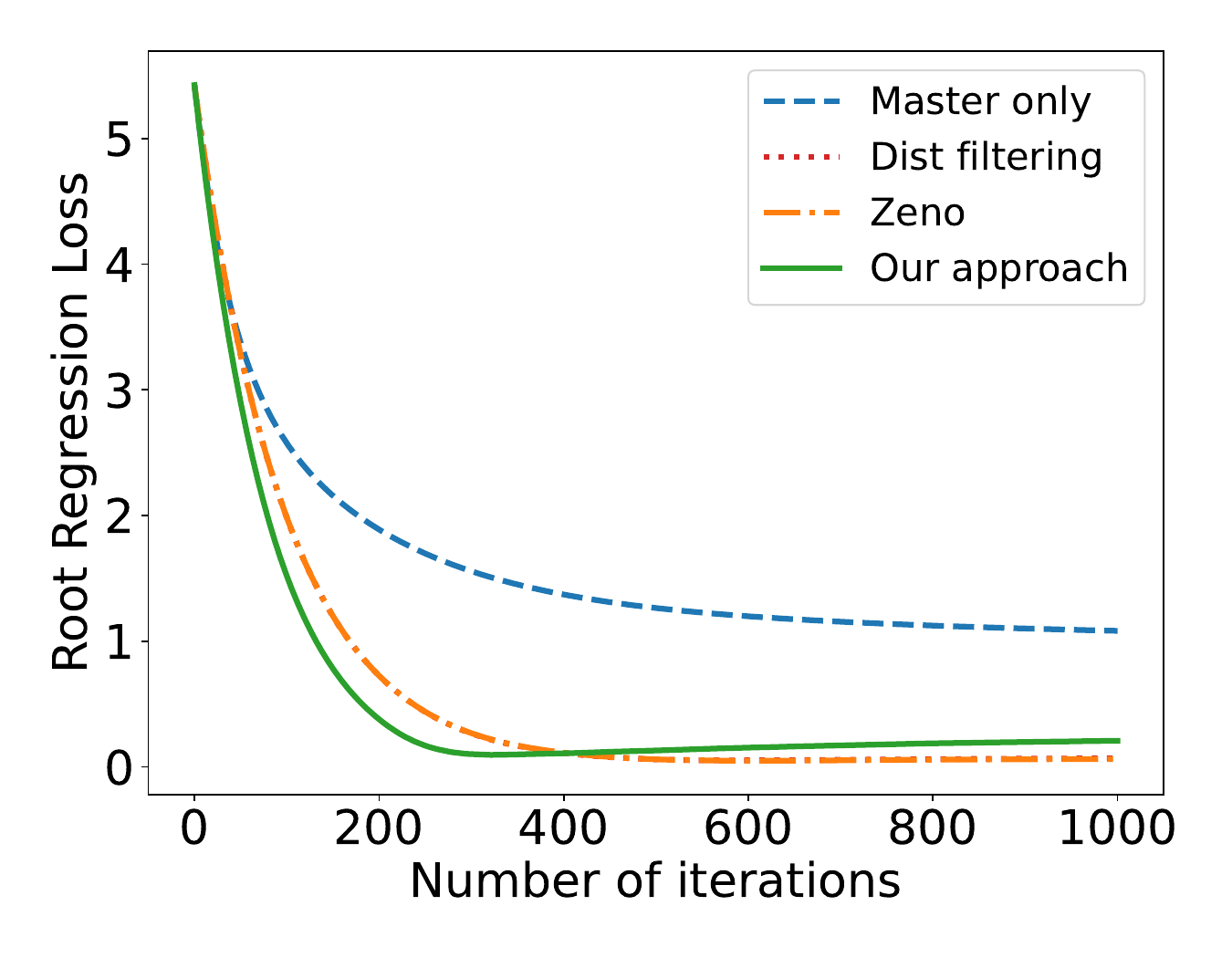}}
\subfigure[Random attack, $d=50$.]{\includegraphics[width=0.48\linewidth,height=0.4\linewidth]{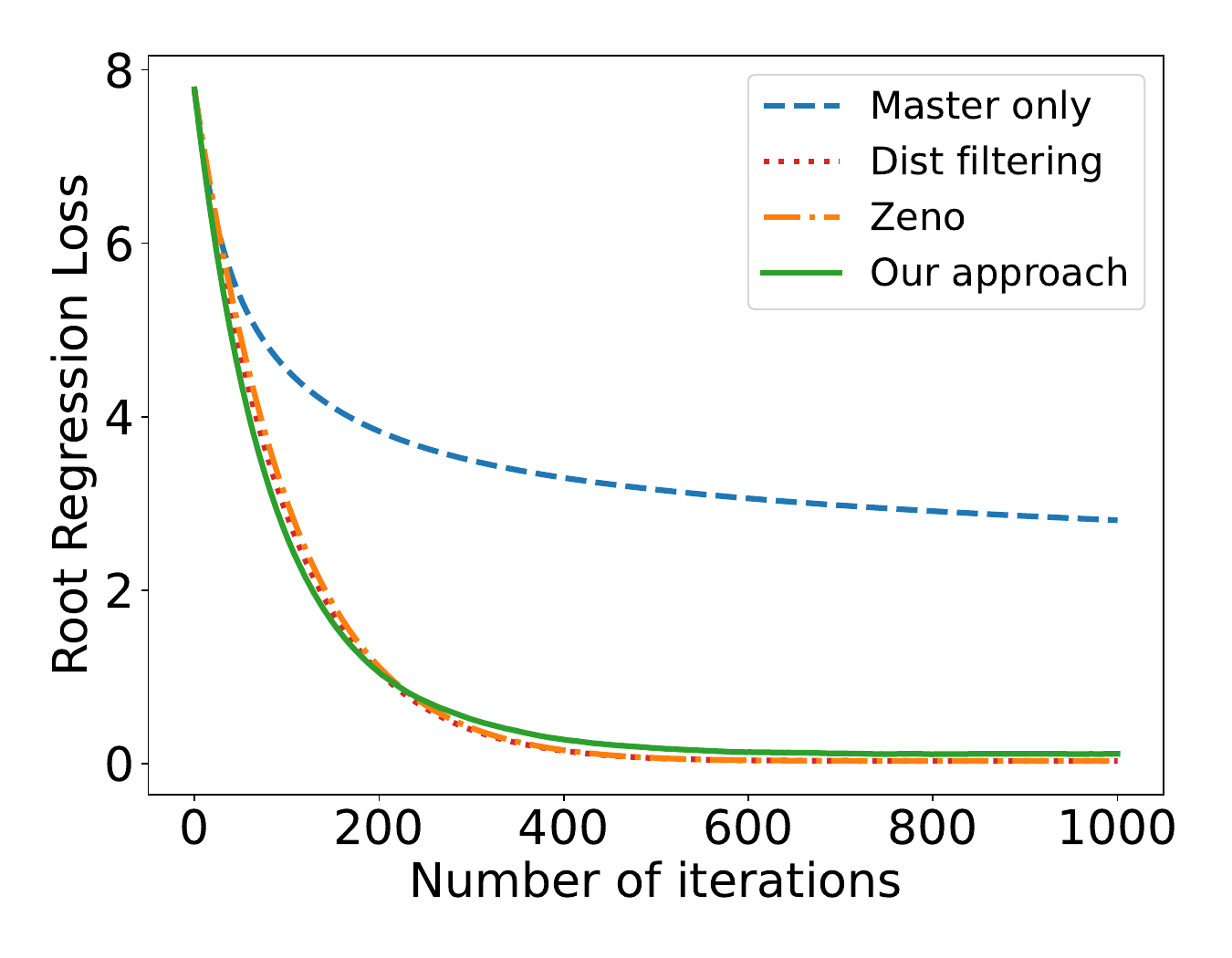}}
\subfigure[Sign-flip attack, $d=50$.]{\includegraphics[width=0.48\linewidth,height=0.4\linewidth]{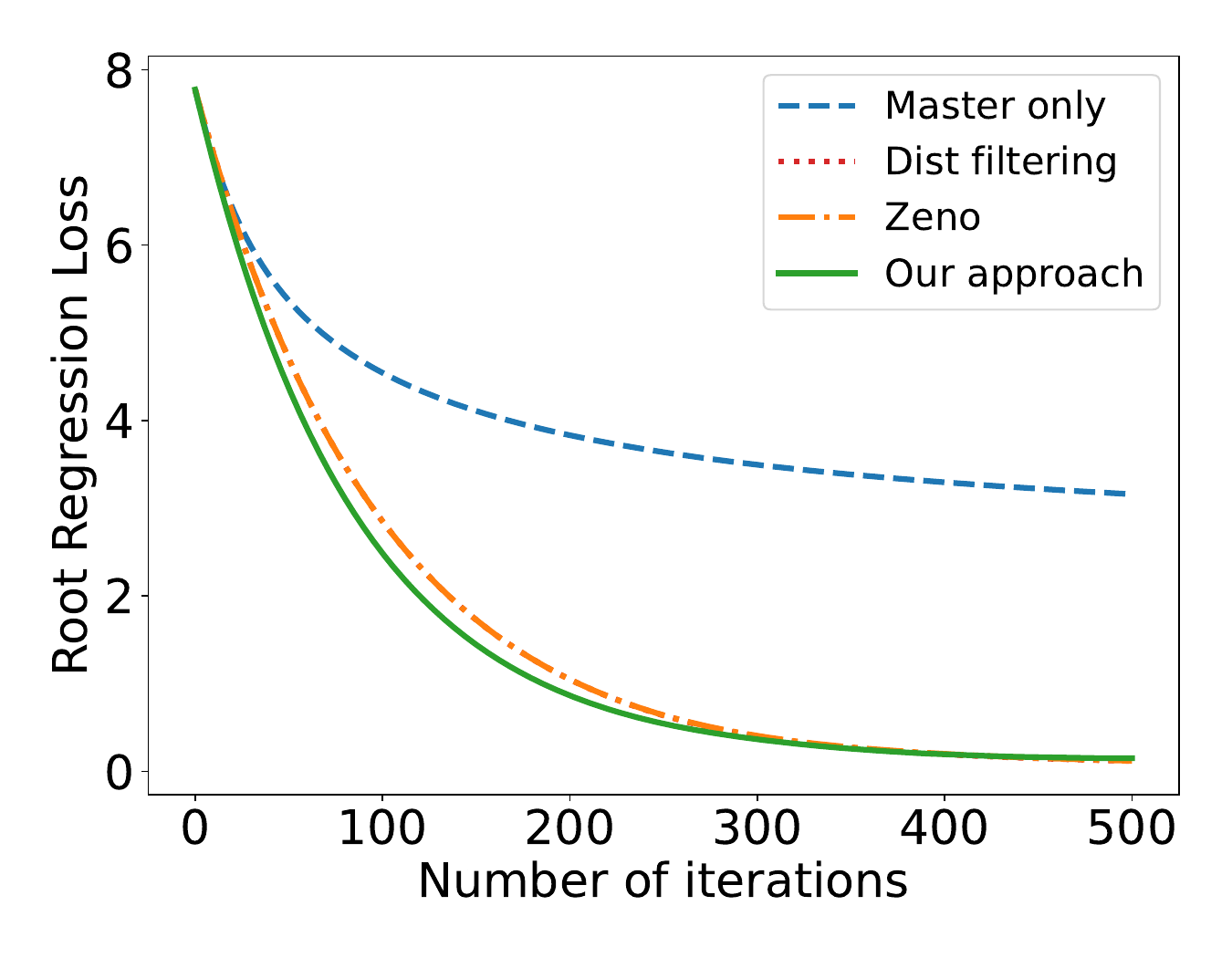}}
\subfigure[Random attack, $d=100$.]{\includegraphics[width=0.48\linewidth,height=0.4\linewidth]{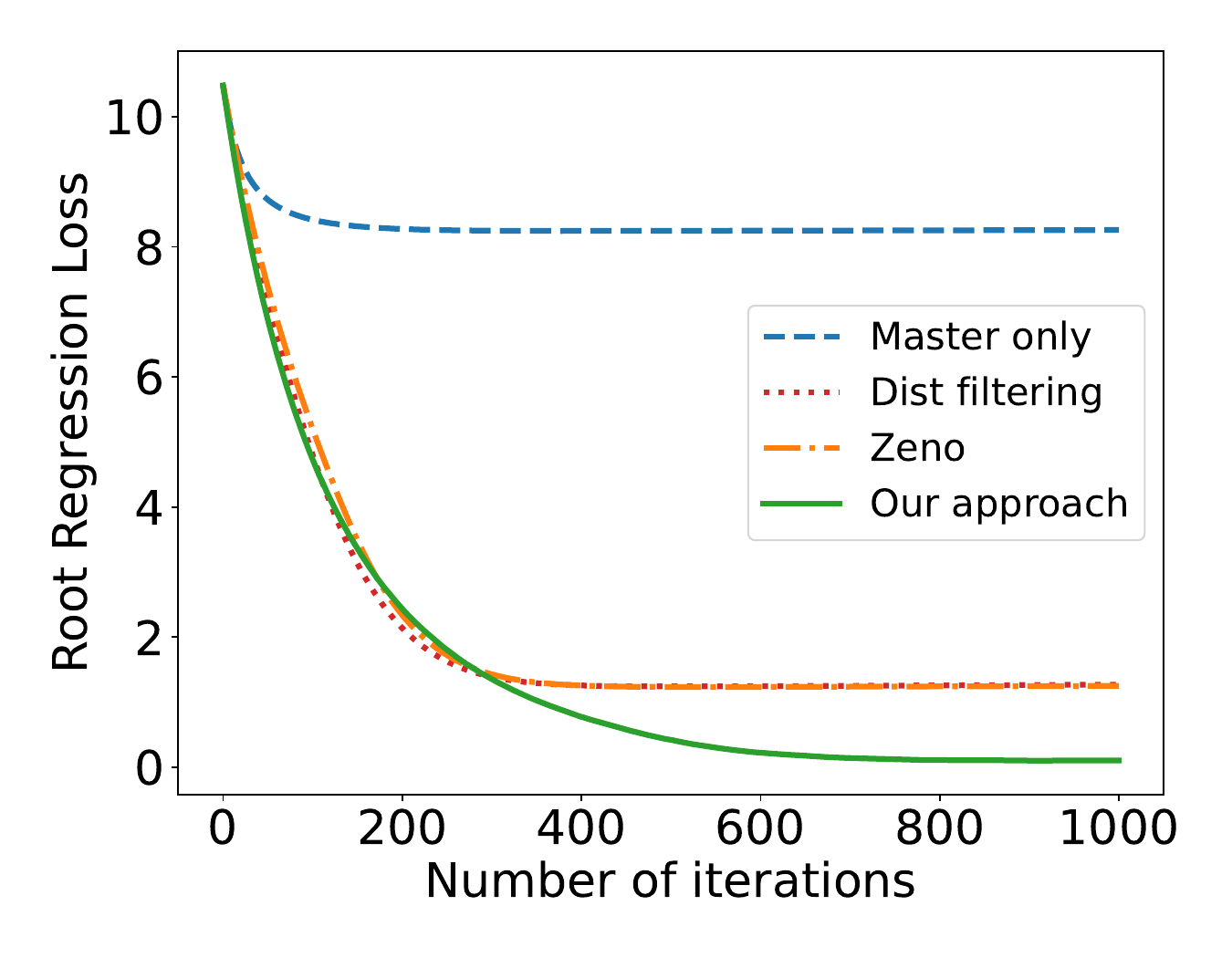}}
\subfigure[Sign-flip attack, $d=100$.]{\includegraphics[width=0.48\linewidth,height=0.4\linewidth]{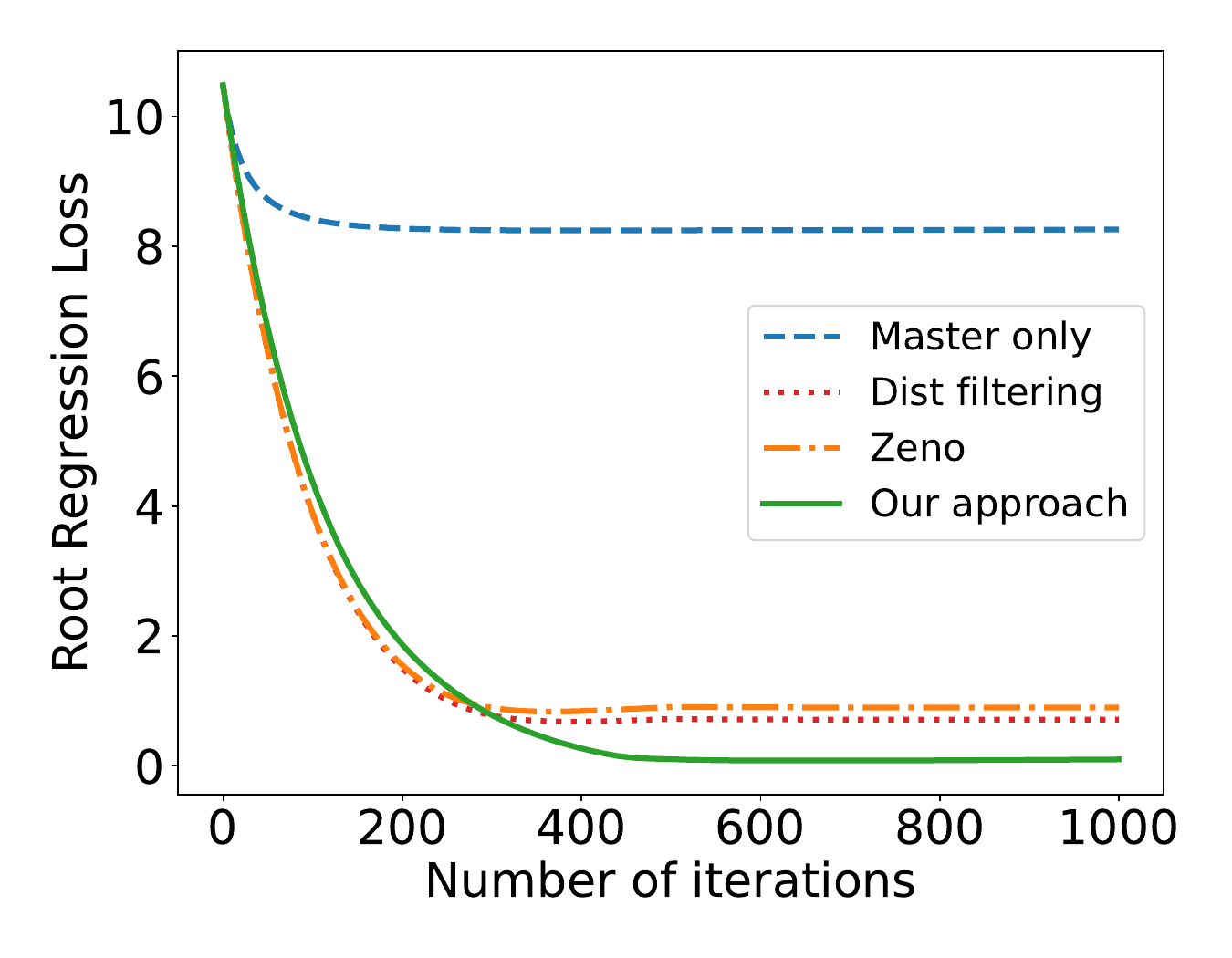}}
	\caption{Experiment results with synthesized data under linear model, with $q/m=0.2$.}\label{fig:syn2}
    \vspace{-3mm}
\end{figure}
\begin{figure}[h!]
	\centering
	\subfigure[Random attack, $q/m=0.8$.]{\includegraphics[width=0.48\linewidth,height=0.42\linewidth]{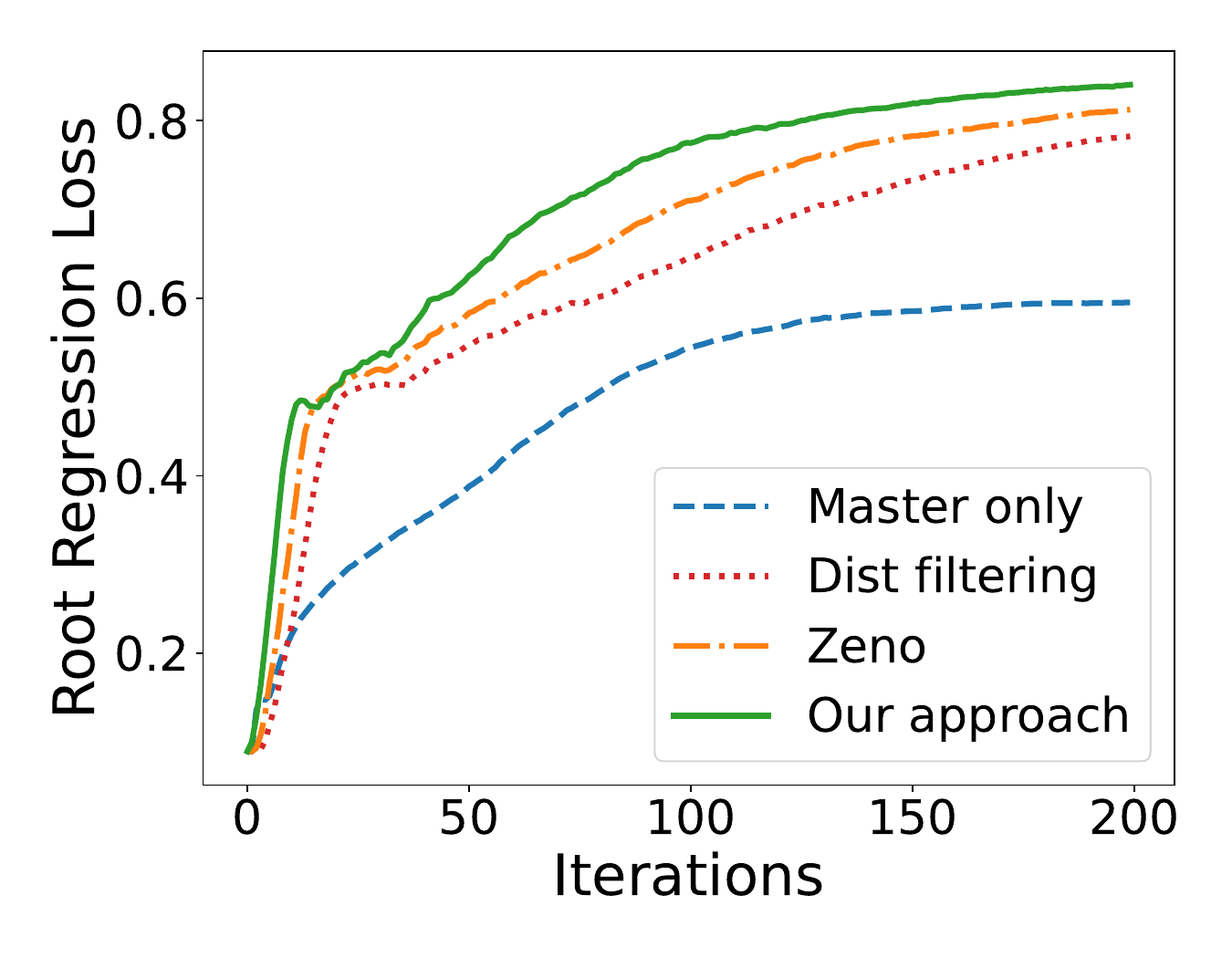}}
	\subfigure[Sign-flip attack, $q/m=0.8$.]{\includegraphics[width=0.48\linewidth,height=0.42\linewidth]{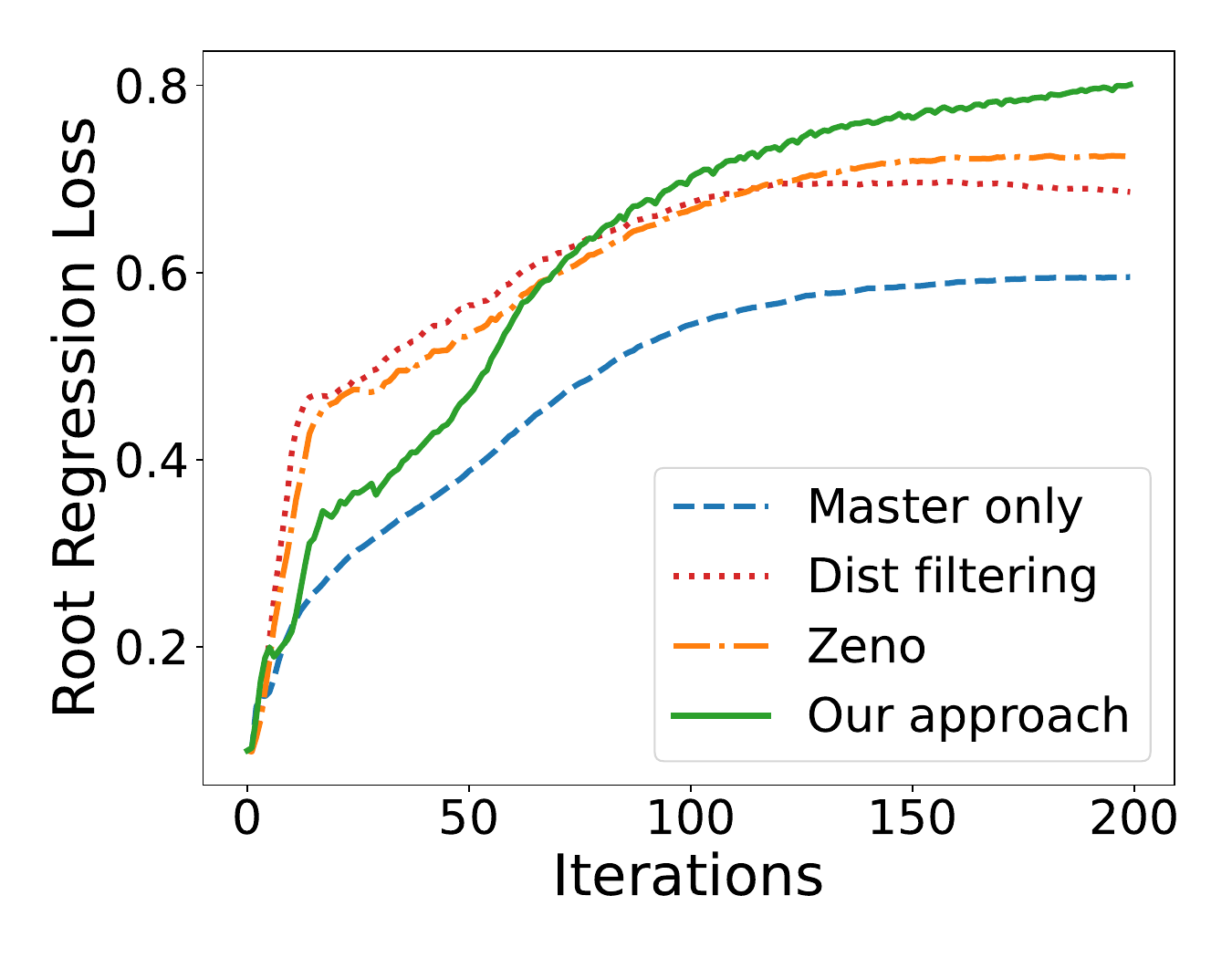}}
\subfigure[Random attack, $q/m=0.2$.]{\includegraphics[width=0.48\linewidth,height=0.42\linewidth]{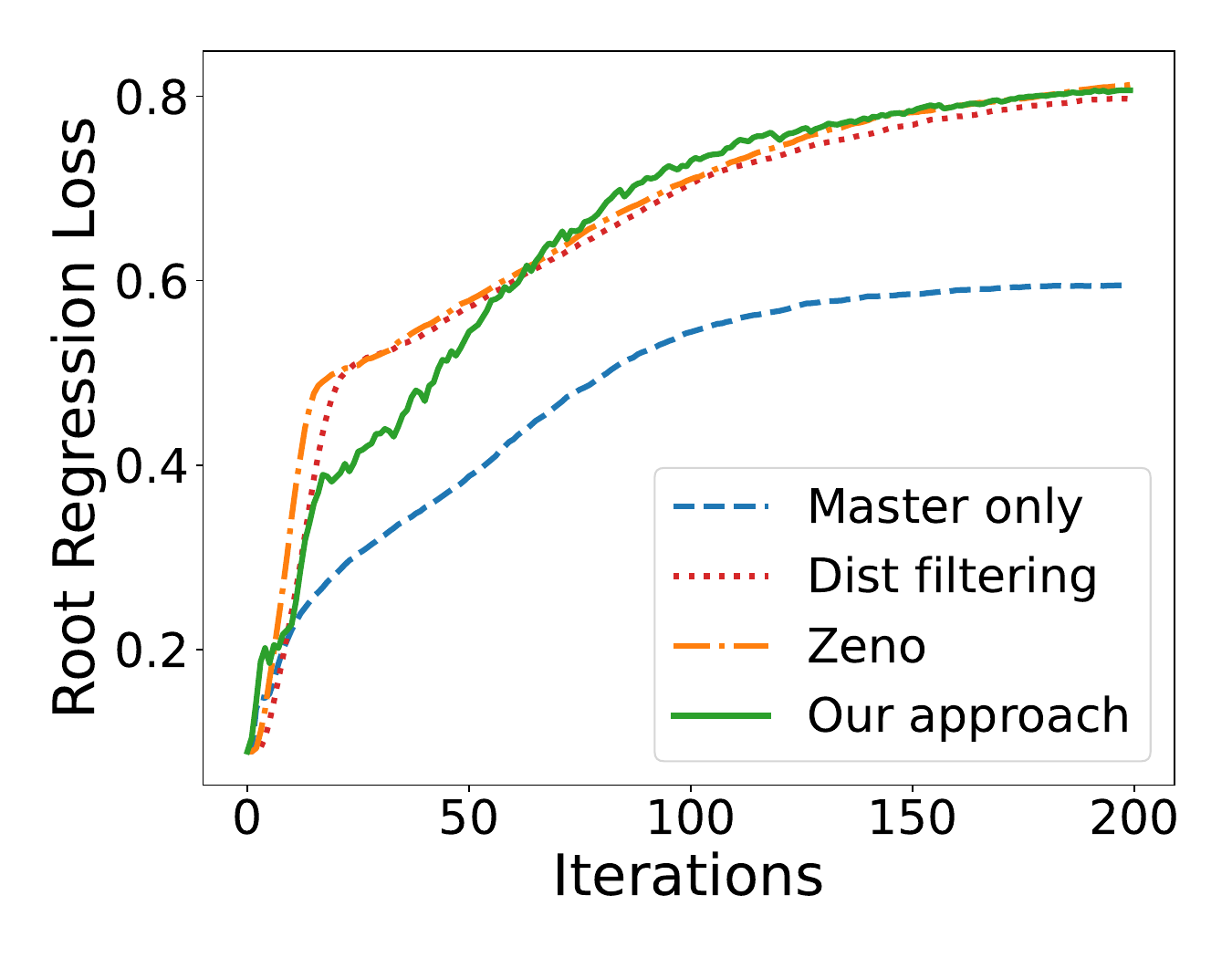}}
	\subfigure[Sign-flip attack, $q/m=0.2$.]{\includegraphics[width=0.48\linewidth,height=0.42\linewidth]{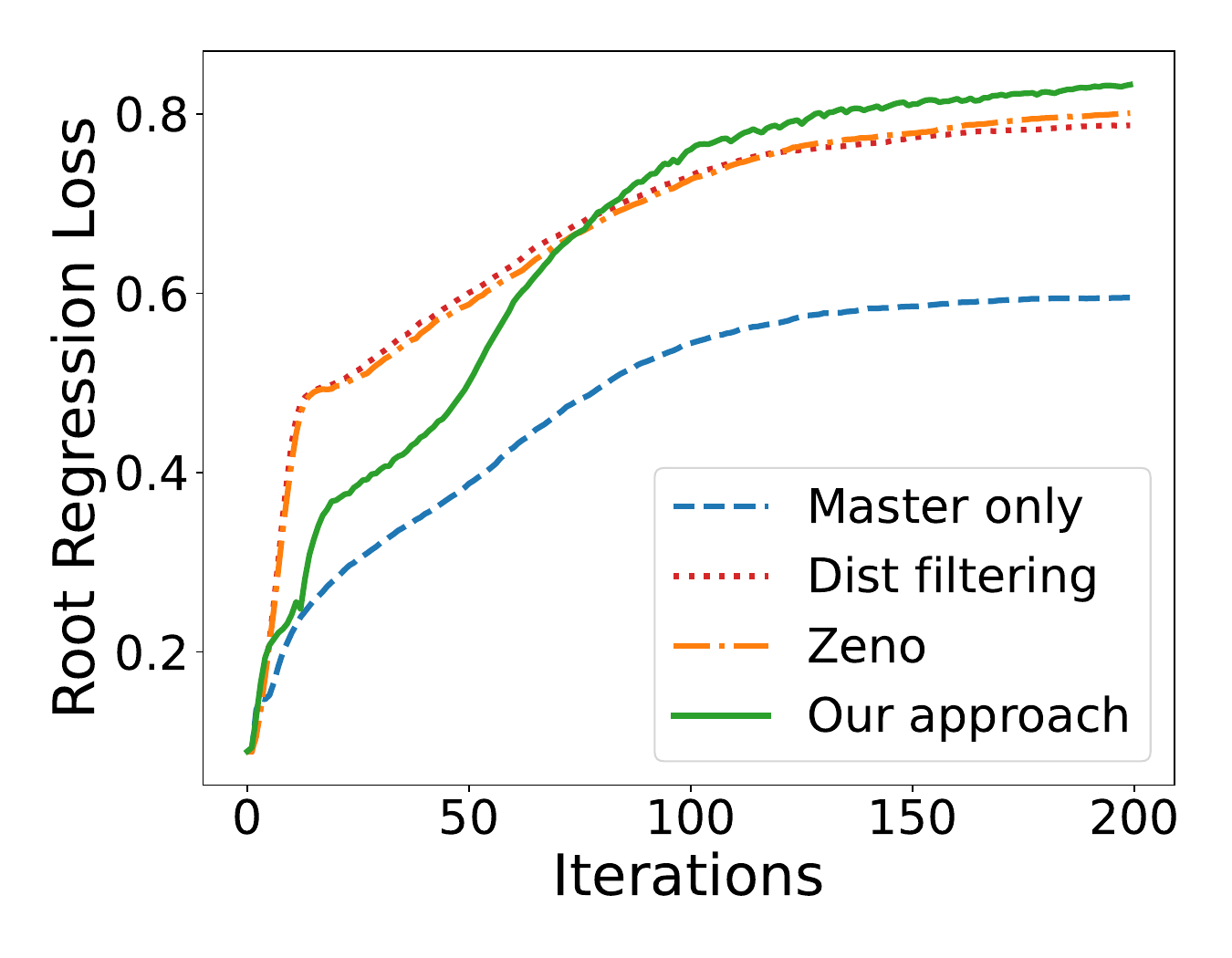}}
	\caption{Experiment results with MNIST data.}\label{fig:mnist}
	\vspace{-4mm}
\end{figure}

Figure \ref{fig:syn1} (a) and (b) show that with $80\%$ Byzantine workers, if $d=20$, then the loss curves of our method (green solid curves in the figure) are nearly the same as previous methods, including distance based filtering and Zeno. With $d=50$, our method begins to outperform existing methods, which is shown in (c) and (d). If the dimensionality is further increased to $d=100$, then the advantage of our method becomes clearer. This result agrees well with our theoretical analysis.

Now we move on to the case with only $20\%$ worker machines being attacked. The results are shown in Figure \ref{fig:syn2}. The result shows that even in the case with less than half machines being attacked, our method still performs comparable to or better than existing approaches.

\subsection{Real data}
Here we use MNIST dataset \cite{lecun1998mnist} to test the performance of robust gradient aggregators. In MNIST dataset, there are  $60,000$ training images and $12,000$ testing images, with each image has a size of $28\times 28$.

The model is a neural network with one hidden layer between input and output. The size of hidden layer is $32$. In each experiment, we first randomly select $N_A=50$ samples as the auxiliary clean dataset. The remaining samples are distributed into $m=500$ worker machines evenly. The gradients are obtained by backpropagating cross entropy loss function. The results for $q=400$ (i.e. $\alpha = 0.2$) and $q=100$ (i.e. $\alpha = 0.8)$ under both random and sign-flip attack are shown in Figure \ref{fig:mnist}, respectively. For random attack, $\sigma_{attack} = 0.01$.

From Figure \ref{fig:mnist} (a)-(d), our new method outperforms existing approaches in general, for both random attack and sign-flip attack, especially with $q/m=0.8$. For the case with $q/m=0.2$, our method is slightly better than others, but the advantage becomes less obvious. This result is natural since our design is primarily for the case with more than half workers being Byzantine.
Furthermore, we would like to remark that Zeno appears to perform better than distance-based filtering. Our understanding is that the gradient descent score \eqref{eq:sdscore} actually selects gradient vectors whose directions are close to the gradient from auxiliary data. In neural networks with complex loss landscape, this rule may be better than filtering based entirely on distances.

\section{Conclusion}\label{sec:conc}
This paper solves the problem of high dimensional distributed learning problem under arbitrary number of Byzantine attackers. Firstly, we have proposed a new method for semi-verified mean estimation, which combines a small clean dataset and a large corrupted dataset to estimate the statistical mean. We have also conducted theoretical analysis under both additive and strong contamination model. The results show that the new method is minimax rate optimal. We have then applied the semi-verified mean estimation method into the aggregator function in distributed learning. Compared with existing methods, the performance of our method is nearly the same as existing approaches for low dimensional problems. With higher dimensionality, our method performs significantly better. Finally, numerical results validate the effectiveness of our new method.

\section{Acknowledgments}
This work is supported by National Natural Science Foundation of China under Grants (62072106), the Open Research Projects of the Key Laboratory of Blockchain Technology and Data Security of the Ministry of Industry and Information Technology for the year 2025 (Grant KT20250015), the Natural Science Foundation of China (General Program, No.62571529), the Science and Technology Innovation Platform Project of Fujian Province, China (Grant 2023-P-003), and the Industry-University Cooperative Project for Colleges and Universities of Fujian Province, China (Grant 2024H6027).

\bibliography{aaai2026}

@article{li2020federated,
	title={Federated learning: Challenges, methods, and future directions},
	author={Li, Tian and Sahu, Anit Kumar and Talwalkar, Ameet and Smith, Virginia},
	journal={IEEE signal processing magazine},
	volume={37},
	number={3},
	pages={50--60},
	year={2020},
	publisher={IEEE}
}

@inproceedings{shejwalkar2021manipulating,
	title={Manipulating the byzantine: Optimizing model poisoning attacks and defenses for federated learning},
	author={Shejwalkar, Virat and Houmansadr, Amir},
	booktitle={NDSS},
	year={2021}
}

@article{su2018securing,
	title={Securing distributed machine learning in high dimensions},
	author={Su, Lili and Xu, Jiaming},
	journal={arXiv preprint arXiv:1804.10140},
	pages={1536--1233},
	year={2018},
	publisher={CoRR}
}

@article{zhang2021survey,
	title={A survey on federated learning},
	author={Zhang, Chen and Xie, Yu and Bai, Hang and Yu, Bin and Li, Weihong and Gao, Yuan},
	journal={Knowledge-Based Systems},
	volume={216},
	pages={106775},
	year={2021},
	publisher={Elsevier}
}

@article{lyu2020threats,
	title={Threats to federated learning: A survey},
	author={Lyu, Lingjuan and Yu, Han and Yang, Qiang},
	journal={arXiv preprint arXiv:2003.02133},
	year={2020}
}

@article{lecun1998mnist,
	title={The MNIST database of handwritten digits},
	author={LeCun, Yann},
	journal={http://yann. lecun. com/exdb/mnist/},
	year={1998}
}

@article{huber1964robust,
	title={Robust Estimation of a Location Parameter},
	author={Huber, Peter J},
	journal={The Annals of Mathematical Statistics},
	pages={73--101},
	year={1964},
	publisher={JSTOR}
}

@book{huber2004robust,
	title={Robust statistics},
	author={Huber, Peter J},
	volume={523},
	year={2004},
	publisher={John Wiley \& Sons}
}

@inproceedings{tukey1975mathematics,
	title={Mathematics and the picturing of data},
	author={Tukey, John W},
	booktitle={Proceedings of the International Congress of Mathematicians, Vancouver, 1975},
	volume={2},
	pages={523--531},
	year={1975}
}

@inproceedings{lai2016agnostic,
	title={Agnostic estimation of mean and covariance},
	author={Lai, Kevin A and Rao, Anup B and Vempala, Santosh},
	booktitle={2016 IEEE 57th Annual Symposium on Foundations of Computer Science (FOCS)},
	pages={665--674},
	year={2016},
	organization={IEEE}
}

@inproceedings{cheng2019high,
	title={High-dimensional robust mean estimation in nearly-linear time},
	author={Cheng, Yu and Diakonikolas, Ilias and Ge, Rong},
	booktitle={Proceedings of the thirtieth annual ACM-SIAM symposium on discrete algorithms},
	pages={2755--2771},
	year={2019},
	organization={SIAM}
}

@inproceedings{balakrishnan2017computationally,
	title={Computationally efficient robust sparse estimation in high dimensions},
	author={Balakrishnan, Sivaraman and Du, Simon S and Li, Jerry and Singh, Aarti},
	booktitle={Conference on Learning Theory},
	pages={169--212},
	year={2017},
	organization={PMLR}
}

@inproceedings{hopkins2018mixture,
	title={Mixture models, robustness, and sum of squares proofs},
	author={Hopkins, Samuel B and Li, Jerry},
	booktitle={Proceedings of the 50th Annual ACM SIGACT Symposium on Theory of Computing},
	pages={1021--1034},
	year={2018}
}

@inproceedings{mcmahan2017communication,
	title={Communication-efficient learning of deep networks from decentralized data},
	author={McMahan, Brendan and Moore, Eider and Ramage, Daniel and Hampson, Seth and y Arcas, Blaise Aguera},
	booktitle={Artificial intelligence and statistics},
	pages={1273--1282},
	year={2017},
	organization={PMLR}
}

@inproceedings{diakonikolas2016robust,
	title={Robust Estimators in High Dimensions without the Computational Intractability},
	author={Diakonikolas, Ilias and Kamath, Gautam and Kane, Daniel M and Li, Jerry and Moitra, Ankur and Stewart, Alistair},
	booktitle={57th Annual Symposium on Foundations of Computer Science},
	pages={655--664},
	year={2016},
}

@article{li2020review,
	title={A review of applications in federated learning},
	author={Li, Li and Fan, Yuxi and Tse, Mike and Lin, Kuo-Yi},
	journal={Computers \& Industrial Engineering},
	volume={149},
	pages={106854},
	year={2020},
	publisher={Elsevier}
}

@inproceedings{guerraoui2018hidden,
	title={The hidden vulnerability of distributed learning in byzantium},
	author={Guerraoui, Rachid and Rouault, S{\'e}bastien and others},
	booktitle={International Conference on Machine Learning},
	pages={3521--3530},
	year={2018},
	organization={PMLR}
}

@article{cao2020fltrust,
	title={Fltrust: Byzantine-robust federated learning via trust bootstrapping},
	author={Cao, Xiaoyu and Fang, Minghong and Liu, Jia and Gong, Neil Zhenqiang},
	journal={arXiv preprint arXiv:2012.13995},
	year={2020}
}

@article{hampel1971general,
	title={A general qualitative definition of robustness},
	author={Hampel, Frank R},
	journal={The annals of mathematical statistics},
	volume={42},
	number={6},
	pages={1887--1896},
	year={1971},
	publisher={Institute of Mathematical Statistics}
}

@inproceedings{diakonikolas2017being,
	title={Being robust (in high dimensions) can be practical},
	author={Diakonikolas, Ilias and Kamath, Gautam and Kane, Daniel M and Li, Jerry and Moitra, Ankur and Stewart, Alistair},
	booktitle={International Conference on Machine Learning},
	pages={999--1008},
	year={2017},
	organization={PMLR}
}

@inproceedings{diakonikolas2020outlier,
	title={Outlier robust mean estimation with subgaussian rates via stability},
	author={Diakonikolas, Ilias and Kane, Daniel M and Pensia, Ankit},
	booktitle={Advances in Neural Information Processing Systems},
	pages={1830--1840},
	year={2020},
	month={Dec}
}

@phdthesis{steinhardt2018robust,
	title={Robust learning: Information theory and algorithms},
	author={Steinhardt, Jacob},
	year={2018},
	publisher={Stanford University}
}

@article{zhu2022a,
	title={Generalized resilience and robust statistics},
	author={Zhu, Banghua and Jiao, Jiantao and Steinhardt, Jacob},
	journal={The Annals of Statistics},
	volume={50},
	number={4},
	pages={2256--2283},
	year={2022},
	publisher={Institute of Mathematical Statistics}
}

@article{zhu2022b,
	title={Robust estimation via generalized quasi-gradients},
	author={Zhu, Banghua and Jiao, Jiantao and Steinhardt, Jacob},
	journal={Information and Inference: A Journal of the IMA},
	volume={11},
	number={2},
	pages={581--636},
	year={2022},
	publisher={Oxford University Press}
}

@inproceedings{zhu2022c,
	title={Robust Estimation for Non-parametric Families via Generative Adversarial Networks},
	author={Zhu, Banghua and Jiao, Jiantao and Jordan, Michael I},
	booktitle={2022 IEEE International Symposium on Information Theory (ISIT)},
	pages={1100--1105},
	year={2022},
	organization={IEEE}
}

@book{diakonikolas2023algorithmic,
	title={Algorithmic high-dimensional robust statistics},
	author={Diakonikolas, Ilias and Kane, Daniel M},
	year={2023},
	publisher={Cambridge University Press}
}

@inproceedings{bagdasaryan2020backdoor,
	title={How to backdoor federated learning},
	author={Bagdasaryan, Eugene and Veit, Andreas and Hua, Yiqing and Estrin, Deborah and Shmatikov, Vitaly},
	booktitle={International conference on artificial intelligence and statistics},
	pages={2938--2948},
	year={2020},
	organization={PMLR}
}

@inproceedings{bhagoji2019analyzing,
	title={Analyzing federated learning through an adversarial lens},
	author={Bhagoji, Arjun Nitin and Chakraborty, Supriyo and Mittal, Prateek and Calo, Seraphin},
	booktitle={International Conference on Machine Learning},
	pages={634--643},
	year={2019},
	organization={PMLR}
}

@inproceedings{fang2020local,
	title={Local model poisoning attacks to Byzantine-Robust federated learning},
	author={Fang, Minghong and Cao, Xiaoyu and Jia, Jinyuan and Gong, Neil},
	booktitle={29th USENIX security symposium (USENIX Security 20)},
	pages={1605--1622},
	year={2020}
}

@article{sun2021data,
	title={Data poisoning attacks on federated machine learning},
	author={Sun, Gan and Cong, Yang and Dong, Jiahua and Wang, Qiang and Lyu, Lingjuan and Liu, Ji},
	journal={IEEE Internet of Things Journal},
	volume={9},
	number={13},
	pages={11365--11375},
	year={2021},
	publisher={IEEE}
}

@inproceedings{luo2021feature,
	title={Feature inference attack on model predictions in vertical federated learning},
	author={Luo, Xinjian and Wu, Yuncheng and Xiao, Xiaokui and Ooi, Beng Chin},
	booktitle={2021 IEEE 37th International Conference on Data Engineering (ICDE)},
	pages={181--192},
	year={2021},
	organization={IEEE}
}

@article{lamport1982byzantine,
	title={The Byzantine Generals Problem},
	author={Lamport, Leslie and Shostak, Robert and Pease, Marshall},
	journal={ACM Transactions on Programming Languages and Systems},
	volume={4},
	number={3},
	pages={382--401},
	year={1982}
}

@inproceedings{blanchard2017machine,
	title={Machine learning with adversaries: Byzantine tolerant gradient descent},
	author={Blanchard, Peva and El Mhamdi, El Mahdi and Guerraoui, Rachid and Stainer, Julien},
	booktitle={Advances in neural information processing systems},
	volume={30},
	year={2017}
}

@article{chen2017distributed,
	title={Distributed statistical machine learning in adversarial settings: Byzantine gradient descent},
	author={Chen, Yudong and Su, Lili and Xu, Jiaming},
	journal={Proceedings of the ACM on Measurement and Analysis of Computing Systems},
	volume={1},
	number={2},
	pages={1--25},
	year={2017},
	publisher={ACM New York, NY, USA}
}

@inproceedings{yin2018byzantine,
	title={Byzantine-robust distributed learning: Towards optimal statistical rates},
	author={Yin, Dong and Chen, Yudong and Kannan, Ramchandran and Bartlett, Peter},
	booktitle={International Conference on Machine Learning},
	pages={5650--5659},
	year={2018},
	organization={PMLR}
}

@inproceedings{zhu2023byzantine,
	title={Byzantine-Robust Federated Learning with Optimal Statistical Rates},
	author={Zhu, Banghua and Wang, Lun and Pang, Qi and Wang, Shuai and Jiao, Jiantao and Song, Dawn and Jordan, Michael I},
	booktitle={International Conference on Artificial Intelligence and Statistics},
	pages={3151--3178},
	year={2023},
	organization={PMLR}
}

@inproceedings{diakonikolas2021list,
	title={List-decodable mean estimation in nearly-pca time},
	author={Diakonikolas, Ilias and Kane, Daniel and Kongsgaard, Daniel and Li, Jerry and Tian, Kevin},
	booktitle={Advances in Neural Information Processing Systems},
	volume={34},
	pages={10195--10208},
	year={2021}
}

@inproceedings{charikar2017learning,
	title={Learning from untrusted data},
	author={Charikar, Moses and Steinhardt, Jacob and Valiant, Gregory},
	booktitle={Proceedings of the 49th Annual ACM SIGACT Symposium on Theory of Computing},
	pages={47--60},
	year={2017}
}

@article{tropp2015introduction,
	title={An introduction to matrix concentration inequalities},
	author={Tropp, Joel A and others},
	journal={Foundations and Trends{\textregistered} in Machine Learning},
	volume={8},
	number={1-2},
	pages={1--230},
	year={2015},
	publisher={Now Publishers, Inc.}
}

@article{kairouz2021advances,
	title={Advances and open problems in federated learning},
	author={Kairouz, Peter and McMahan, H Brendan and Avent, Brendan and Bellet, Aur{\'e}lien and Bennis, Mehdi and Bhagoji, Arjun Nitin and Bonawitz, Kallista and Charles, Zachary and Cormode, Graham and Cummings, Rachel and others},
	journal={Foundations and Trends{\textregistered} in Machine Learning},
	volume={14},
	number={1--2},
	pages={1--210},
	year={2021},
	publisher={Now Publishers, Inc.}
}

@inproceedings{cherapanamjeri2020list,
	title={List decodable mean estimation in nearly linear time},
	author={Cherapanamjeri, Yeshwanth and Mohanty, Sidhanth and Yau, Morris},
	booktitle={2020 IEEE 61st Annual Symposium on Foundations of Computer Science (FOCS)},
	pages={141--148},
	year={2020},
	organization={IEEE}
}

@inproceedings{diakonikolas2018list,
	title={List-decodable robust mean estimation and learning mixtures of spherical gaussians},
	author={Diakonikolas, Ilias and Kane, Daniel M and Stewart, Alistair},
	booktitle={Proceedings of the 50th Annual ACM SIGACT Symposium on Theory of Computing},
	pages={1047--1060},
	year={2018}
}

@article{cao2019distributed,
	title={Distributed gradient descent algorithm robust to an arbitrary number of byzantine attackers},
	author={Cao, Xinyang and Lai, Lifeng},
	journal={IEEE Transactions on Signal Processing},
	volume={67},
	number={22},
	pages={5850--5864},
	year={2019},
	publisher={IEEE}
}

@article{regatti2020bygars,
	title={ByGARS: Byzantine SGD with arbitrary number of attackers},
	author={Regatti, Jayanth and Chen, Hao and Gupta, Abhishek},
	journal={arXiv preprint arXiv:2006.13421},
	year={2020}
}

@inproceedings{xie2019zeno,
	title={Zeno: Distributed stochastic gradient descent with suspicion-based fault-tolerance},
	author={Xie, Cong and Koyejo, Sanmi and Gupta, Indranil},
	booktitle={International Conference on Machine Learning},
	pages={6893--6901},
	year={2019},
	organization={PMLR}
}

@inproceedings{xie2020zeno++,
	title={Zeno++: Robust fully asynchronous SGD},
	author={Xie, Cong and Koyejo, Sanmi and Gupta, Indranil},
	booktitle={International Conference on Machine Learning},
	pages={10495--10503},
	year={2020},
	organization={PMLR}
}


\newpage
\appendix
\onecolumn
\section{Appendix}

In the appendix, we first present a review of related work. Next, we compare our direct semi-verified mean estimation approach  with the indirect one based on list-decodable mean estimation. Finally, we provide formal proofs of the theorems.

\subsection{Related Work}\label{sec:related}
\subsubsection{Robust mean estimation}\label{sec:mean}
Distributed learning relies on accurate estimation of the statistical mean of gradients gathered from worker machines. Therefore, we provide a brief review of robust mean estimation first. Estimation with corrupted samples has been an important task in statistics since 1960s \cite{huber1964robust,huber2004robust}. Traditional robust estimation methods such as trimmed mean behaves badly as dimensionality increases. An exception is Tukey median \cite{tukey1975mathematics}, which is sample efficient in high dimensional spaces. However, Tukey median is not practical, since the computation can not finish in polynomial time. \cite{lai2016agnostic,diakonikolas2016robust,diakonikolas2017being} proposed convex programming and filtering methods for high dimensional robust mean estimation. Since then, a flurry of works has emerged, focusing on improving the time complexity \cite{cheng2019high}, extending to sparse problems \cite{balakrishnan2017computationally,diakonikolas2020outlier}, and improving the performance under stronger tail assumption with sum of squares method \cite{hopkins2018mixture}. These methods work if the proportion of contaminated samples is less than $1/2$. 

If majority of samples are corrupted, then it is impossible to estimate the mean with a single solution. Suppose that the fraction of clean samples is ensured to be at least $\alpha$, i.e. $q/m<1-\alpha$, then the adversary can produce $\Omega(1/\alpha)$ clusters of points, with one of them composed of clean samples, while others are corrupted. In this case, there is no way to determine which cluster is correct. Therefore, in recent works, the goal of mean estimation is relaxed to generating a list of hypotheses, with the guarantee that at least one of the hypotheses is close to the ground truth $\mu^*$. This paradigm is called list-decodable learning. Typically, the length of generated hypotheses list need to be controlled within $O(1/\alpha)$. We hope to make the minimum estimation error among the hypotheses list to be as low as possible. \cite{charikar2017learning} proposed an SDP-based method, achieving error of $O(\sigma\sqrt{\ln(1/\alpha)/\alpha})$. For spherical Gaussians, a multi-filter method can achieve a better rate \cite{diakonikolas2018list}. Several recent works improves the time complexity of list-decodable mean estimation. \cite{diakonikolas2021list} proposed SIFT (Subspace Isotropic Filtering), which achieves nearly PCA time complexity. \cite{cherapanamjeri2020list} introduced an optimization based approach with linear time complexity. 

Apart from the untrusted data, if we also have some auxiliary clean samples, then it is possible to determine one estimated value from the hypotheses list, which is close to the ground truth with high probability. \cite{charikar2017learning} named it semi-verified learning, and claims that with $K$ hypotheses, one needs $O(\ln K)$ clean samples to determine the correct one.

\subsubsection{Federated Learning with less than half Byzantine machines}

Suppose that there are $m$ worker machines, among which $q$ of them being attacked.  Each worker machine has $n$ samples. The total size of dataset is $N=mn$. Here we briefly review the previous works with $q<m/2$.

\cite{blanchard2017machine} proposed Krum, which collects gradient vectors from these $m$ worker machines, and then select a vector that is closest to its $m-q$ neighbors. It is expected that the selected one is close to the ground truth. However, the error bound is suboptimal. \cite{chen2017distributed} uses a geometric median of mean method, with a suboptimal error rate of $\sqrt{qd/N}$. \cite{yin2018byzantine} made the first step towards optimal statistical rates of Byzantine robust federated learning. Two methods are proposed: coordinate-wise median and coordinate-wise trimmed mean. Under the assumption that the gradient has bounded skewness, the coordinate-wise median method achieves $\tilde{O}(q\sqrt{d}/(m\sqrt{n})+d/\sqrt{nm}+\sqrt{d}/n)$. If $n\gtrsim m$, then this method is optimal in $q$, $n$ and $m$. Under the subexponential tail assumption, the coordinate-wise trimmed mean achieves $\tilde{O}(d(q/(m\sqrt{n})+1/\sqrt{nm}))$, which is also optimal. However, these two methods are not optimal in $d$. As a result, in large scale models, the performance of these models are not desirable.

Recently, \cite{zhu2023byzantine} shows that by using several recent high dimensional robust mean estimation methods \cite{diakonikolas2016robust,diakonikolas2017being,diakonikolas2020outlier,steinhardt2018robust,zhu2022a,zhu2022b,zhu2022c,diakonikolas2023algorithmic} as the gradient aggregator functions, the error rate of federated learning becomes nearly optimal not only in $m$, $n$, $q$, but also in dimensionality $d$, up to a logarithm factor.  

\subsubsection{Federated learning with arbitrary number of Byzantine machines}\label{sec:related-arbitrary}
The condition that less than half worker machines are attacked may not hold in practice. In this case, the methods mentioned above no longer works. If $q>m/2$, without additional information, it is impossible to provide a precise solution. \cite{cao2019distributed} provides a simple method with the help of a clean auxiliary dataset. The basic idea is to calculate the gradient vector only using the auxiliary dataset as a baseline, and then filter out gradient vectors uploaded from worker machines that are too far away from such baseline. A slightly different method was proposed in \cite{regatti2020bygars}, which introduces reputation scores as the weights for gradient. However, the error rate of these methods is $\|\hat{\mathbf{w}}-\mathbf{w}^*\| = O(\sqrt{d/N_A})$, with $N_A$ being the size of auxiliary dataset. This rate is not significantly improved compared with learning with auxiliary data only, and is undesirable in high dimension problems. A significantly different method called Zeno was proposed in \cite{xie2019zeno}, which formulates a stochastic descent score to filter out adversarial workers. \cite{xie2020zeno++} proposed Zeno++, an improvement to Zeno. The main advantage of Zeno++ is the suitability under asynchronous setting. Moreover, by computing a first order approximation of the gradient descent score, this method significantly improves the computational efficiency. However, the error bound is also significantly suboptimal. In particular, the error rate of Zeno and Zeno++ is still $O(\sqrt{d/N_A})$ (this comes from calculating the term $V_3$ in Theorem 1 in \cite{xie2020zeno++}), which also scales badly with the increase of dimensionality. To the best of our knowledge, our work is the first attempt to solve Byzantine tolerant distributed learning problem in high dimensions.

\subsection{Comparison with Indirect Method Based on List-decodable Mean Estimation}\label{sec:compare}

The SIFT method proposed in \cite{diakonikolas2021list} is an efficient approach for list-decodable mean estimation. Given the ratio of trusted samples $\alpha$, SIFT outputs a list of $O(1/\alpha)$ hypotheses whose minimum distance to $\mu^*$ is $O(1/\alpha)$. From \cite{charikar2017learning}, with $O(\ln(1/\alpha))$ verified samples, the list-decodable method can be converted to semi-verified method.

The main difference of our method with the indirect approach mentioned above is that after obtaining the projection matrix, SIFT projects a list of random samples, i.e. calculating $\mathbf{P}\mathbf{X}_0+(\mathbf{I}-\mathbf{P})\mathbf{X}_i$ as one of hypothesis, in which $i$ is randomly selected. However, we only calculate one output, i.e. $\hat{\mu}=\mathbf{P}\mathbf{X}_0+(\mathbf{I}-\mathbf{P})\mu(S)$. There are also some differences in details of matrix calculation. 

Now we compare our direct method with the indirect approach mentioned above. For a fair comparison, we let $n=1$, then under additive contamination model, \eqref{eq:finala} becomes
\begin{eqnarray}
	\mathbb{E}[\norm{\hat{\mu}-\mu^*}_2]\lesssim \sigma \alpha^{-\frac{1}{2}}\left(\frac{1}{\sqrt{N_A}}+1\right)\sim \sigma \alpha^{-\frac{1}{2}},
\end{eqnarray}
as long as $1/\alpha \lesssim p\lesssim N_A/\alpha$.

Our method has several advantages. Firstly, according to \cite{charikar2017learning}, the indirect method using SIFT requires $O(\ln (1/\alpha))$ auxiliary clean samples, while we only need at least one clean sample. Secondly, the parameter selection is easier, since SIFT requires $p\sim 1/\alpha$, while our method only requires $1/\alpha \lesssim p\lesssim N_A/\alpha$, which is a wider constraint. In practice, $\alpha$ is usually unknown and setting an appropriate $p$ is challenging. Our direct method partially alleviates this issue.

\subsection{Proof of Theorem 1}\label{sec:additive}

This section shows the error bound of semi-verified mean estimation problem under additive contamination model. The notations in this section is slightly different from those in Algorithm 2. In Algorithm 2, $S$ is a dynamic set with some samples gradually removed, while in this section, $S$ refers to the final remaining dataset after the whole algorithm is finished. This means that $S$ is a fixed set. Furthermore, denote $S_t$ as the remaining set $S$ in Algorithm 2 after running the while loop $t$ iterations. 

The proof begins with the following lemma.

\begin{lemma}\label{lem:G}
	(\cite{charikar2017learning}, Proposition B.1) With probability at least $1-e^{-\alpha m/64}$, there exists a set $G_0\subset S_0^*$ with $|G_0|\geq \alpha m/2$, and
	\begin{eqnarray}
		\left\|\frac{1}{|G_0|}\sum_{i\in G_0}(\mathbf{Y}_i-\mu^*)(\mathbf{Y}_i-\mu^*)^T\right\|_{op}\leq 8\sigma^2 \left(1+\frac{2d}{\alpha m}\right).
		\label{eq:G}
	\end{eqnarray}
\end{lemma}

With this lemma, we can pick $G_0$ with $|G_0|\geq \alpha m/2$ and satisfy \eqref{eq:G}. $G_0$ can be intuitively understood as 'good data', which means that samples in $G_0$ are neither corrupted nor outliers in clean samples. Note that with a small probability $e^{-\alpha m/64}$, this is impossible. In this case, just let $G_0$ be a set of $\lceil \alpha m/2\rceil$ samples randomly selected from $S_0^*$.

Recall that 
\begin{eqnarray}
	\hat{\mu}(A)=\mathbf{P}\mathbf{X}_0 + (\mathbf{I}-\mathbf{P})\mu(S),
\end{eqnarray}
in which $S$ is the remaining set after filtering. Denote $G=G_0\cap S$ and the set of remaining good samples. Then
\begin{eqnarray}
	\|\hat{\mu}-\mu^*\|_2^2 \leq 3\|\mathbf{P}(\mathbf{X}_0- \mu^*)\|_2^2 + 3\|(\mathbf{I}-\mathbf{P})(\mu(S)-\mu(G))\|_2^2 + 3\|(\mathbf{I}-\mathbf{P})(\mu(G)-\mu^*)\|_2^2.
	\label{eq:mse}
\end{eqnarray}
We bound these three terms in \eqref{eq:mse} separately.

\textbf{Bound of $\|\mathbf{P}(\mathbf{X}_0- \mu^*)\|_2^2$.} We prove the following lemma:
\begin{lemma}\label{lem:I1bound}
	\begin{eqnarray}
		\mathbb{E}[\|\mathbf{P}(\mathbf{X}_0-\mu^*)\|_2^2]\leq \frac{\sigma^2 p}{N_A}.	
	\end{eqnarray}
\end{lemma}
\begin{proof}
	Recall that in Algorithm 2, the auxiliary clean set $A$ is only used in the last step, and the attacker has no knowledge of $A$. As a result, $\mathbf{P}$ and samples in $A$ are mutually independent. Hence
	\begin{eqnarray}
		\mathbb{E}[\|\mathbf{P}(\mathbf{X}_0-\mu^*)\|_2^2]&=&\mathbb{E}[\tr(\mathbf{P}(\mathbf{X}_0-\mu^*)(\mathbf{X}_0-\mu^*)^T\mathbf{P})]\nonumber\\
		&\leq & \frac{1}{N_A}\mathbb{E}[\tr(\sigma^2 \mathbf{P})]\nonumber\\
		&=&\frac{\sigma^2 p}{N_A},
		\label{eq:I1}
	\end{eqnarray}
	in which $\tr$ denotes trace of a matrix. For the last step, recall that $\mathbf{P}=\mathbf{U}_p\mathbf{U}_p^T$, in which $\mathbf{U}_p\in \mathbb{R}^{d\times p}$ is the first $p$ columns of the matrix $\mathbf{U}$, whose columns are eigenvectors of $V(S)$. Hence $\tr(\mathbf{P}) = p$.
\end{proof}

\textbf{Bound of $\|(\mathbf{I}-\mathbf{P})(\mu(S)-\mu(G))\|_2^2$.} 
We show the following lemma:
\begin{lemma}\label{lem:I2bound}
	Under the following conditions: 
	
	(a) \eqref{eq:G} holds; 
	
	(b) there exists an $\epsilon\in (0,1)$, such that
	\begin{eqnarray}
		\lambda_c&\geq& 32\frac{\sigma^2}{n} \left(1+\frac{2d}{\alpha m}\right)\frac{1+\epsilon}{1-\epsilon},\label{eq:lamc2}\\
		p&\geq &\frac{8(1+\epsilon)}{\alpha(1-\epsilon)};\label{eq:m2}
	\end{eqnarray}
	
	(c) In all iterations in Algorithm 2, denote $L_t\subset S_t$ as removed samples in step $t$, $G_t=G_0\cap S_t$ is the set of remaining good data after $t$ iterations, $L_{Gt}=L_t\cap G_t$, then for all iterations,
	\begin{eqnarray}
		\frac{|L_{Gt}|}{|G_t|}&\leq& \frac{1+\epsilon}{|G_t|}\sum_{i\in G}\frac{\tau_i}{\tau_{\max}},\label{eq:cond1}\\
		\frac{|L_t|}{|S_t|}&\geq & \frac{1-\epsilon}{|S_t|}\sum_{i\in S}\frac{\tau_i}{\tau_{\max}}.\label{eq:cond2}
	\end{eqnarray}
	Then
	\begin{eqnarray}
		\|(\mathbf{I}-\mathbf{P})(\mu(S)-\mu(G))\|_2^2 \leq \frac{2\lambda_c}{\alpha}.
		\label{eq:I2bound}
	\end{eqnarray}
\end{lemma}
\begin{proof}
	The proof is shown in Appendix \ref{sec:I2bound}.
\end{proof}
Now we discuss three conditions in Lemma \ref{lem:I2bound}. For (1), Lemma \ref{lem:G} has shown that \eqref{eq:G} holds with high probability. (2) has been guaranteed in the statements in Theorem 1. It remains to show that \eqref{eq:cond1} and \eqref{eq:cond2} also hold with high probability. The result is stated in the following lemma.
\begin{lemma}\label{lem:prob}
	\begin{eqnarray}
		\text{P}\left(\exists t, \frac{|L_{Gt}|}{|G_t|}> \frac{1+\epsilon}{|G_t|}\sum_{i\in G}\frac{\tau_i}{\tau_{\max}} \text{ or } 	\frac{|L_t|}{|S_t|}< \frac{1-\epsilon}{|S_t|}\sum_{i\in S}\frac{\tau_i}{\tau_{\max}}\right)\leq 4m\exp\left[-\frac{1}{8}\alpha\lambda_c m^\frac{1}{3}\epsilon^2\right].
	\end{eqnarray}
\end{lemma}
\begin{proof}
	The proof is shown in section \ref{sec:prob}.
\end{proof}
If at least one of conditions \eqref{eq:G}, \eqref{eq:cond1}, \eqref{eq:cond2} are violated, then due to the prefilter step in Algorithm 2, the maximum distance between any two samples is at most $2m^{1/3}$. Therefore, $\|(\mathbf{I}-\mathbf{P})(\mu(S)-\mu(G))\|_2^2\leq 4m^{2/3}$ always hold. With Lemma \ref{lem:G} and \ref{lem:prob}, \eqref{eq:I2highp} is violated with exponential decaying probability. As a result, 
\begin{eqnarray}
	\mathbb{E}[\|(\mathbf{I}-\mathbf{P})(\mu(S)-\mu(G))\|_2^2]\leq \frac{2\lambda_c}{\alpha} + 4m^\frac{2}{3}\left[e^{-\frac{1}{64}\alpha m}+4m\exp\left(-\frac{1}{8}\alpha \lambda_cm^\frac{1}{3}\epsilon^2\right)\right].
	\label{eq:I2}
\end{eqnarray}

\textbf{Bound of $\|(\mathbf{I}-\mathbf{P})(\mu(G)-\mu^*)\|_2^2$.} We show the following lemma:
\begin{lemma}\label{lem:I3bound}
	If \eqref{eq:G}, \eqref{eq:lamc2} and \eqref{eq:m2} hold, then
	\begin{eqnarray}
		\|(\mathbf{I}-\mathbf{P})(\mu(G)-\mu^*)\|_2^2\leq \frac{\lambda_c}{2\alpha}.
		\label{eq:I3bound}
	\end{eqnarray}
\end{lemma}
\begin{proof}
	The proof is shown in Appendix \ref{sec:I3bound}.
\end{proof}

Following the same argument as \eqref{eq:I2}, we get
\begin{eqnarray}
	\mathbb{E}[\|(\mathbf{I}-\mathbf{P})(\mu(G)-\mu^*)\|_2^2]=\frac{\lambda_c}{2\alpha} + 4e^{-\frac{1}{64}\alpha m}m^\frac{2}{3}.
	\label{eq:I3}
\end{eqnarray}
By \eqref{eq:mse}, \eqref{eq:I1}, \eqref{eq:I2} and \eqref{eq:I3},
\begin{eqnarray}
	\mathbb{E}[\|\hat{\mu}-\mu^*\|_2^2]\leq \frac{3\sigma^2 p}{N_A}+\frac{15\lambda_c}{2\alpha} + \delta_N,
\end{eqnarray}
in which $\delta_N$ decays faster than any polynomial of $N$.

\textbf{Proof of Auxiliary Lemmas.} Here we prove lemmas mentioned above.
\subsubsection{Proof of Lemma \ref{lem:I2bound}}\label{sec:I2bound}

Denote $Q=S\setminus G$ as the remaining corrupted samples that are not filtered by Algorithm 2. Since $\mathbf{P}$ is a projection matrix, $\mathbf{P}^2 = \mathbf{P}$ and $\mathbf{P}=\mathbf{P}^T$. For any vector $\mathbf{u}$ with $\|\mathbf{u}\|_2=1$,
\begin{eqnarray}
	&&\mathbf{u}^T (\mathbf{I}-\mathbf{P})V(S)(\mathbf{I}-\mathbf{P})\mathbf{u} \nonumber\\
	&=& \frac{1}{|S|}\mathbf{u}^T (\mathbf{I}-\mathbf{P})\left[\sum_{i\in S} (\mathbf{Y}_i-\mu(S))(\mathbf{Y}_i-\mu(S))^T\right] (\mathbf{I}-\mathbf{P})\mathbf{u}\nonumber\\
	&=& \frac{1}{|S|}\mathbf{u}^T (\mathbf{I}-\mathbf{P})\left[\sum_{i\in G} (\mathbf{Y}_i-\mu(G)) (\mathbf{Y}_i-\mu(G))^T + |G|(\mu(G)- \mu(S))(\mu(G)-\mu(S))^T\right.\nonumber\\
	&& \left. + \sum_{i\in Q}(\mathbf{Y}_i-\mu(Q))(\mathbf{Y}_i-\mu(Q))^T + |Q|(\mu(Q)-\mu(S))(\mu(Q)-\mu(S))^T\right] (\mathbf{I}-\mathbf{P})\mathbf{u}\nonumber\\
	&\geq & \frac{1}{|S|}\mathbf{u}^T (\mathbf{I}-\mathbf{P})\left[|G|(\mu(G)- \mu(S))(\mu(G)-\mu(S))^T\right.\nonumber\\
	&&\left.+|Q|(\mu(Q)-\mu(S))(\mu(Q)-\mu(S))^T\right](\mathbf{I}-\mathbf{P})\mathbf{u}\nonumber\\
	&\geq & \frac{|G|}{|Q|}\mathbf{u}^T (\mathbf{I}-\mathbf{P})(\mu(G)-\mu(S))(\mu(G)-\mu(S))^T (\mathbf{I}-\mathbf{P})\mathbf{u}.
	\label{eq:diff}
\end{eqnarray}
For the last step, note that $\mu(\cdot)$ is the sample mean, hence $|S|\mu(S)=|G|\mu(G)|+|Q|\mu(Q)$. Hence,
\begin{eqnarray}
	\mu(Q)-\mu(S) = -\frac{|G|}{|Q|}(\mu(G) -\mu(S)).
	\label{eq:meandiff}
\end{eqnarray}
\eqref{eq:meandiff} is then used to derive the last step in \eqref{eq:diff}.

Recall Algorithm 2, $V(S) = \sum_{j=1}^d \lambda_j(S)\mathbf{u}_j\mathbf{u}_j^T$, in which $\lambda_j(S)$ is the $j$-th largest eigenvalue of $V(S)$. Therefore
\begin{eqnarray}
	(\mathbf{I}-\mathbf{P})V(S) (\mathbf{I}-\mathbf{P})=\sum_{j=p+1}^d \lambda_j(S)\mathbf{u}_j\mathbf{u}_j^T,
\end{eqnarray}
hence for any $\mathbf{u}$ with $\|\mathbf{u}\|_2=1$, $\mathbf{u}^T (\mathbf{I}-\mathbf{P})V(S) (\mathbf{I}-\mathbf{P}) \mathbf{u}\leq \lambda_{p+1} (S)$. Let 
\begin{eqnarray}
	\mathbf{u} = \frac{(\mathbf{I}-\mathbf{P})(\mu(G)-\mu(S))}{\|(\mathbf{I}-\mathbf{P})(\mu(G)-\mu(S))\|_2},
\end{eqnarray}
then from \eqref{eq:diff},
\begin{eqnarray}
	\lambda_{k+1}(V(S))\geq \frac{|G|}{|Q|}\|(\mathbf{I}-\mathbf{P})(\mu(G)-\mu(S))\|_2^2.
	\label{eq:sgdist1}
\end{eqnarray}
$|Q|<|S|$ always hold, thus \eqref{eq:sgdist1} implies 	\begin{eqnarray}
	\|(\mathbf{I}-\mathbf{P})(\mu(S)-\mu(G))\|_2^2 \leq \frac{|S|}{|G|}\lambda_{k+1}(S).
	\label{eq:sgdist}
\end{eqnarray}
It remains to bound $|S|/|G|$. 

\begin{lemma}\label{lem:transition}
	With \eqref{eq:lamc2} and \eqref{eq:m2}, under \eqref{eq:G}, if \eqref{eq:cond1} and \eqref{eq:cond2} is satisfied in all iterations, then for all iterations $t=1,2,\ldots$,
	\begin{eqnarray}
		\frac{|G_t|}{\sqrt{|S_t|}}\geq \frac{|G_0|}{\sqrt{|S_0|}}.
		\label{eq:transition}
	\end{eqnarray}
\end{lemma}
\begin{proof}
	The proof is shown in section \ref{sec:transition}.
\end{proof}
\begin{cor}\label{cor}
	Under the conditions of Lemma \ref{lem:transition}, $|G|\geq \alpha^2 N/4$, $|S|/|G|\leq 2/\alpha$.
\end{cor}
\begin{proof}
	From Lemma \ref{lem:transition}, 
	\begin{eqnarray}
		\sqrt{|G|} \geq \frac{|G|}{\sqrt{|S|}}\geq \frac{|G_0|}{\sqrt{|S_0|}}=\frac{1}{2}\alpha\sqrt{N},
	\end{eqnarray} 	
	in which the first inequality holds because $G\subset S$, and
	\begin{eqnarray}
		\frac{|S|}{|G|}=\sqrt{|S|}\frac{\sqrt{|S|}}{|G|}\leq \sqrt{|S|}\frac{\sqrt{|S_0|}}{|G_0|}\leq |S_0|/|G_0|\leq \frac{2}{\alpha}.
	\end{eqnarray}
\end{proof}
After all iterations, $\lambda_{k+1}(S)\leq \lambda_k(S)\leq \lambda_c$. From Corollary \ref{cor}, under \eqref{eq:G}, \eqref{eq:cond1} and \eqref{eq:cond2}, 
\begin{eqnarray}
	\|(\mathbf{I}-\mathbf{P})(\mu(S)-\mu(G))\|_2^2 \leq \frac{2\lambda_c}{\alpha}.
	\label{eq:I2highp}
\end{eqnarray}
\subsubsection{Proof of Lemma \ref{lem:transition}}\label{sec:transition}

We show the following lemmas first.

\begin{lemma}\label{lem:proj}
	\begin{eqnarray}
		\|\mathbf{P}V^{-\frac{1}{2}}(S_t)(\mu(G_t)-\mu(S_t))\|_2^2\leq \frac{|S_t|}{|G_t|}.
	\end{eqnarray}
\end{lemma}
\begin{proof}
	The proof is similar to the proof of Lemma \ref{lem:I2bound} in section \ref{sec:I2bound}. Denote $Q_t=S_t\setminus G_t$. For any $\mathbf{u}$ with $\|\mathbf{u}\|_2=1$,
	\begin{eqnarray}
		\mathbf{u}^T \mathbf{P}\mathbf{u} &=& \mathbf{u}^T V^{-\frac{1}{2}}(S_t)V(S_t)V^{-\frac{1}{2}}(S_t)\mathbf{P}\mathbf{u}\nonumber\\
		&=&\frac{1}{|S_t|}\mathbf{u}^T \mathbf{P}V^{-\frac{1}{2}}(S_t)\sum_{i\in S}(\mathbf{Y}_i-\mu(S_t))(\mathbf{Y}_i-\mu(S_t))^TV^{-\frac{1}{2}}(S_t) \mathbf{P}\mathbf{u}\nonumber\\
		&\geq & \frac{1}{|S_t|}\mathbf{u}^T\mathbf{P}V^{-\frac{1}{2}}(S_t)[|G_t|(\mu(G_t)-\mu(S_t))(\mu(G_t)-\mu(S_t))^T\nonumber\\
		&&+|Q|(\mu(Q)-\mu(S_t))(\mu(Q)-\mu(S_t))^T]V^{-\frac{1}{2}}(S_t) \mathbf{P}\mathbf{u}\nonumber\\
		&=&\frac{|G_t|}{|Q|}\mathbf{u}^T \mathbf{P}V^{-\frac{1}{2}}(S_t)(\mu(G_t)-\mu(S_t))(\mu(G_t)-\mu(S_t))^TV^{-\frac{1}{2}}(S_t)\mathbf{P}\mathbf{u}.
	\end{eqnarray}
	Let
	\begin{eqnarray}
		\mathbf{u}=\frac{PV^{-\frac{1}{2}}(S_t)(\mu(G_t)-\mu(S_t))}{\|PV^{-\frac{1}{2}}(S_t)(\mu(G_t)-\mu(S_t))\|_2},
	\end{eqnarray}	
	since $\mathbf{P}$ is a projection matrix, $\mathbf{u}^T \mathbf{P}\mathbf{u}\leq 1$. Hence
	\begin{eqnarray}	\|\mathbf{P}V^{-\frac{1}{2}}(S_t)(\mu(G_t)-\mu(S_t))\|_2^2\leq \frac{|Q|}{|G_t|}.
	\end{eqnarray}
	Lemma \ref{lem:proj} can then be proved using the fact that $|Q|<|S_t|$.
\end{proof}
\begin{lemma}\label{lem:tau}
	Under \eqref{eq:G}, 
	\begin{eqnarray}
		\frac{1}{|G_t|}\sum_{i\in G}\tau_i &\leq& 8\sigma^2 \left(1+\frac{2d}{\alpha m}\right) \frac{p}{\lambda_p(S_t)}+\frac{|S_t|}{|G_t|},\label{eq:taug}\\
		\frac{1}{|S_t|}\sum_{i\in S}\tau_i &=& p.\label{eq:taus}
	\end{eqnarray}
\end{lemma}
\begin{proof}
	\textbf{Proof of \eqref{eq:taug}}.
	\begin{eqnarray}
		\frac{1}{|G_t|}\sum_{i\in G}\tau_i &=& \frac{1}{|G_t|}\sum_{i\in G}\|\mathbf{P} V^{-\frac{1}{2}}(S_t) (\mathbf{Y}_i-\mu(S_t))\|_2^2\nonumber\\
		&=&\frac{1}{|G_t|}\sum_{i\in G}\tr(\mathbf{P}V^{-\frac{1}{2}}(S_t) (\mathbf{Y}_i-\mu(S_t))(\mathbf{Y}_i-\mu(S_t))^T V^{-\frac{1}{2}}(S_t)\mathbf{P})\nonumber\\
		&=&\frac{1}{|G_t|}\sum_{i\in G}\tr(\mathbf{P}V^{-\frac{1}{2}}(S_t)(\mathbf{Y}_i-\mu(G_t))(\mathbf{Y}_i-\mu(G_t))^T V^{-\frac{1}{2}}(S_t)\mathbf{P})\nonumber\\
		&&+\tr(\mathbf{P}V^{-\frac{1}{2}}(S_t)(\mu(G_t)-\mu(S_t))(\mu(G_t)-\mu(S_t))^T V^{-\frac{1}{2}}(S_t)\mathbf{P})\nonumber\\
		&\overset{(a)}{\leq} & 8\sigma^2 \left(1+\frac{2d}{\alpha m}\right)\tr(\mathbf{P}V^{-1}(S_t)\mathbf{P})+\|PV^{-\frac{1}{2}}(S_t)(\mu(G_t)-\mu(S_t))\|_2^2\nonumber\\
		&\overset{(b)}{\leq} & 8\sigma^2 \left(1+\frac{2d}{\alpha m}\right)\tr(\mathbf{P}V^{-1}(S_t)\mathbf{P}) + \frac{|S_t|}{|G_t|}\nonumber\\
		&\overset{(c)}{\leq} & 8\sigma^2 \left(1+\frac{2d}{\alpha m}\right)\sum_{j=1}^m \frac{1}{\lambda_j(S_t)}+\frac{|S_t|}{|G_t|}\nonumber\\
		&\leq & 8\sigma^2\left(1+\frac{2d}{\alpha m}\right)\frac{p}{\lambda_p(S_t)}+\frac{|S_t|}{|G_t|}.
	\end{eqnarray}
	For (a), from \eqref{eq:G},
	\begin{eqnarray}
		\frac{1}{|G_t|}\sum_{i\in G}(\mathbf{Y}_i-\mu(G_t))(\mathbf{Y}_i-\mu(G_t))^T\preceq \frac{1}{|G_t|}\sum_{i\in G}(\mathbf{Y}_i-\mu^*)(\mathbf{Y}_i-\mu^*)^T\leq 8\sigma^2\left(1+\frac{2d}{\alpha m}\right).
	\end{eqnarray}
	(b) comes from Lemma \ref{lem:proj}. For (c), note that
	\begin{eqnarray}
		V(S_t)=\sum_{j=1}^d \lambda_j(S_t)\mathbf{u}_j\mathbf{u}_j^T,
	\end{eqnarray}
	and $\mathbf{P}\mathbf{u}_j=\mathbf{u}_j$ for $i\leq m$, $\mathbf{P}\mathbf{u}_j=\mathbf{0}$ for $i>m$, then
	\begin{eqnarray}
		\mathbf{P}V^{-1}(S_t)\mathbf{P} = \sum_{j=1}^d \frac{1}{\lambda_j(S_t)}\mathbf{P}\mathbf{u}_j\mathbf{u}_j^T \mathbf{P} = \sum_{i=1}^m \frac{1}{\lambda_j(S_t)}\mathbf{u}_j\mathbf{u}_j^T,
	\end{eqnarray}
	thus $\tr(\mathbf{P} V^{-1}(S_t)\mathbf{P})=\sum_{i=1}^m 1/\lambda_j(S_t)$.
	
	\textbf{Proof of \eqref{eq:taus}}. 
	\begin{eqnarray}
		\frac{1}{|S_t|}\sum_{i\in S}\tau_i &=&\frac{1}{|S_t|}\sum_{i\in S}\tr(\mathbf{P} V^{-\frac{1}{2}}(S_t)(\mathbf{Y}_i-\mu(S_t))(\mathbf{Y}_i-\mu(S_t))^T V^{-\frac{1}{2}}(S_t)\mathbf{P})\nonumber\\
		&=&\tr(\mathbf{P}V^{-\frac{1}{2}}(S_t) V(S_t) V^{-\frac{1}{2}}(S_t)\mathbf{P})\nonumber\\
		&=&\tr(\mathbf{P}) = p.
	\end{eqnarray}
	The proof is complete.
\end{proof}
Now it remains to prove Lemma \ref{lem:transition} with the above lemmas. The proof is by induction. \eqref{eq:transition} holds clearly for $t=0$. Therefore, it remains to show that if \eqref{eq:transition} holds for some $t$, then it must also hold for $t+1$. From \eqref{eq:lamc2}, 
\begin{eqnarray}
	\lambda_p(S_t)\geq \lambda_c\geq 32\sigma^2\left(1+\frac{2d}{\alpha m}\right)\frac{1+\epsilon}{1-\epsilon},
\end{eqnarray}
and from \eqref{eq:m2},
\begin{eqnarray}
	\frac{|S_t|}{|G_t|}=\sqrt{|S_t|}\frac{\sqrt{|S_t|}}{|G_t|}\leq \sqrt{|S_t|}\frac{\sqrt{|S_0|}}{|G_0|}\leq \frac{|S_0|}{|G_0|}=\frac{2}{\alpha}\leq \frac{1}{4}p\frac{1-\epsilon}{1+\epsilon}
\end{eqnarray}
Using \eqref{eq:taug}, 
\begin{eqnarray}
	\frac{1}{|G_t|}\sum_{i\in G}\tau_i &\leq& 8\sigma^2 \left(1+\frac{2d}{\alpha m}\right) \frac{p}{\lambda_p(S_t)}+\frac{|S_t|}{|G_t|}\\
	&\leq & \frac{1}{4}p\frac{1-\epsilon}{1+\epsilon}+\frac{1}{4}p\frac{1-\epsilon}{1+\epsilon}\nonumber\\
	&=&\frac{1}{2}p\frac{1-\epsilon}{1+\epsilon}.
\end{eqnarray}
Recall the condition \eqref{eq:cond1} and \eqref{eq:cond2}. Define $r_G(t)=|L_{Gt}|/|G_t|$ and $r(t)=|L_t|/|S_t|$ as the ratio of removed samples in $G_t$ and $S_t$. Then
\begin{eqnarray}
	\frac{r_0(t)}{r(t)}&\leq& \frac{1+\epsilon}{1-\epsilon}\frac{\frac{1}{|G_t|}\sum_{i\in G_t}\frac{\tau_i}{\tau_{\max}}}{\frac{1}{|S_t|}\sum_{i\in S_t}\frac{\tau_i}{\tau_{\max}}}\nonumber\\
	&\leq& \frac{1+\epsilon}{1-\epsilon}\frac{\frac{1}{2}p\frac{1-\epsilon}{1+\epsilon}}{p}=\frac{1}{2}.
\end{eqnarray}
Therefore, after $t+1$-th iteration,
\begin{eqnarray}
	\frac{|G_{t+1}|}{\sqrt{|S_{t+1}|}} = \frac{|G_t\setminus L_{Gt}|}{\sqrt{|S_t\setminus L_t|}}\geq \frac{|G_t|(1-r_0(t))}{\sqrt{|S_t|(1-r(t))}}\geq \frac{|G_t|(1-r_0(t))}{\sqrt{|S_t|}(1-\frac{r(t)}{2})}\geq \frac{|G_{t}|}{\sqrt{|S_{t}|}}.
\end{eqnarray}
The proof is complete.
\subsubsection{Proof of Lemma \ref{lem:prob}}\label{sec:prob}
Suppose there are $k$ independent random variables $Z_1,\ldots, Z_k$. $\text{P}(Z_i=1) = p_i$, and $\text{P}(Z_i=0) = 1-p_i$. Let $\bar{p}=(1/k)\sum_{i=1}^n p_i$, then we show a concentration inequality of the mean of $Z_i$. 

Let $Z=\sum_{i=1}^k Z_i$. Then
\begin{eqnarray}
	&&\mathbb{E}[e^{\lambda Z}]=\mathbb{E}\left[\exp\left(\lambda \sum_{i=1}^k Z_i\right)\right]=\Pi_{i=1}^k \mathbb{E}[e^{\lambda Z_i}] =\Pi_{i=1}^k (1-p_i+p_ie^\lambda)\nonumber\\
	&\leq & \exp\left(\sum_{i=1}^n \ln (1-p_i+p_ie^\lambda)\right)\leq \exp\left(\sum_{i=1}^k p_i(e^\lambda - 1)\right)=\exp\left(k\bar{p}(e^\lambda - 1)\right).
\end{eqnarray}
Hence
\begin{eqnarray}
	\text{P}(Z>k\bar{p}(1+\epsilon))&\leq & \underset{\lambda\geq 0}{\inf}\exp\left[-k\bar{p}(1+\epsilon)\lambda+k\bar{p}(e^\lambda - 1)\right]\nonumber\\
	&=& \exp\left[-k\bar{p}((1+\epsilon)\ln (1+\epsilon) - \epsilon)\right].
\end{eqnarray}
Similarly,
\begin{eqnarray}
	\text{P}(Z<k\bar{p}(1-\epsilon))&\leq & \underset{\lambda\geq 0}{\inf}\exp\left[k\bar{p}(1+\epsilon)\lambda+k\bar{p}(e^{-\lambda} - 1)\right]\nonumber\\
	&=& \exp\left[-k\bar{p}((1-\epsilon)\ln (1-\epsilon) - \epsilon)\right].
\end{eqnarray}
It is straightforward to show that $(1+\epsilon)\ln(1+\epsilon)\geq \epsilon+\epsilon^2/4$, $(1-\epsilon)\ln(1-\epsilon)\geq -\epsilon+\epsilon^2/2$ for $\epsilon\in (0,1)$. Therefore
\begin{eqnarray}
	\text{P}\left(\left|\frac{1}{k\bar{p}}\sum_{i=1}^n Z_i-1\right|>\epsilon\right)\leq 2e^{-\frac{1}{4}k\bar{p}\epsilon^2}.
	\label{eq:conc}
\end{eqnarray}
Now we use \eqref{eq:conc} to prove Lemma \ref{lem:prob}. For each step $t$, let $Z_i$ be the indicator that sample $i$ is removed, $n$ be the remaining samples. $n$ can be lower bounded by $\alpha^2N/4$ by Corollary \ref{cor}. $\bar{p}$ can be lower bounded by
\begin{eqnarray}
	\bar{p}\overset{(a)}{\geq} \frac{1}{|S_t|}\sum_{i\in S}\frac{\tau_i}{\tau_{\max}}\overset{(b)}{=}\frac{p}{\tau_{\max}}\overset{(c)}{\geq}\frac{\lambda_c }{m^\frac{2}{3}},
\end{eqnarray}
in which (a) comes from Algorithm 2. (b) comes from \eqref{eq:taus}. For (c),
\begin{eqnarray}
	\tau_i&=&\|\mathbf{P}V^{-\frac{1}{2}}(s)(\mathbf{Y}_i-\mu(S))\|_2^2\nonumber\\
	&=&\|\sum_{j=1}^p \lambda_j^{-\frac{1}{2}}\mathbf{u}_j\mathbf{u}_j^T(\mathbf{Y}_j-\mu(S))\|_2^2\nonumber\\
	&\leq & p\lambda_c^{-1}\|\mathbf{Y}_j-\mu(S)\|_2^2\nonumber\\
	&\leq& p\lambda_c^{-1}m^\frac{2}{3}.
\end{eqnarray}
Hence, by taking union bound, and use $k\geq \alpha m/2$, we have
\begin{eqnarray}
	\text{P}\left(\exists t, \frac{|L_{Gt}|}{|G_t|}> \frac{1+\epsilon}{|G_t|}\sum_{i\in G}\frac{\tau_i}{\tau_{\max}} \text{ or } 	\frac{|L_t|}{|S_t|}< \frac{1-\epsilon}{|S_t|}\sum_{i\in S}\frac{\tau_i}{\tau_{\max}}\right)\leq 4m\exp\left[-\frac{1}{8}\alpha\lambda_c m^\frac{1}{3}\epsilon^2\right].
\end{eqnarray}

\subsubsection{Proof of Lemma \ref{lem:I3bound}}\label{sec:I3bound}
Under \eqref{eq:G}, we have
\begin{eqnarray}
	\|\mu(G)-\mu^*\|_2^2 &=&\underset{\|\mathbf{u}\|_2=1}{\max}\mathbf{u}^T (\mu(G)-\mu^*)(\mu(G)-\mu^*)^T\mathbf{u}\nonumber\\
	&\leq&\underset{\|\mathbf{u}\|_2=1}{\max}\mathbf{u}^T\frac{1}{|G|}\sum_{i\in G}(\mathbf{Y}_i-\mu^*)(\mathbf{Y}_i-\mu^*)^T\mathbf{u}\nonumber\\
	&\overset{(a)}{=}& \underset{\|\mathbf{u}\|_2=1}{\max}\mathbf{u}^T\frac{1}{|G|}\sum_{i\in G}(\mathbf{X}_i-\mu^*)(\mathbf{X}_i-\mu^*)^T\mathbf{u}\nonumber\\
	&\leq & \frac{|G_0|}{|G|}\underset{\|\mathbf{u}\|_2=1}{\max}\mathbf{u}^T\frac{1}{|G_0|}\sum_{i\in G_0}(\mathbf{X}_i-\mu^*)(\mathbf{X}_i-\mu^*)^T\mathbf{u}\nonumber\\
	&\overset{(b)}{\leq} & 8\sigma^2\left(1+\frac{2d}{\alpha m}\right)\frac{|G_0|}{|G|}\nonumber\\
	&\overset{(c)}{<} & \frac{\lambda_c}{2\alpha},
	\label{eq:I3highp}
\end{eqnarray}
in which (a) comes from the definition that $G_0\subset S_0^*$. Therefore, $\mathbf{Y}_i$ is not modified from the clean sample $\mathbf{X}_i$. (b) use Lemma \ref{lem:G}. (c) uses $\lambda_c>32\sigma^2(1+2d/(\alpha m))$, $|G_0|=\alpha m/2$, and from Corollary \ref{cor}, $|G|\geq \alpha^2 N/4$.

Since $\mathbf{P}$ is a projection matrix,
\begin{eqnarray}
	\|(\mathbf{I}-\mathbf{P})(\mu(G) - \mu^*)\|_2^2\leq \frac{\lambda_c}{2\alpha}.
\end{eqnarray}

\subsection{Proof of Theorem 2}\label{sec:strong}
In this section we show Theorem 2. Most of the proofs are the same as the proof of Theorem 1. The only exception is that Lemma \ref{lem:G} no longer holds under strong contamination model, since the adversary can modify $(1-\alpha)N$ arbitrarily. Therefore, in this section, we derive a new version of Lemma \ref{lem:G}. Other steps can just follow the proof of Theorem 1 in Appendix~\ref{sec:additive}.

Our proof begins with the matrix Bernstein inequality \cite{tropp2015introduction}.
\begin{lemma}\label{lem:bern}
	(From \cite{tropp2015introduction}, Theorem 1.6.2) Let $\mathbf{W}_1,\ldots, \mathbf{W}_N$ be i.i.d random matrices, with $\mathbb{E}[\mathbf{W}_i] = \mathbf{0}$, and $\|\mathbf{W}_i\|\leq L$ almost surely. Define $\mathbf{Z}=\sum_{i=1}^N \mathbf{W}_i$. Then for all $t\geq 0$,
	\begin{eqnarray}
		\text{P}(\|\mathbf{Z}\|>t)\leq 2d\exp\left(-\frac{t^2}{2N\|\mathbb{E}[\mathbf{W}\mathbf{W}^T]\|+2Lt/3}\right),
		\label{eq:bern}
	\end{eqnarray}
	in which $\mathbf{W}$ is an i.i.d copy of $\mathbf{W}_1,\ldots, \mathbf{W}_N$.
\end{lemma}

Based on Lemma \ref{lem:bern}, we show the following lemma.

\begin{lemma}\label{lem:Gstrong}
	With probability at least $1-\exp[-((1/2)\ln 2-1/4)N\alpha]-2d\exp(-(\ln^2 N)/2)$, there exists a set $G_0\subset S_0^*$ with $|G_0|\geq \alpha m/2$, and
	\begin{eqnarray}
		\left\|\frac{1}{|G_0|}\sum_{i\in G_0} (\mathbf{Y}_i-\mu^*)(\mathbf{Y}_i-\mu^*)^T\right\|_{op}\leq \frac{2\sigma^2}{\alpha} \left(1+\sqrt{\frac{16d\ln^2 N}{3\alpha m}}\right)^2.
	\end{eqnarray}
\end{lemma}
\begin{proof}
	Recall that $\mathbf{X}_i$ for all $i\in S_0$ are samples before contamination, while $\mathbf{Y}_i$ are observed values. $\mathbf{Y}_i = \mathbf{X}_i$ for at least $\alpha m$ samples.
	
	Let
	\begin{eqnarray}
		\mathbf{W}_i &=& \frac{1}{N}(\mathbf{X}_i-\mu^*)(\mathbf{X}_i-\mu^*)^T \mathbf{1}\left(\|\mathbf{X}_i-\mu^*\|\leq M\right)\nonumber\\
		&&-\frac{1}{N}\mathbb{E}[(\mathbf{X}-\mu^*)(\mathbf{X}-\mu^*)^T \mathbf{1}\left(\|\mathbf{X}-\mu^*\|\leq M\right)],
	\end{eqnarray}
	in which $\mathbf{X}$ is an i.i.d copy of $\mathbf{X}_i$ for $i\in S_0$, $M$ is a truncation threshold that will be determined later. Define $\mathbf{Z}=\sum_{i=1}^N \mathbf{W}_i$. Note that $\mathbb{E}[\mathbf{W}_i] = \mathbf{0}$. Hence we can bound $\mathbf{Z}$ with high probability using Lemma \ref{lem:bern}. It remains to find $L$ and $\mathbb{E}[\mathbf{W} \mathbf{W}^T]$. From triangle inequality,
	\begin{eqnarray}
		\|\mathbf{W}\|&\leq& \frac{1}{N}\left\|(\mathbf{X}_i-\mu^*)(\mathbf{X}_i-\mu^*)^T \mathbf{1}(\|\mathbf{X}_i-\mu^*\|\leq M)\right\|\nonumber\\
		&&+\frac{1}{N}\left\|\mathbb{E}[(\mathbf{X}-\mu^*)(\mathbf{X}-\mu^*)^T \mathbf{1}(\|\mathbf{X}_i-\mu^*\|\leq M)]\right\|\nonumber\\
		&\leq & \frac{2M^2}{N}.
	\end{eqnarray}
	Hence, in \eqref{eq:bern}, we can let $L=2M^2/N$. Moreover,
	\begin{eqnarray}
		\mathbb{E}[\mathbf{W}\mathbf{W}^T]&=&\frac{1}{m^2}\mathbb{E}\left[(\mathbf{X}-\mu^*)(\mathbf{X}-\mu^*)^T(\mathbf{X}-\mu^*)(\mathbf{X}-\mu^*)^T\mathbf{1}(\|\mathbf{X}-\mu^*\|\leq M)\right]\nonumber\\
		&&-\left[\frac{1}{N}\mathbb{E}[(\mathbf{X}_i-\mu^*)(\mathbf{X}_i-\mu^*)^T\mathbf{1}(\|\mathbf{X}_i-\mu^*\|\leq M)]\right]^2\nonumber\\
		&\preceq& \frac{M^2}{m^2}\mathbb{E}\left[(\mathbf{X}-\mu^*)(\mathbf{X}-\mu^*)^T\right]\nonumber\\
		&=&\frac{M^2}{m^2}V^*\preceq \frac{M^2\sigma^2}{m^2} \mathbf{I}.
	\end{eqnarray}
	Therefore,
	\begin{eqnarray}
		\text{P}(\|\mathbf{Z}\|>t)\leq 2d\exp\left[-\frac{t^2}{2\frac{M^2\sigma^2}{N}+\frac{4M^2}{3N}t}\right].
	\end{eqnarray}
	Let
	\begin{eqnarray}
		t_0=\max\left\{M\sigma\sqrt{\frac{2\ln^2 N}{N}}, \frac{4M^2}{3N}\ln^2 N \right\}.
		\label{eq:t0}
	\end{eqnarray}
	Then
	\begin{eqnarray}
		\text{P}(\|\mathbf{Z}\|>t_0)\leq 2de^{-\frac{1}{2}\ln^2 N}.
		\label{eq:prob1}
	\end{eqnarray}
	If $\|\mathbf{Z}\|\leq t_0$, then
	\begin{eqnarray}
		\left\|\frac{1}{N}\sum_{i\in S_0}(\mathbf{X}_i-\mu^*)(\mathbf{X}_i-\mu^*)^T\mathbf{1}(\|\mathbf{X}_i-\mu^*\|\leq M)\right\|&\leq& \left\|\sum_{i\in S_0} \mathbf{W}_i\right\|+\left\|\mathbb{E}[(\mathbf{X}-\mu^*)(\mathbf{X}-\mu^*)^T]\right\|\nonumber\\
		&\leq & t_0+\sigma^2.
		\label{eq:mbound}
	\end{eqnarray}
	Let 
	\begin{eqnarray}
		M=\sqrt{\frac{4\sigma^2 d}{\alpha}}.
		\label{eq:M}
	\end{eqnarray}
	Then
	\begin{eqnarray}
		\text{P}\left(\|\mathbf{X}_i-\mu^*\|>M\right)\leq \frac{\mathbb{E}[\|\mathbf{X}_i-\mu^*\|_2^2]}{M^2}=\frac{\mathbb{E}[\tr(V^*)]}{M^2}=\frac{\sigma^2 d}{M^2}=\frac{\alpha}{4}.
	\end{eqnarray}
	Define $n=|\{i\in S_0|\|\mathbf{X}_i-\mu^*\|>M \}|$ as the number of samples whose distances to $\mu^*$ are more than $M$. From Chernoff inequality,
	\begin{eqnarray}
		\text{P}\left(n>\frac{1}{2}\alpha m\right)\leq \exp\left[-\left(\frac{1}{2}\ln 2-\frac{1}{4}\right) N\alpha\right].
		\label{eq:prob2}
	\end{eqnarray}
	Now we assume that $\|\mathbf{Z}\|\leq t_0$ and $n\leq \alpha m/2$. The probability of violating these two conditions are bounded in \eqref{eq:prob1} and \eqref{eq:prob2}, respectively. Then define
	\begin{eqnarray}
		G_0=\{i\in S_0\setminus \mathcal{B}|\|\mathbf{X}_i-\mu^*\|\leq M \}.
	\end{eqnarray}
	The size of $G_0$ is lower bounded by
	\begin{eqnarray}
		|G_0|\geq N-|\mathcal{B}|-n\geq N-(1-\alpha)N-\frac{1}{2}\alpha m=\frac{1}{2}\alpha m.
	\end{eqnarray}
	Then
	\begin{eqnarray}
		&&\left\|\frac{1}{|G_0|}\sum_{i\in G_0} (\mathbf{Y}_i-\mu^*)(\mathbf{Y}_i-\mu^*)^T\right\|\nonumber\\
		&\overset{(a)}{=}&\left\|\frac{1}{|G_0|}\sum_{i\in G_0} (\mathbf{X}_i-\mu^*)(\mathbf{X}_i-\mu^*)^T\right\|\nonumber\\
		&\leq &\frac{|S_0|}{|G_0|}\left\|\frac{1}{|S_0|}\sum_{i\in S_0} (\mathbf{X}_i-\mu^*)(\mathbf{X}_i-\mu^*)^T\mathbf{1}\left(\|\mathbf{X}_i-\mu^*\|\leq M\right)\right\|\nonumber\\
		&\overset{(b)}{\leq} & \frac{2}{\alpha} (\sigma^2 + t_0)\nonumber\\
		&\overset{(c)}{=}&\frac{2}{\alpha}\sigma^2\left(1+\frac{M}{\sigma}\sqrt{\frac{2\ln^2 N}{N}}+\frac{4M^2}{3N\sigma^2}\ln^2 N\right)\nonumber\\
		&\overset{(d)}{=}&\frac{2}{\alpha}\sigma^2 \left(1+\sqrt{\frac{8d\ln^2 N}{\alpha m}}+\frac{16d}{3\alpha} \frac{\ln^2 N}{N}\right)\nonumber\\
		&<&\frac{2\sigma^2}{\alpha}\left(1+\sqrt{\frac{16d\ln^2 N}{3\alpha m}}\right)^2.
		\label{eq:Gresult}
	\end{eqnarray}
	For (a), recall that for all samples that are not attacked, $\mathbf{Y}_i=\mathbf{X}_i$. (b) comes from \eqref{eq:mbound}. (c) comes from \eqref{eq:t0}. (d) comes from \eqref{eq:M}. Recall that \eqref{eq:Gresult} relies on two conditions, $\|\mathbf{Z}\|\leq t_0$ and $n\leq \alpha m/2$. Therefore, the probability that \eqref{eq:Gresult} does not hold is no more than the sum of \eqref{eq:prob1} and \eqref{eq:prob2}. The proof is complete.
\end{proof}
Under strong contamination model, Lemma \ref{lem:Gstrong} can be used to replace Lemma \ref{lem:G} under additive contamination model. Lemma \ref{lem:I1bound} still holds here. For Lemma \ref{lem:I2bound}, \eqref{eq:lamc2} becomes
\begin{eqnarray}
	\lambda_c \geq \frac{8\sigma^2}{\alpha} \left(1+\sqrt{\frac{16d\ln^2 N}{3\alpha m}}\right)^2,
	\label{eq:lamc2-strong}
\end{eqnarray}
while other conditions remain the same. Using the same steps as the proof of Lemma \ref{lem:I2bound}, it can be shown that \eqref{eq:I2bound} in Lemma \ref{lem:I2bound} holds. Similar arguments hold for Lemma \ref{lem:I3bound}. Therefore,
\begin{eqnarray}
	\mathbb{E}[\|\hat{\mu}-\mu^*\|_2^2]\leq \frac{3\sigma^2p}{N_A}+\frac{15\lambda_c}{2\alpha} + \delta_N,
\end{eqnarray}
in which $\delta_N$ decays faster than any polynomial of $N$. The proof of Theorem 2 is complete.

\subsection{Proof of Theorem 3}\label{sec:minimax}

Recall that $R_A$ and $R_S$ are defined as
\begin{eqnarray}
	R_A(\alpha) = \underset{\hat{\mu}}{\inf}\underset{D\in \mathcal{F}}{\sup}\underset{\pi_A(\alpha)}{\sup}\|\hat{\mu}-\mu^*\|_2,
	\label{eq:ra}
\end{eqnarray}
and
\begin{eqnarray}
	R_S(\alpha) = \underset{\hat{\mu}}{\inf}\underset{D\in \mathcal{F}}{\sup}\underset{\pi_S(\alpha)}{\sup}\|\hat{\mu}-\mu^*\|_2,
	\label{eq:rs}
\end{eqnarray}
in which $\mathcal{F}$ is the set of all distributions satisfying
\begin{eqnarray}
	\Cov[\mathbf{X}_i]\preceq \frac{\sigma^2}{n}\mathbf{I},\label{eq:vxi}\\
\end{eqnarray}
and 
\begin{eqnarray}
	\Cov[\mathbf{X}_0]\preceq \frac{\sigma^2}{N_A}\mathbf{I}.
	\label{eq:vx0}
\end{eqnarray}
Under the condition
\begin{eqnarray}
	\left\lceil\frac{N_A}{n}\right\rceil\leq \frac{\ln \frac{d}{4}}{4\beta\left(\ln \frac{1}{\beta} + 1\right)},
	\label{eq:naub}
\end{eqnarray}
we need to prove

\begin{eqnarray}
	R_A(\alpha)&\geq& \frac{\sigma}{2\sqrt{2\alpha}}\left(\frac{1}{\sqrt{n}}+\frac{1}{\sqrt{N_A}}\right),\label{eq:rabound}\\
	R_S(\alpha)&\geq & \frac{\sigma}{2\sqrt{2}\alpha}\left(\frac{1}{\sqrt{n}}+\frac{1}{\sqrt{N_A}}\right).
	\label{eq:rsbound}
\end{eqnarray}
\textbf{Proof of \eqref{eq:rabound}}. Now we first derive the minimax lower bound under additive contamination model.

Recall the definition of $R_A$ in \eqref{eq:ra}. Define
\begin{eqnarray}
	R_A'(\beta)=\underset{\hat{\mu}}{\inf}\underset{D\in \mathcal{F}}{\sup}\underset{\pi_A'(\beta)}{\sup} \|\hat{\mu}-\mu^*\|_2,
\end{eqnarray}
in which $\pi_A'(\beta)$ means that for each sample in $S_0$, $\mathbf{Y}_i=\mathbf{X}_i$ with probability $\beta$. In additive contamination model, whether $\mathbf{Y}_i=\mathbf{X}_i$ is independent of the value of $\mathbf{X}_i$. Note that now the number of uncorrupted samples is random, following a binomial distribution $\mathbb{B}(N, \beta)$. On the contrary, $\pi_A(\alpha)$ means that the ratio of uncorrupted sample is fixed to be at least $\alpha$. The following lemma links $R_A$ and $R_A'$.

\begin{lemma}\label{lem:link}
	With probability at least $1-\exp[-(\ln 2-1/2)m\alpha]$,
	\begin{eqnarray}
		R_A(\alpha)\geq R_A'\left(\frac{\alpha}{2}\right).
	\end{eqnarray}
\end{lemma}
\begin{proof}
	Let $n$ be the number of samples that are not corrupted, under corruption policy $\pi_A'(\beta)$. By Chernoff inequality, for $t>N\beta$,
	\begin{eqnarray}
		\text{P}(n>T)\leq e^{-N\beta}\left(\frac{eN\beta}{t}\right)^t.
	\end{eqnarray}
	Let $t=2N\beta$. Then
	\begin{eqnarray}
		\text{P}(n>2N\beta)\leq e^{-(2\ln 2 - 1)N\beta}.
	\end{eqnarray}
	Let $\beta=\alpha/2$. As long as $n\leq 2N\beta=N\alpha$, the corruption will be not as severe as $\pi_A(\alpha)$, then $R_A(\alpha)\geq R_A'(\beta)$. The proof is complete.
\end{proof}
It remains to get a lower bound of $R_A'(\beta)$.  

We analyze the case with $N_A\geq n$ first. Let 
\begin{eqnarray}
	k = \left\lceil \frac{N_A}{n}\right\rceil.
	\label{eq:k}
\end{eqnarray}

Let $Z$ be a random variable taking values in $\{1,\ldots, d\}$ uniformly, with each value of $Z$ corresponds to a hypothesis.
Denote $p_j(\cdot)$ as the distribution of the underlying truth of corrupted samples $\mathbf{X}_1,\ldots, \mathbf{X}_m$ and the auxiliary clean sample $\mathbf{X}_0$ as following.
\begin{definition}
	Construct $d$ hypotheses. For $j$-th hypothesis, define a distribution $p_j$ as following: For $i=1,\ldots, m$, let $X_{ij} = \sigma/\sqrt{2n\beta}$, in which $X_{ij}$ is the $j$-th component of $\mathbf{X}_i$. For all $l\neq j$, $p_i(X_{il} = 0) = 1-2\beta$, $p_i(X_{il}=\sigma/\sqrt{2n\beta}) = \beta$, and $p_i(X_{il}=-\sigma/\sqrt{2n\beta}) = \beta$. Moreover, construct $k$ vectors $\mathbf{U}_1, \ldots, \mathbf{U}_k$. Conditional on $Z=j$, $\mathbf{X}_1,\ldots, \mathbf{X}_m, \mathbf{U}_1,\ldots, \mathbf{U}_k$ are i.i.d. Then let
	\begin{eqnarray}
		\mathbf{X}_0=\frac{1}{k}\sum_{i=1}^k \mathbf{U}_i.
	\end{eqnarray}
\end{definition}

The above construction uses some idea from \cite{diakonikolas2018list}, Proposition 5.4.

The mean of $p_j$ is
\begin{eqnarray}
	\mu_j = \frac{\sigma}{\sqrt{2n\beta}}\mathbf{e}_j,
	\label{eq:mui}
\end{eqnarray}
in which $\mathbf{e}_i$ is the unit vector in $i$-th coordinate. Such construction also ensures that $\Var[X_j]=\sigma^2$ for $j\neq i$. Thus assumption \ref{ass:var} is satisfied.

Construct $p_Y$, the distribution of $Y$ as following: $p_Y(Y_j = 0)=1-2\beta$, $p_Y(Y_j = \sigma/\sqrt{2n\beta})=p_Y(Y_j=-\sigma/\sqrt{2n\beta}) = \beta$. Then it is straightforward to show that for each $i$, there exists a $q_i$ such that $p_Y=\beta p_i+(1-\beta)q_i$. Therefore, we design the contamination strategy $\pi_A'(\beta)$ as follows: for each $i\in [m]$, with probability $\beta$, $\mathbf{Y}_i=\mathbf{X}_i$; with probability $1-\beta$, the sample is corrupted, and then $\mathbf{Y}_i$ follows a conditional distribution $q_i$. With this operation, for all values of $Z$, the distribution of $\mathbf{Y}_i$ is always $p_Y$, which does not depend on $Z$.

For all possible estimator $\hat{\mu}$, define
\begin{eqnarray}
	\hat{Z}=\underset{i\in [d]}{\arg\min}\|\hat{\mu}-\mu_i\|.
	\label{eq:zhat}
\end{eqnarray}
From Fano's inequality,
\begin{eqnarray}
	\text{P}(\hat{Z}\neq Z)&\overset{(a)}{\geq}& \frac{H(Z|\mathbf{Y}_1,\ldots, \mathbf{Y}_m, \mathbf{X}_0)-\ln 2}{\ln (d-1)}\nonumber\\
	&=&\frac{H(Z)-I(Z;\mathbf{Y}_1,\ldots, \mathbf{Y}_m, \mathbf{X}_0)-\ln 2}{\ln(d-1)}\nonumber\\
	&\geq & 1-\frac{I(Z;\mathbf{Y}_1,\ldots, \mathbf{Y}_m, \mathbf{X}_0)+\ln 2}{\ln d}\nonumber\\
	&\overset{(b)}{=}&1-\frac{I(Z;\mathbf{X}_0)+\ln 2}{\ln d}.
\end{eqnarray}
(a) holds since $Z\rightarrow \mathbf{Y}_1,\ldots, \mathbf{Y}_m, \mathbf{X}_0\rightarrow \hat{Z}$ is a Markov chain. (b) holds the distribution of $\mathbf{Y}$ does not depend on $Z$. The mutual information can be bounded by
\begin{eqnarray}
	I(Z;\mathbf{X}_0) = I\left(Z;\frac{1}{k}\sum_{i=1}^k \mathbf{U}_i\right)\leq I(Z;\mathbf{U}_1,\ldots, \mathbf{U}_k)=k(H(\mathbf{U}_i) - H(\mathbf{U}_i|Z)),
\end{eqnarray}

For any $j$, given $Z=j$, $U_{ij}$ is fixed. Therefore
\begin{eqnarray}
	H(\mathbf{U}_i) - H(\mathbf{U}_i|Z=j) &\leq& -2\beta\ln \beta - (1-2\beta)\ln (1-2\beta)\nonumber\\
	&\leq & 2\beta\left(\ln \frac{1}{\beta} + 1\right).
\end{eqnarray}
Hence
\begin{eqnarray}
	\text{P}(\hat{Z}\neq Z) \geq 1-\frac{2k\beta \left(\ln \frac{1}{\beta} + 1\right)+\ln 2}{\ln d}.
\end{eqnarray}
Recall \eqref{eq:naub} yields
\begin{eqnarray}
	k\leq \frac{\ln \frac{d}{4}}{4\beta\left(\ln \frac{1}{\beta} + 1\right)}.
	\label{eq:nabound}
\end{eqnarray}
Hence $\text{P}(\hat{Z}\neq Z)\geq 1/2$.

Recall \eqref{eq:mui}. For $i,j\in [d]$, $i\neq j$, $\|\mu_i-\mu_j\|_2\geq \sigma/\sqrt{\beta}$. Therefore, if $\hat{Z}\neq Z$, then
\begin{eqnarray}
	\|\hat{\mu}-\mu^*\|_2&\overset{(a)}{=}&\|\hat{\mu}-\mu_Z\|_2\nonumber\\
	&\overset{(b)}{\geq}& \frac{1}{2}(\|\hat{\mu}-\mu_Z\|_2+\|\hat{\mu}-\mu_{\hat{Z}}\|_2)\nonumber\\
	&\geq & \frac{1}{2}\|\mu_Z-\mu_{\hat{Z}}\|_2\nonumber\\
	&\overset{(c)}{\geq} & \frac{\sigma}{2\sqrt{n\beta}}.
\end{eqnarray}
For (a), note that given $Z=j$, the real mean of $\mathbf{X}$ is $\mu^*=mu_i$. (b) uses \eqref{eq:zhat}, which yields $\|\hat{\mu}-\mu_{\hat{Z}}\|_2\leq \|\hat{\mu}-\mu_Z\|_2$. (c) uses \eqref{eq:mui}, which yields $\|\mu_j-\mu_l\|_2=\sigma/\sqrt{n\beta}$ for $j\neq l$.

From the above analysis, we conclude that if \eqref{eq:nabound} holds, then with probability at least $1/2$, $\|\hat{\mu}-\mu^*\|_2\geq \sigma/(2\sqrt{\beta})$. Taking $\beta=\alpha/2$. From Lemma \ref{lem:link}, with probability at least $1/2-\exp[-(\ln 2-1/2)m\alpha]$, 
\begin{eqnarray}
	\|\hat{\mu}-\mu^*\|_2\geq \frac{\sigma}{2\sqrt{n\beta}} = \frac{\sigma}{\sqrt{2n\alpha}}.
\end{eqnarray}
The above derivation holds for $N_A\geq n$. Similar bound can be derived for $N_A<n$:
\begin{eqnarray}
	\|\hat{\mu}-\mu^*\|_2\geq = \frac{\sigma}{\sqrt{2N_A\alpha}}.
\end{eqnarray}
Therefore \eqref{eq:rabound} holds.

\textbf{Proof of \eqref{eq:rsbound}}. Now we move on to strong contamination model. Define
\begin{eqnarray}
	R_S'(\beta)=\underset{\hat{\mu}}{\inf}\underset{D\in \mathcal{F}}{\sup}\underset{\pi_S'(\beta)}{\sup} \|\hat{\mu}-\mu^*\|_2,
\end{eqnarray}
in which $\pi_S'(\beta)$ is a corruption strategy such that $\mathbf{Y}_i=\mathbf{X}_i$ with probability
$\beta$. The difference of $\pi_S'$ with $\pi_A'$ is that $\pi_S'$ allows the whether $\mathbf{Y}_i=\mathbf{X}_i$ depends on the value of $\mathbf{X}_i$. Similar to Lemma \ref{lem:link}, the following relation holds:
\begin{eqnarray}
	R_S(\alpha) \geq R_S'\left(\frac{\alpha}{2}\right),
\end{eqnarray}
Let $\beta = \alpha / 2$. Now we construct hypotheses as following, which is more complex than those in additive contamination model. For simplicity, now we only show the case with $N_A\geq n$.
\begin{definition}
	Construct $d$ hypotheses, and denote $p_j$ as the distribution of $\mathbf{X}$ under $i$-th hypothesis. If $Z=j$, then $\mathbf{X}$ follows $p_j$. $p_j$ is defined as follows. Given $Z=j$, for $i\in [m]$, $X_{ij}=\sigma/(2\sqrt{n}\beta)$. Moreover, construct another auxiliary random variable $R$, such that $p_j(R=0) = 1-\beta$, $p_j(R=1)=\beta$. Conditional on $R=0$, for all $j\neq i$, $p_j(X_j=0|R=0) = 1-2\beta$, $p_j(X_j = \sigma/(2\sqrt{n\beta})|R=0)=p_j(X_j = -\sigma/(2\sqrt{n\beta})|R=0)=\beta$. Conditional on $R=1$, for all $j\neq i$, $p_j(X_j=\sigma/(2\sqrt{n}\beta)|R=1) = p_j(X_j=-\sigma/(2\sqrt{n}\beta)|R=1) = \beta$. Conditional on $Z=j$ and $U$, $X_j$ are independent for $j\in [d]\setminus\{i\}$. 
\end{definition}

Now the mean of $p_j$ is
\begin{eqnarray}
	\mu_i=\frac{\sigma}{2\sqrt{n}\beta} \mathbf{e}_i.
\end{eqnarray}
Under $Z=j$, for $j\neq i$, the variance in $j$-th dimension is
\begin{eqnarray}
	\Var_i[X_j]=2\beta^2\left(\frac{\sigma}{2\sqrt{n}\beta}\right)^2 + 2\beta(1-\beta)\left(\frac{\sigma}{2\sqrt{n\beta}}\right)^2\leq \frac{\sigma^2}{n},
\end{eqnarray}
hence \eqref{eq:vxi} is satisfied.

Construct $p_Y$, the distribution of $Y$ as following: $p_Y(Y_j=0)=1-2\sqrt{n}\beta$, $p_Y(Y_j=\sigma/(2\sqrt{n}\beta)) = p_Y(Y_j=-\sigma/(2\sqrt{n}\beta)) = \beta$. Then the total variation distance between $p_j$ and $p_Y$ is $\mathbb{TV}(p_j, p_Y)$. Hence, for any $i\in [d]$, $\mathbf{X}$ follows $p_j$, and the attacker can ensure that $\mathbf{Y}$ follows $p_Y$ and $\mathbf{Y}=\mathbf{X}$ with probability $\beta$. Note that now for $i,j\in [d]$, $i\neq j$, $\|\mu_i-\mu_j\|_2\geq \sigma/(\sqrt{2n}\beta)$. Let $\beta=\alpha/2$, and follow similar steps in additive contamination model, with probability at least $1/2-\exp[-(\ln 2-1/2)N\alpha]$,
\begin{eqnarray}
	\|\hat{\mu}-\mu^*\|_2\geq \frac{\sigma}{\sqrt{2n}\alpha}.
\end{eqnarray}
The case with $N_A<n$ can be proved similarly. The proof of Theorem 3 is complete.

\subsection{Proof of Theorem 4}\label{sec:distributed}
Conduct the same decomposition as \eqref{eq:mse} to $\|g(\mathbf{w}) - \nabla F(\mathbf{w})\|_2^2$:
\begin{eqnarray}
	\|g(\mathbf{w})-\nabla F(\mathbf{w})\|_2^2&\leq& 3\|\mathbf{P}(\mu_w(A) - \nabla F(\mathbf{w}))\|_2^2+3\|(\mathbf{I}-\mathbf{P})(\mu_w(S) - \mu_w(G))\|_2^2 \nonumber\\
	&&+ 3\|(\mathbf{I}-\mathbf{P})(\mu_w(G)-\mu^*)\|_2^2,
	\label{eq:msegrad}
\end{eqnarray}
in which
\begin{eqnarray}
	\mu_w(A)=\frac{1}{N_A}\sum_{j=1}^{N_A}\nabla f(\mathbf{w}, \mathbf{Z}_{0j}).
\end{eqnarray}
Then we show the following lemma.
\begin{lemma}
	For all $0\leq \lambda \leq N_A/\sigma$,
	\begin{eqnarray}
		\underset{\mathbf{v}:\|\mathbf{v}\|_2=1}{\sup}\mathbb{E}\left[e^{\lambda \mathbf{v}^T(\mu_w(A)-\nabla F(\mathbf{w}))}\right]\leq e^{\frac{1}{2N_A}\sigma^2 \lambda^2}.
	\end{eqnarray}
\end{lemma}
\begin{proof}
	For any $\mathbf{v}\in \mathbb{R}^d$ with $\|\mathbf{v}\|_2=1$, from Assumption \ref{ass:distributed}(a), for all $0\leq \lambda\leq N_A/\sigma$,
	\begin{eqnarray}
		\mathbb{E}\left[e^{\lambda \mathbf{v}^T(\mu_w(A)-\nabla F(\mathbf{w}))}\right]&=& \mathbb{E}\left[\exp\left(\frac{1}{N_A}\lambda \mathbf{v}^T \sum_{j=1}^{N_A}\left(\nabla f(\mathbf{w}, Z_{0j}) - \nabla F(\mathbf{w})\right)\right)\right]\nonumber\\
		&\leq & \left(e^{\frac{1}{2}\sigma^2 \frac{\lambda^2}{N_A^2}}\right)^{N_A}\nonumber\\
		&=& e^{\frac{1}{2N_A}\sigma^2 \lambda^2}.
	\end{eqnarray}
\end{proof}
Therefore for any $t$,
\begin{eqnarray}
	\text{P}(\mathbf{v}^T (\mu_w(A)- \nabla F(\mathbf{w})) > t)&\leq & \underset{\lambda \geq 0}{\inf} e^{-\lambda t}\mathbb{E}\left[e^{\lambda \mathbf{v}^T(\mu_w(A)-\nabla F(\mathbf{w}))}\right]\nonumber\\
	&=& \underset{0\leq \lambda \leq N_A/\sigma}{\inf}e^{-\lambda t}e^{\frac{1}{2N_A}\sigma^2 \lambda^2}\nonumber\\
	&=&e^{-\frac{N_A}{2}\min\left\{\frac{t^2}{\sigma^2}, \frac{t}{\sigma} \right\}}
\end{eqnarray}
Let 
\begin{eqnarray}
	t=\sqrt{\frac{2}{N_A}\ln \frac{p}{\delta_0}}\sigma.
\end{eqnarray}
From the condition (condition (5) in Theorem 4)
\begin{eqnarray}
	N_A\geq 2\ln \frac{mT}{\delta},
	\label{eq:na}
\end{eqnarray}	
$t\leq \sigma$ holds. Then
\begin{eqnarray}
	\text{P}\left(\mathbf{v}^T(\mu_w(A)- \nabla F(\mathbf{w}))>\sqrt{\frac{2}{N_A}\ln \frac{p}{\delta_0}}\sigma\right)\leq \frac{\delta_0}{p}.
\end{eqnarray}
Note that $\mathbf{P}=\mathbf{U}_p\mathbf{U}_p^T$, hence
\begin{eqnarray}
	\|\mathbf{P}(\mu_w(A)-\nabla F(\mathbf{w}))\|_2^2 &=& \mathbf{U}_p \mathbf{U}_p^T (\mu_w(A)-\nabla F(\mathbf{w}))(\mu_w(A)-\nabla F(\mathbf{w}))^T \mathbf{U}_p \mathbf{U}_p^T\nonumber\\
	&=&\|\mathbf{U}_p^T(\mathbf{X}_0-\nabla F(\mathbf{w}))\|_2^2.
\end{eqnarray}
Then
\begin{eqnarray}
	\text{P}\left(\|\mathbf{P}(\mu_w(A)-\nabla F(\mathbf{w}))\|_2^2 > \frac{2p}{N_A}\ln \frac{p}{\delta_0}\sigma^2\right) 
	&=& \text{P}\left(\|\mathbf{U}_p^T(\mathbf{X}_0-\nabla F(\mathbf{w}))\|_2^2 > \frac{2p}{N_A}\ln \frac{p}{\delta_0}\sigma^2\right)\nonumber\\
	&\leq & \sum_{k=1}^p \text{P}\left((\mathbf{u}_k^T (\mu_w(A)-\mu^*))^2>\frac{2}{N_A}\ln \frac{p}{\delta_0}\sigma^2\right)\nonumber\\
	&=&\sum_{k=1}^p \text{P}\left(\mathbf{u}_k^T (\mu_w(A)-\mu^*)>\sqrt{\frac{2}{N_A}\ln \frac{p}{\delta_0}}\sigma\right)\nonumber\\
	&=&\delta_0.
\end{eqnarray}
Therefore, with probability at least $1-\delta_0$,
\begin{eqnarray}
	\|\mathbf{P}(\mu_w(A)-\nabla F(\mathbf{w}))\|_2^2\leq \frac{2p}{N_A}\sigma^2\ln \frac{p}{\delta_0}.
	\label{eq:I1grad}
\end{eqnarray}
\eqref{eq:I1grad} bounds the first term in \eqref{eq:msegrad}. For the second and third term in \eqref{eq:msegrad}, we can use \eqref{eq:I2highp} and \eqref{eq:I3highp}. Recall that \eqref{eq:I2highp} and \eqref{eq:I3highp} rely on \eqref{eq:lamc2} and \eqref{eq:m2}. In distributed learning setting, the variance of $\mathbf{X}_i$ satisfies
\begin{eqnarray}
	\Var[X_i]\preceq \frac{\sigma^2}{n} \mathbf{I},
\end{eqnarray} 
and there are $m$ working machines in total, hence under additive contamination model, \eqref{eq:lamc2} becomes
\begin{eqnarray}
	\lambda_c\geq 32\frac{\sigma^2}{n}\left(1+\frac{2d}{\alpha m}\right)\frac{1+\epsilon}{1-\epsilon},
	\label{eq:lamc2new}
\end{eqnarray}
while under strong contamination model, \eqref{eq:lamc2-strong} becomes
\begin{eqnarray}
	\lambda_c \geq \frac{8\sigma^2}{n\alpha} \left(1+\sqrt{\frac{16d\ln^2 m}{3\alpha m}}\right)^2\frac{1+\epsilon}{1-\epsilon}.
	\label{eq:lamc2new-strong}
\end{eqnarray}

Using Lemma \ref{lem:I2bound} and Lemma \ref{lem:I3bound}, if $\lambda_c$ satisfies \eqref{eq:lamc2new} or \eqref{eq:lamc2new-strong} for additive or strong contamination model, respectively, then with probability at least $1-e^{-\frac{1}{64}\alpha m}- 4me^{-\frac{1}{16}\lambda_cp\alpha^2 m^\frac{1}{3}\epsilon^2}$,
\begin{eqnarray}
	\|(\mathbf{I}-\mathbf{P})(\mu_w(S)-\mu_w(G))\|_2^2&\leq& \frac{2\lambda_c}{\alpha},\label{eq:I2grad}\\
	\|(\mathbf{I}-\mathbf{P})(\mu_w(G)-\nabla F(\mathbf{w}))\|_2^2&\leq& \frac{\lambda_c}{2\alpha}.
	\label{eq:I3grad}
\end{eqnarray}
Substitute three terms in \eqref{eq:msegrad} with \eqref{eq:I1grad}, \eqref{eq:I2grad} and \eqref{eq:I3grad}, it can be shown that with probability at least $1-\delta_0-e^{-\frac{1}{64}\alpha m}- 4me^{-\frac{1}{16}\lambda_cp\alpha^2 m^\frac{1}{3}\epsilon^2}$,
\begin{eqnarray}
	\|g(\mathbf{w})-\nabla F(\mathbf{w})\|_2^2 \leq \frac{6p}{N_A}\sigma^2 \ln \frac{p}{\delta_0} + \frac{15\lambda_c}{2\alpha}.
	\label{eq:gradbound} 
\end{eqnarray}
Recall that there are $T$ iterations in total. Therefore, we get the following union bound. With probability at least $1-T\delta_0-Te^{-\frac{1}{64}\alpha m}- 4Tme^{-\frac{1}{16}\lambda_cp\alpha^2 m^\frac{1}{3}\epsilon^2}$, for $t=1,\ldots, T$, \eqref{eq:gradbound} holds for $\mathbf{w}_1,\ldots, \mathbf{w}_T$. Let $\delta_0=\delta/T$. Then with probability at least $1-\delta-Te^{-\frac{1}{64}\alpha m}- 4mTe^{-\frac{1}{16}\lambda_cp\alpha^2 m^\frac{1}{3}\epsilon^2}$, for $t=1,\ldots, T$,
\begin{eqnarray}
	\|g(\mathbf{w}_t)-\nabla F(\mathbf{w}_t)\|_2^2 \leq \frac{6p}{N_A}\sigma^2 \ln \frac{pT}{\delta} + \frac{15\lambda_c}{2\alpha}.	
	\label{eq:union}
\end{eqnarray}
Based on \eqref{eq:union}, we then solve the optimization problem. Recall that $\mathbf{w}_{t+1} = \mathbf{w}_t-\eta g(\mathbf{w}_t)$. Then
\begin{eqnarray}
	\|\mathbf{w}_{t+1} - \mathbf{w}^*\|_2&=& \|\mathbf{w}_t-\eta g(\mathbf{w}_t) - \mathbf{w}^*\|_2\nonumber\\
	&\leq & \|\mathbf{w}_t-\eta\nabla F(\mathbf{w}_t) - \mathbf{w}^*\|_2+\eta \|\nabla F(\mathbf{w}_t) - g(\mathbf{w}_t)\|_2.
	\label{eq:wdecomp}
\end{eqnarray}
For the first term in \eqref{eq:wdecomp},
\begin{eqnarray}
	&&\|\mathbf{w}_t-\eta\nabla F(\mathbf{w}_t) - \mathbf{w}^*\|_2^2\nonumber\\
	&=& \|\mathbf{w}_t-\mathbf{w}^*\|_2^2 - 2\eta \langle \mathbf{w}_t-\mathbf{w}^*, \nabla F(\mathbf{w}_t)\rangle+\eta^2 \|\nabla F(\mathbf{w}_t)\|_2^2\nonumber\\
	&\overset{(a)}{\leq} &\|\mathbf{w}_t-\mathbf{w}^*\|_2^2-2\eta\left(F(\mathbf{w}_t) - F(\mathbf{w}^*)+\frac{\alpha}{2}\|\mathbf{w}_t-\mathbf{w}^*\|_2^2\right)+\eta^2\|\nabla F(\mathbf{w}_t)\|_2^2\nonumber\\
	&\overset{(b)}{\leq}&(1-\eta\alpha)\|\mathbf{w}_t-\mathbf{w}^*\|_2^2 - 2\eta(F(\mathbf{w}_t - F(\mathbf{w}^*))) + 2\eta^2 L(F(\mathbf{w}_t) - F(\mathbf{w}^*))\nonumber\\
	&\overset{(c)}{\leq} & (1-\eta\alpha) \|\mathbf{w}_t-\mathbf{w}^*\|_2^2.
\end{eqnarray}
(a) comes from the $\alpha$-strong convexity of $F$ (Assumption \ref{ass:distributed}(b)), which ensures that
\begin{eqnarray}
	F(\mathbf{w}^*)\geq F(\mathbf{w}_t)+\langle \nabla F(\mathbf{w}_t), \mathbf{w}^*-\mathbf{w}_t\rangle + \frac{\alpha}{2}\|\mathbf{w}_t-\mathbf{w}^*\|_2^2.
\end{eqnarray}
(b) comes from the $L$-smoothness of $F$, which ensures that $\|\nabla F(\mathbf{w}_t)\|_2^2\leq 2L (F(\mathbf{w}_t) - F(\mathbf{w}^*))$. (c) comes from condition (4) in Theorem 4, which requires $\eta\leq 1/L$. Hence
\begin{eqnarray}
	\|\mathbf{w}_t-\eta\nabla F(\mathbf{w}_t) - \mathbf{w}^*\|_2\leq \sqrt{1-\eta\alpha}\|\mathbf{w}_t-\mathbf{w}^*\|_2\leq \left(1-\frac{1}{2}\eta\alpha\right) \|\mathbf{w}_t-\mathbf{w}^*\|_2.	
\end{eqnarray}
For the second term in \eqref{eq:wdecomp}, recall \eqref{eq:gradbound}. Therefore
\begin{eqnarray}
	\|\mathbf{w}_{t+1} - \mathbf{w}^*\|_2\leq (1-\rho)\|\mathbf{w}_t-\mathbf{w}^*\|_2+\eta\Delta,
\end{eqnarray}
in which $\rho$ and $\Delta$ are defined as
$\rho = \eta\mu/2$ and
\begin{eqnarray}
	\Delta = \sqrt{\frac{6p}{N_A}\sigma^2 \ln \frac{pT}{\delta} + \frac{15\lambda_c}{2\alpha}}.
	\label{eq:delta}
\end{eqnarray}

Then it is straightforward to show by induction that for $t=1,\ldots, T$,
\begin{eqnarray}
	\|\mathbf{w}_t-\mathbf{w}^*\|_2\leq (1-\rho)^t\|\mathbf{w}_0-\mathbf{w}^*\|_2+\frac{\eta\Delta}{\rho}.
\end{eqnarray}
The algorithm returns 	$\hat{\mathbf{w}}=\mathbf{w}_T$. Therefore,
\begin{eqnarray}
	\|\hat{\mathbf{w}}-\mathbf{w}^*\|_2\leq (1-\rho)^T\|\mathbf{w}_0-\mathbf{w}^*\|_2+\frac{\eta\Delta}{\rho}.
\end{eqnarray}

\end{document}